\documentclass{article}

\usepackage[top=1in, bottom=1in, left=1in, right=1in]{geometry}

\usepackage{comment,url,algorithmic,graphicx,subcaption,relsize}
\usepackage{amssymb,amsfonts,amsmath,amsthm,amscd,dsfont,mathrsfs,mathtools,nicefrac}
\usepackage{float,psfrag,epsfig,color,xcolor,url,hyperref,cleveref}
\usepackage{epstopdf,bbm,enumitem}
\usepackage[toc,page]{appendix}
\usepackage[mathscr]{euscript}

\def\ddefloop#1{\ifx\ddefloop#1\else\ddef{#1}\expandafter\ddefloop\fi}

\def\ddef#1{\expandafter\def\csname bb#1\endcsname{\ensuremath{\mathbb{#1}}}}
\ddefloop ABCDEFGHIJKLMNOPQRSTUVWXYZ\ddefloop

\def\ddef#1{\expandafter\def\csname v#1\endcsname{\ensuremath{\boldsymbol{#1}}}}
\ddefloop ABCDEFGHIJKLMNOPQRSTUVWXYZabcdefghijklmnopqrstuvwxyz\ddefloop

\def\ddef#1{\expandafter\def\csname v#1\endcsname{\ensuremath{\boldsymbol{\csname #1\endcsname}}}}
\ddefloop {alpha}{beta}{gamma}{delta}{epsilon}{varepsilon}{zeta}{eta}{theta}{vartheta}{iota}{kappa}{lambda}{mu}{nu}{xi}{pi}{varpi}{rho}{varrho}{sigma}{varsigma}{tau}{upsilon}{phi}{varphi}{chi}{psi}{omega}{Gamma}{Delta}{Theta}{Lambda}{Xi}{Pi}{Sigma}{varSigma}{Upsilon}{Phi}{Psi}{Omega}{ell}\ddefloop

\def\balign#1\ealign{\begin{align}#1\end{align}}
\def\baligns#1\ealigns{\begin{align*}#1\end{align*}}
\def\balignat#1\ealign{\begin{alignat}#1\end{alignat}}
\def\balignats#1\ealigns{\begin{alignat*}#1\end{alignat*}}
\def\bitemize#1\eitemize{\begin{itemize}#1\end{itemize}}
\def\benumerate#1\eenumerate{\begin{enumerate}#1\end{enumerate}}

\newenvironment{talign*}
 {\let\displaystyle\textstyle\csname align*\endcsname}
 {\endalign}
\newenvironment{talign}
 {\let\displaystyle\textstyle\csname align\endcsname}
 {\endalign}

\def\balignst#1\ealignst{\begin{talign*}#1\end{talign*}}
\def\balignt#1\ealignt{\begin{talign}#1\end{talign}}
\let\originalleft\left
\let\originalright\right
\renewcommand{\left}{\mathopen{}\mathclose\bgroup\originalleft}
\renewcommand{\right}{\aftergroup\egroup\originalright}

\def\tinycitep*#1{{\tiny\citep*{#1}}}
\def\tinycitealt*#1{{\tiny\citealt*{#1}}}
\def\tinycite*#1{{\tiny\cite*{#1}}}
\def\smallcitep*#1{{\scriptsize\citep*{#1}}}
\def\smallcitealt*#1{{\scriptsize\citealt*{#1}}}
\def\smallcite*#1{{\scriptsize\cite*{#1}}}

\def\mbb#1{\mathbb{#1}}

\theoremstyle{plain}  
\newtheorem*{remark}{\textbf{Remark}}

\def\<{\left\langle} % Angle brackets
\def\>{\right\rangle}

\def\E{\mbb{E}} % Expectation symbol

\newcommand{\Exp}[1]{\operatorname{exp}\left({#1}\right)} % Exponential 
\DeclareSymbolFont{rsfs}{U}{rsfs}{m}{n}
\DeclareSymbolFontAlphabet{\mathscrsfs}{rsfs}
\providecommand{\argmin}{\mathop\mathrm{arg min}}

\ifdefined\nonewproofenvironments\else
\ifdefined\ispres\else
\newtheorem{theo}{Theorem}

\newtheorem{definition}{Definition}
\newtheorem{theorem}[theo]{Theorem}
\newtheorem{lemma}[theo]{Lemma}
\newtheorem{example}{Example}
\newtheorem{corollary}[theo]{Corollary}

\newtheorem{proposition}[theo]{Proposition}

\newtheorem{assumption}{Assumption}

\renewenvironment{proof}{\noindent\textbf{Proof.}\hspace*{.3em}}{\qed\\}
\newenvironment{proof-sketch}{\noindent\textbf{Proof Sketch}
  \hspace*{1em}}{\qed\bigskip\\}
\newenvironment{proof-idea}{\noindent\textbf{Proof Idea}
  \hspace*{1em}}{\qed\bigskip\\}
\newenvironment{proof-of-lemma}[1][{}]{\noindent\textbf{Proof of Lemma {#1}}
  \hspace*{1em}}{\qed\\}
\newenvironment{proof-of-theorem}[1][{}]{\noindent\textbf{Proof of Theorem {#1}}
  \hspace*{1em}}{\qed\\}
\newenvironment{proof-attempt}{\noindent\textbf{Proof Attempt}
  \hspace*{1em}}{\qed\bigskip\\}
\fi

\newenvironment{proofof}[1][{}]{\par\noindent{\bf Proof of {#1}. }  }{\qed\bigskip}   
   
\fi
\makeatletter
\@addtoreset{equation}{section}
\makeatother

\hypersetup{
  colorlinks,
  linkcolor={red!50!black},
  citecolor={blue!50!black},
  urlcolor={blue!80!black}
}

\usepackage{mydefs}
\usepackage{caption}
\usepackage{etoc}
\usepackage{wrapfig}
\usepackage{nicematrix}
\usepackage{amsmath}
\usepackage{amsfonts}
\usepackage{amssymb}
\usepackage{amsthm}
\usepackage{autonum}
\usepackage{algorithmic}
\usepackage{bbm}

\DeclareUnicodeCharacter{02B9}{'}

\usepackage[ruled]{algorithm2e}

\usepackage{tikz}

\renewcommand{\contentsname}{Table of Contents}

\newcommand{\citep}[1]{\cite{#1}}
\newcommand{\citet}[1]{\cite{#1}} 

\makeatletter
\renewcommand{\paragraph}{%
  \@startsection{paragraph}{4}%
  {\z@}{1.5ex \@plus 1ex \@minus .2ex}{-1em}%
  {\normalfont\normalsize\bfseries}%
}
\makeatother

\title{
Learning sum of diverse features: computational hardness and efficient gradient-based training for ridge combinations}

\author{
Kazusato Oko\thanks{University of Tokyo and RIKEN AIP. \texttt{oko-kazusato@g.ecc.u-tokyo.ac.jp}.} ,\,
Yujin Song\thanks{University of Tokyo.  \texttt{song-yujin139@g.ecc.u-tokyo.ac.jp}.} ,\, 
Taiji Suzuki\thanks{University of Tokyo and RIKEN AIP. \texttt{taiji@mist.i.u-tokyo.ac.jp}.}  ,\, 
Denny Wu\thanks{New York University and Flatiron Institute. \texttt{dennywu@nyu.edu}. 
\vspace{-2.5mm}} 
}

\begin{document}
\etocdepthtag.toc{mtchapter}
\etocsettagdepth{mtchapter}{subsection}
\etocsettagdepth{mtappendix}{none}

\maketitle

% \vspace{-2mm}

\begin{abstract}
We study the computational and sample complexity of learning a target function $f_*:\mathbb{R}^d\to\mathbb{R}$ with \textit{additive structure}, that is, $f_*(x) = \frac{1}{\sqrt{M}}\sum_{m=1}^M f_m(\langle x, v_m\rangle)$, where $f_1,f_2,...,f_M:\mathbb{R}\to\mathbb{R}$ are nonlinear link functions of single-index models (ridge functions) with diverse and near-orthogonal index features $\{v_m\}_{m=1}^M$, and the number of additive tasks $M$ grows with the dimensionality $M\asymp d^\gamma$ for $\gamma\ge 0$. This problem setting is motivated by the classical additive model literature, the recent representation learning theory of two-layer neural network, and large-scale pretraining where the model simultaneously acquires a large number of ``skills'' that are often \textit{localized} in distinct parts of the trained network. 
We prove that a large subset of polynomial $f_*$ can be efficiently learned by gradient descent training of a two-layer neural network, with a polynomial statistical and computational complexity that depends on the number of tasks $M$ and the \textit{information exponent} of $f_m$, despite the unknown link function and $M$ growing with the dimensionality. We complement this learnability guarantee with computational hardness result by establishing statistical query (SQ) lower bounds for both the correlational SQ and full SQ algorithms. 
\end{abstract}

\section{Introduction}
\label{sec:intro}

We study the problem of learning a polynomial function $f_*:\R^d\to\R$ on isotropic Gaussian data. Specifically, given pairs of training example $\{(x_i,y_i)\}_{i=1}^n$, where $x_i\in\R^d$ is drawn from some distribution $P_x$ and $y_i = f_*(x_i) + \varepsilon_i$, we aim to construct an estimator $\hat{f}$ that achieves small population error: $\E_{x\sim P_x} |f_*(x) - \hat{f}(x)|\le\epsilon$. 
Since the space of degree-$q$ polynomials in $\R^d$ is of dimension $\Theta(d^q)$, without additional structural assumptions on $f_*$, the sample complexity required to achieve small error should scale as $n\gtrsim d^q$, which is computationally prohibitive for learning large-degree polynomials in high dimensions. 
Many prior works have therefore imposed the constraint that $f_*$ exhibits certain low-dimensional latent structure \citep{andoni2014learning1,chen2020learning,damian2022neural}, as in the multi-index model: $f_*(x) = g(Vx)$ where $V\in\R^{k\times d}$ for $k = \mathcal{O}_d(1)$. However, such a low-dimensional restriction may rule out many interesting classes of target functions; for instance, when specialized to the learning of two-layer neural network, these prior results only apply to the setting where the network width $k$ is much smaller than the ambient dimensionality $d$. 

In this work, we investigate the efficient learning of polynomial $f_*$ under a different kind of structural assumption: we allow the target function to depend on a large number of directions which grows with $d$, but we impose an \textit{additive structure}; specifically, we consider $f_*$ to be the sum of $M\asymp d^\gamma$ single-index models (also known as ridge functions) as follows: 
\begin{align} 
    f_*(x) = \frac{1}{\sqrt{M}} \sum_{m=1}^M f_m(\langle x, v_m\rangle), 
    \label{eq:intro-f*}
\end{align}
where $v_1,v_2,...,v_M\in\R^d$ are the unknown index features, and $f_1,f_2,...,f_M:\R\to\R$ are the unknown link functions; 
the $\sqrt{1/M}$ prefactor ensures that $\|f_*\|_{L^2}=\Theta_d(1)$ when the set of directions $\{v_m\}_{m=1}^M$ is diverse, i.e., $\langle v_i,v_j\rangle = o_d(1)$ for $i\neq j$. 
% (see Appendix~\ref{subsection:HPB-Y}). 
Our setting is motivated by the following lines of literature:
\begin{itemize}[leftmargin=*]
    \item \textbf{Estimation of additive model.} In statistical learning theory, additive model is a classical method employed in high-dimensional nonparametric regression \citep{stone1985additive,hastie1987generalized,ravikumar2009sparse}; 
    especially, when the individual functions take the form of single-index model (ridge function), the estimation of \textit{ridge combinations} has been extensively studied \citep{friedman1981projection,pinkus1997approximating,klusowski2016risk}. 
    While efficient algorithms have been proposed when the basis is given \textit{a priori} \cite{bach2008consistency,raskutti2012minimax,suzuki2012fast}, when the index features $\{v_m\}_{m=1}^M$ are unknown, most existing approaches involve non-convex optimization, which has been treated as a black box without convergence guarantees \citep{kandasamy2015high,agarwal2021neural}, or solved with computationally inefficient algorithm \cite{bach2017breaking}. 
    \item \textbf{Learning two-layer neural network.} An important example in the model class \eqref{eq:intro-f*} is a two-layer neural network with $M$ neurons, and it is natural to ask whether such a network can be efficiently learned via standard gradient-based training. Prior works have shown that in the ``narrow width'' setting $M=\Theta_d(1)$, gradient descent can learn $f_*$ with polynomial sample complexity depending on the \textit{information exponent} of the target function \citep{abbe2022merged,ba2022high,damian2022neural,bietti2022learning}. On the other hand, the regime where $M,d$ jointly diverge is not well-understood, and most existing analyses on the complexity of gradient-based training require significant simplification such as quadratic activation \citep{gamarnik2019stationary,sarao2020optimization,martin2023impact}. 
    \item \textbf{Skill localization \& fine-tuning.} 
    Pretrained large neural networks (e.g., language models) can efficiently adapt to diverse downstream tasks by fine-tuning a small set of trainable parameters \citep{devlin2018bert,li2021prefix,hu2021lora}. Recent works have shown that ``skills'' for each individual task are often \textit{localized} in a subset of neurons \citep{dai2021knowledge,elhage2022toy,wang2022finding,panigrahi2023task,todd2023function,arora2023theory}. 
    The additive model \eqref{eq:intro-f*} gives an idealized setting where learning exhibits such skill localization
    (we interpret each $f_m$ as corresponding to one task). 
    As we will see, after an appropriate hidden representation is obtained via gradient-based training, the neural network can efficiently adapt to downstream tasks via fine-tuning the top layer. 
\end{itemize}

Our goal is to characterize the statistical and computational complexity of learning the additive model \eqref{eq:intro-f*}, when $(i)$ the number of single-index tasks is large, that is, $M$ grows with the ambient dimensionality $d$, and $(ii)$, the tasks are \textit{diverse}, i.e., the index features of $f_m$ do not significantly overlap with one another (see Section~\ref{sec:problem} for precise definition). We ask the following questions: 
\begin{enumerate}[leftmargin=*]
    \item Can we efficiently learn \eqref{eq:intro-f*} via \textit{gradient descent} (GD) training of a two-layer neural network? 
    \item What is the hardness of learning \eqref{eq:intro-f*} as measured by \textit{statistical query} (SQ) lower bounds? 
\end{enumerate}
% \begin{center}
% {\it TBD}  
% \end{center} 

\subsection{Our Contributions} 

We address the two questions by providing complexity upper and lower bounds for learning the additive model class \eqref{eq:intro-f*}, both of which depend on the number of tasks $M$, and the \textit{information exponent} $p\in\N$ of the link functions $f_m$ defined as the lowest degree in the Hermite expansion \citep{arous2021online}. 
Our findings are summarized as follows (see Table~\ref{tab:summary} for details).

% \begin{table}[!htb]
% \begin{center}
% % \vspace{-2mm} 
% \scalebox{1}{
% \begin{NiceTabular}{cccc}[hlines]
% \bhline{1.5pt}
% Result type & Query / Algorithm & Statistical complexity & Theorem \\
% \bhline{1.5pt}
% \Block{1-1}{Upper bound} & Online SGD & $M d^{p-1}$ &  Theorem~\ref{theorem:NN-main}  \\
% % & Smoothed SGD & $M d^{p/2}$ &  Proposition \DW{TBD}  \\
% \bhline{1pt}                               
% \Block{2-1}{Lower bound} & Correlational SQ & $M d^{p/2}$ &  Theorem~\ref{theorem:CSQ-Appendix}  \\
%                             &  Full SQ          & $(Md)^{\rho_{p,q,\gamma}}$ &  Theorem~\ref{theorem:SQ-Appendix}
% % \\
%                             % &  Smoothed SGD          & TBD &  TBD 
%                             \\\bhline{1.5pt}  
% \end{NiceTabular}
% }
% % \vspace{2.5mm}
%  \vspace{-1mm} 
% \caption{\small Complexity upper bound for gradient-based learning and (C)SQ lower bounds for the additive model \eqref{eq:intro-f*}. Our upper bound applies to a subclass of \eqref{eq:intro-f*} specified in Section~\ref{sec:problem}. 
% For the full SQ lower bound we take $d\asymp M^{\gamma}$, and for any fixed $\gamma>0$ we may set $p,q>0$ such that the exponent is arbitrarily large, that is, $\rho_{p,q,\gamma}\overset{p,q\to\infty}{\to} \infty$. 
% We translate the tolerance in the (C)SQ lower bounds to sample complexity using the concentration heuristic $\tau\approx n^{-1/2}$. 
% }
% \label{tab:summary} 
% \end{center}
% \vspace{-2.5mm}
% % \vspace{-2mm}
% \end{table} 

\begin{itemize}[leftmargin=*]
    \item In Section~\ref{sec:GD} we show that a representative subclass of \eqref{eq:intro-f*} can be efficiently learned by a gradient-based algorithm (using correlational information) on two-layer neural network, even though the number of single-index tasks $M$ is large and the link functions $f_m$ are not known. Specifically, for $M=\tilde{\Omega}(\sqrt{d})$ tasks with information exponent $p$, we prove that a layer-wise (online) SGD algorithm similar to that in \citep{bietti2022learning,abbe2023sgd} can achieve small population loss using $n = \tilde{\Theta}(M d^{p-1})$ samples. To establish this learning guarantee, we show that the student neurons \textit{localize} into the task directions during SGD training. 
    \item In Section~\ref{sec:SQ} we establish computational lower bounds for learning \eqref{eq:intro-f*}. For \textit{correlational} SQ algorithms, we prove that a tolerance of $\tau^{-2}\gtrsim Md^{p/2}$ is required when link functions $f_m$ have degree $q$ and information exponent $p\le q$. We also provide a full SQ lower bound in the form of $\tau^{-2}\gtrsim (Md)^{\rho_{p,q}}$, where $\rho_{p,q}$ can be made arbitrarily large by varying $p$ and $q$; under the standard $\tau\approx n^{-1/2}$ heuristic for concentration error, this suggests that prior SQ algorithms that achieve linear-in-$d$ sample complexity in the finite-$M$ regime \citep{chen2020learning} cannot attain the same statistical efficiency in our additive model setting with large $M$. 
\end{itemize}

\begin{table}[!htb]
\centering
\begin{minipage}{\textwidth}
% \hrule
% \vspace{3mm}
\begin{tikzpicture}[x=\textwidth/18]
\usetikzlibrary{arrows.meta}
  % Define the style for the timeline and markers
  \tikzset{
    timeline/.style={very thick, line width=1.8pt, -{Triangle[width=7.5pt,length=10pt]}},
    major tick/.style={thick, line width=1.5pt},
    minor tick/.style={thin, line width=1pt},
    tick label/.style={above, outer sep=2mm, scale=0.6}, % Adjusted the scale for text to fit better
    node style/.style={align=center}
  }
 
  % Draw the timeline
  \draw[timeline] (0,0) -- (15.92,0);
  % Add minor ticks and labels
  \draw[minor tick] (0, -0.15) -- (0, 0.15);
  % \node[tick label] at (1, 0.1) {1};
  \draw[minor tick] (2.8, -0.15) -- (2.8, 0.15);
  \draw[minor tick] (6.7, -0.15) -- (6.7, 0.15);
  \draw[minor tick] (10.6, -0.15) -- (10.6, 0.15);
  \draw[minor tick] (14.3, -0.15) -- (14.3, 0.15);

  % Labels below timeline
  \node[node style,font=\scriptsize] at (-1, 0) {Information \\ theoretic limit};
  \node[node style,font=\small] at (2.8, 0.7) {
SQ lower bound \\ \,\![{\color{red!50!black}\textbf{Theorem~\ref{theorem:SQ-Appendix}}}]};
  \node[node style,font=\small] at (6.7, 0.7) { CSQ lower bound \\ \,\![{\color{red!50!black}\textbf{Theorem~\ref{theorem:CSQ-Appendix}}}]};
  \node[node style,font=\small] at (10.6, 0.7) {Online SGD \\ \,\![{\color{red!50!black}\textbf{Theorem~\ref{theorem:NN-main}}}]};
  \node[node style,font=\small] at (14.3, 0.7) {Kernel methods \\ \cite{ghorbani2019linearized}};

  \node[align=center, font=\normalsize] at (0, -0.41) {$Md$};
  \node[align=center, font=\normalsize] at (2.8, -0.42) {$(Md)^{\rho_{p,q,\gamma}}$};
  \node[align=center, font=\normalsize] at (6.7, -0.42) {$Md^{p/2}$};
  \node[align=center, font=\normalsize] at (10.6, -0.42) {$\tilde{O}(Md^{p-1})$};
  \node[align=center, font=\normalsize] at (14.3, -0.42) {$d^{q}$}; 
\end{tikzpicture}
% \vspace{-6mm}
\caption{\small Complexity upper bound for gradient-based learning and (C)SQ lower bounds for the additive model \eqref{eq:intro-f*}, where the single-index tasks have degree $q$ and information exponent $p$. Our upper bound applies to a subclass of \eqref{eq:intro-f*} specified in Section~\ref{sec:problem}. For the full SQ lower bound we take $d\asymp M^{\gamma}$, and for any fixed $\gamma>0$ we may set $p,q>0$ such that the exponent is arbitrarily large, that is, $\rho_{p,q,\gamma}\overset{p,q\to\infty}{\to} \infty$. 
We translate the tolerance in the (C)SQ lower bounds to sample complexity using the concentration heuristic $\tau\approx n^{-1/2}$. 
}
\label{tab:summary}
\end{minipage}
% \vspace{-6mm}
\end{table}

\subsection{Related Works}

\paragraph{Gradient-based learning.} 
Recent works have shown that neural network trained by gradient descent can learn useful representation and adapt to low-dimensional target functions such as single-index \citep{ba2022high,bietti2022learning,mousavi2023neural,berthier2023learning} and multi-index models \citep{damian2022neural,abbe2022merged,bietti2023learning}. In this finite-$M$ setting, the complexity of gradient-based learning is governed by the \textit{information exponent} \citep{arous2021online} or \textit{leap complexity} \citep{abbe2023sgd} of the target function. However, these learning guarantees for low-dimensional $f_*$ yield \textit{superpolynomial} dimension dependence when $M$ grows with $d$. Another line of works goes beyond the low-dimensional assumption by considering the well-specified setting where $f_*$ and the student network have the same architecture and special activation function \citep{gamarnik2019stationary,zhou2021local,akiyama2021learnability,veiga2022phase,martin2023impact}. The setting we consider \eqref{eq:intro-f*} lies between the two regimes: we allow the width to diverge with dimensionality $M=\omega_d(1)$ but do not assume the nonlinear activation is known, and we show that gradient descent can learn $f_*$ with polynomial sample complexity depending on the information exponent when the target weights are ``diverse''. 
We also note that beyond gradient descent training, 
various SQ algorithms have been introduced to solve related polynomial regression tasks \citep{dudeja2018learning,chen2020learning,garg2020learning,diakonikolas2023agnostically}.

\paragraph{Statistical query lower bound.} A statistical query learner \citep{kearns1998efficient,reyzin2020statistical} can access the target function via noisy queries $\tilde{\phi}$ with error tolerance $\tau$: 
$|\tilde{\phi} - \E_{x,y}[\phi(x,y)]|\le\tau$. Lower bound on the performance of SQ algorithm 
is a classical measure of computational hardness. An often-studied subclass of SQ is the \textit{correlational} statistical query (CSQ) \citep{bshouty2002using} where the query is restricted to (noisy version of) $\E_{x,y}[\phi(x)y]$. 
Many existing results on the CSQ complexity assume $f_*$ is low-dimensional ($M=\mathcal{O}_d(1)$), in which case the tolerance scales with $\tau\asymp d^{-\Omega(p)}$, where $p$ is the information exponent or leap complexity of $f_*$ \citep{damian2022neural,abbe2022non,abbe2023sgd}. 
On the other hand, \cite{vempala2019gradient,diakonikolas2020algorithms,goel2020superpolynomial} established CSQ lower bounds for learning two-layer ReLU network without structural assumption on the weights, where the error tolerance $\tau\asymp d^{-\Omega(M)}$ implies a \textit{superpolynomial} complexity when $M=\omega_d(1)$. Similar superpolynomial lower bound was shown for three-layer neural networks in the full SQ model \citep{chen2022hardness}. 
Our result connects these two lines of analyses: we show that in the setting of diverse (near-orthogonal) weights $\{v_m\}_{m=1}^M$, the (C)SQ complexity is polynomial in $M,d$, where the dimension dependence is specified by the information exponent of $f_m$. 

\section{Problem Setting}
\label{sec:problem}

\paragraph{Notations.} Throughout the analysis, $\|\cdot\|$ denotes the $\ell_2$ norm for vectors and the $\ell_2\to\ell_2$ operator norm for matrices. 
$O_d(\cdot)$ and $o_d(\cdot)$ stand for the big-O and little-o notations, where the subscript highlights the asymptotic variable $d$ and suppresses dependence on $p,q$; we write $\tilde{O}(\cdot)$ when (poly-)logarithmic factors are ignored. 
$\Omega(\cdot),\Theta(\cdot)$ are defined analogously. 
We say an event $A$ happens with high probability when the failure probability is bounded by $\Exp{-C\log d}$ where $C$ is a sufficiently large constant; the high probability events are closed under taking union bounds over sets of size ${\mathrm{poly}(d)}$. 

\subsection{Assumptions on Target Function}

We focus on learning a sum of single-index polynomials over Gaussian input, the complexity of which crucially depends on the notion of \textit{information exponent} \citep{dudeja2018learning,arous2021online} defined as the smallest degree of non-zero coefficients for the Hermite expansion of the link function.
\begin{definition}[Information exponent]
Let $\{\He_j\}_{j=0}^{\infty}$ be the normalized Hermite polynomials. The information exponent of square-integrable $g:\R\!\to\!\R$, which we denote by $\mathrm{IE}(g):=p\in\N_+$, is the index of its first non-zero Hermite coefficient, i.e., given the Hermite expansion $g(z) = \sum_{j=0}^\infty\alpha_j \He_j(z)$, $p := \min\{j > 0: \alpha_j \neq 0\}$.
\end{definition}
For example, prior works have shown that online SGD can learn a single-index polynomial with information exponent $p$ over $d$-dimensional Gaussian input with $\mathcal{O}(d^{p-1})$ samples for $p>2$ \citep{arous2021online}, and the complexity can be further improved to $\mathcal{O}(d^{p/2})$ via landscape smoothing \citep{damian2023smoothing}. 
With this definition, we formally state the class of additive models \eqref{eq:intro-f*} considered in this work.
\begin{assumption}[Additive model]
    We consider the following problem class $\mathcal{F}_{d,M,\varsigma}^{p,q}$:
    \begin{align}
    x \sim \mathcal{N}(0,I_d),\quad y = f_*(x) + \nu,
    \quad \text{where~~} f_*(x) = \frac{1}{\sqrt{M}}\sum_{m=1}^M f_m(v_m^\top x), ~~ \nu \sim \mathcal{N}(0,\varsigma^2),
    \end{align} 
    where each $f_m$ is a univariate polynomial with information exponent $p>2$ and degree $q$, and we assume proper normalization $\mathbb{E}_{t\sim \mathcal{N}(0,1)}[f_m(t)^2]=1$, and $\|v_m\|=1$.
    We write the Hermite expansion of $f_m$ as $f_m = \sum_{i=p}^q \alpha_{m,i}{\sf He}_i$, and define the constant $C_p = \left(\frac{\max_{m}|\alpha_{m,p}|}{\min_{m'}|\alpha_{m',p}|}\right)^{\frac{2}{p-2}}$.
\label{assump:additive}
\end{assumption}

\begin{remark}~
% We make the following remarks on the condition of link functions $f_m$. 
\begin{itemize}[leftmargin=*,topsep=0mm]
    \item We restrict ourselves to $\{f_m\}_{m=1}^M$ with the same information exponent $p$; this simplifies the optimization dynamics in that all directions are learned at roughly the same rate. To handle heterogeneous tasks with different information exponents, one may consider a student model with mixture of different nonlinearities. 
    % , or a more complicated multi-stage algorithm where parameters are re-initialized.
% \vspace{-1.8mm}
    \item We focus on the high information exponent setting $p>2$, which corresponds to target functions that are more ``difficult'' to learn via gradient descent. To handle lower information exponent $f_m$, we may employ a pre-processing procedure analogous to that in \cite{damian2022neural}: we first fit $f_*$ with a quadratic function, subtract it from the labels, and add the subtracted components back to the predictor after Algorithm~\ref{alg:main} is executed. 
\end{itemize}
\end{remark}

In addition to the above specification of the link function, we place the following diversity assumption on the index features $\{v_m\}_{m=1}^M$. 

\begin{assumption}[Task diversity]\label{assumption:diversity} We assume the following diversity condition on $\{v_m\}_{m=1}^M$: 
$$
M \leq c_v \max\left\{\big(\max_{m\ne m'} |v_m^\top v_{m'}|\big)^{-1}, \sqrt{d}\right\}, 
\quad \text{where~} c_v\asymp1/\mathrm{polylog}(d).
$$ 
\end{assumption}

\begin{remark}\,
% We make the following remarks on the diversity assumption. 
\begin{itemize}[leftmargin=*,topsep=0mm]
    \item This condition ensures that the single-index tasks are diverse, in that the index feature directions do not significantly overlap; similar assumption also appeared in prior works on gradient-based feature learning \citep{wang2023learning}. 
    When each $v_m$ is an independent sample from the $(d-1)$-dimensional unit sphere $\mathbb{S}^{d-1}$, via a standard concentration argument, Assumption~\ref{assumption:diversity} is satisfied with high probability for $M \asymp d^\gamma, \gamma\in[0,1/2)$. 
    \item The above assumption justifies the prefactor of $\frac{1}{\sqrt{M}}$ instead of $\frac{1}{M}$: 
    since $v_m$ are almost orthogonal and each $f_m(t)$ is assumed not to have the first Hermite coefficient, $f_m(v_m^\top x)$ are weakly dependent mean-zero variables.
    Thus the scaling prefactor should be $\frac{1}{\sqrt{M}}$ for the output scaling to be $\Theta(1)$ due to CLT. 
\end{itemize}
\end{remark}

\section{Complexity of Gradient-based Training}
\label{sec:GD}

\paragraph{Neural Network Architecture.} In this section we show that the additive model class specified in Section~\ref{sec:problem} can be efficiently learned via gradient-based training of a neural network. Specifically, we consider the following two-layer network 
with trainable parameters $\Theta=(a_j,w_j,b_j)_{j=1}^J\in \mathbb{R}^{(1+d+1)\times J}$:
\begin{align}
    f_{\Theta}= \frac{1}{J}\sum_{j=1}^J a_j \sigma_j (w_j^\top x+b_j).\label{eq:TwoLayerNN}
\end{align}
For each neuron, we define the Hermite expansion of $a_j \sigma_j (\cdot+b_j)$ as $a_j \sigma_j (\cdot+b_j) = \sum_{i=0}^\infty \beta_{j,i}\He_i(\cdot)$; note that the Hermite coefficient $\beta_{j,i}$ may differ across neurons. 
To ensure a descent path from weak recovery to strong recovery, we need the following technical condition on the Hermite coefficients: for each task $f_i$, there exist some neurons $j$ such that
\begin{align}
    \alpha_{m,i}\beta_{j,i} > 0 \text{~for $i=p$, ~ and ~} 
    \alpha_{m,i}\beta_{j,i} \ge  0\text{~for $p<i\le q$.}
    \label{eq:Hermite-sign}
\end{align} 
Note that \eqref{eq:Hermite-sign} is automatically satisfied in the well-specified setting, i.e., the student and teacher models share the same nonlinearity as in \cite{arous2021online}. In the misspecified setting, such condition has been directly assumed in \cite{mousavi2024gradient}. We defer further discussions to Appendix~\ref{subsection:DescentPath}. 
Below, we give two concrete settings where \eqref{eq:Hermite-sign} hold despite the link mismatch. 
\begin{assumption}[Activation function]
We consider the nonlinearity of the student model \eqref{eq:TwoLayerNN} and the link functions $\{f_i\}_{i=1}^M$ that satisfy one of the following: 
\begin{enumerate}[leftmargin=*,topsep=0mm,itemsep=0mm]
        \item $\sigma_j$ is a randomized polynomial activation defined in Appendix~\ref{subsubsection:DescentPath-2}, and $f_i$ satisfies Assumption \ref{assump:additive}. 
    \item $\sigma_j$ is the ReLU activation function, and we additionally require that for each $f_i$, all the non-zero Hermite coefficients $\alpha_{m,i}$ have the same sign. 
\end{enumerate}
\label{assumption:activation}
\end{assumption}
In both settings, we utilize the ``diversity'' of student nonlinearities to deduce that when the network width $J$ is sufficiently large (quantified in Theorem~\ref{theorem:NN-main}), a subset of neurons can achieve alignment with each target task even though the link functions are unknown. Similar use of overparameterization and diverse activation functions to handle link misspecification also appeared in \citep{bietti2022learning,ba2023learning}. 

\begin{algorithm}[t]
\SetKwInOut{Input}{Input}
\SetKwInOut{Output}{Output}
\SetKwBlock{StageOne}{Phase I: normalized SGD on first-layer parameters}{}
\SetKwBlock{StageTwo}{Phase II: Convex optimization for second-layer parameters}{}
\SetKw{Initialize}{Initialize}
	
\Input{Learning rates $\eta^t$, regularization parameter $\lambda$, sample size $T_1$, $T_2$,   %and $T_2$, 
   initialization scale $C_b$. }

\Initialize{$w^0_j\sim\mathrm{Unif}(\mathbb{S}^{d-1}(1))$, $a_j \sim \mathrm{Unif}\{\pm 1\}$.}
% $a_j \sim \mathrm{Unif}\{\pm 1\}$. }

\StageOne{
\For{$t=0$ {\bfseries to} $T_1-1$}{
Draw new sample $(x^t,y^t)$. \\
$w^{t+1}_j \leftarrow w^t_j+ \eta^t y^t \tilde{\nabla}_w f_{(a_j,b_j,w^t_j)_{j=1}^J}(x^t)$, \\ 
$w^{t+1}_j \leftarrow w^{t+1}_j / \|w^{t+1}_j\|, \quad (j=1,\dots,J)$.
}
}	
% \Initialize{$b_j\sim \mathrm{Unif}([-C_b,C_b])$. }
\Initialize{$b_j\sim \mathrm{Unif}([-C_b,C_b])$, and let $\hat{w}_j \leftarrow \delta_j w_j^{T_1}\ (\delta_j \sim \mathrm{Unif}(\{\pm 1\}))$}

\StageTwo{
Draw new samples $(x^t,y^t)_{t=T_1}^{T_1+T_2-1}$. \\
$\hat{a} \leftarrow \argmin_{a\in\R^J}\frac{1}{T_2}\sum_{t=T_1}^{T_1+T_2-1} \big(f_{(a_j,b_j,\hat{w}_j)_{j=1}^J}(x^t)-y^t\big)^2 + \bar{\lambda}\|a\|_{r}^r$, \quad ($r=1$ or $2$).
}
 
\Output{Prediction function $x\to f_{\hat{\Theta}}(x)$ with $\hat{\Theta}=(\hat{a}_j,\hat{w}_j,b_j)_{j=1}^{J}$.}
\caption{Gradient-based training of two-layer neural network} 
 \label{alg:main}
\end{algorithm}

\paragraph{Optimization procedure.}
The training algorithm is presented in Algorithm \ref{alg:main}.
First, the first-layer parameters are trained to minimize the correlation loss $\mathcal{L}=-yf_{\Theta}(x)$. 
We use the correlation loss to ignore the interaction between neurons during the first-layer training; analogous strategies appeared in prior analyses of feature learning, for example, by the “one-step” analysis \citep{ba2022high,damian2022neural} or by considering the squared loss with sufficiently small second-layer initialization \citep{abbe2023sgd}. Similar to \cite{arous2021online,damian2023smoothing}, we use the \textit{spherical gradient} $\tilde{\nabla}_w f(w) := (I - ww^\top) \nabla_w f(w)$, where $\nabla_w$ denotes the Euclidean gradient. 
We note that for $\sigma=\text{ReLU}$, we require an additional sign randomization step at the end of Phase I, to enhance the expressivity of the feature map for the second-layer training. 

After $T_1$ steps, we can show that for each sub-problem there exist some students neurons aligning with the class direction $v_m$. Then, we train the second-layer parameters (which is a convex problem) with either of $L^2$ or $L^1$ regularization. 
While ridge regression is easier to implement, $L^1$-regularization gives the better generalization error due to the induced sparsity; more specifically, because not all neurons are aligned with $f_i$ after first-layer training, and redundant neurons might increase the complexity of the network, we need $L^1$ regularization to efficiently single out neurons that succeeded in aligning with one of the single-index tasks.  
 
Now we present our main theorem for gradient-based training of neural networks.

\begin{theorem}\label{theorem:NN-main}
    Under Assumptions~\ref{assump:additive} \ref{assumption:diversity}, and \ref{assumption:activation}, take the number of neurons $J=\tilde{\Theta}(M^{C_p+\frac12}\varepsilon^{-1})$, the number of steps for first-layer training $T_{1}=\tilde{\Theta}(Md^{p-1}\lor Md\varepsilon^{-2}\lor M^\frac52\varepsilon^{-3})$, and the number of steps for second-layer training $(i)$ when $r=2$ (ridge)
    $T_2=\tilde{\Theta}(M^{C_p}\varepsilon^{-2})$, and $(ii)$ when $r=1$ (LASSO) $T_2=\tilde{\Theta}(M^{1+\iota}\varepsilon^{-2(1+\iota)})$ for arbitrary fixed $\iota>0$.  
    Then, under appropriate choices of $\eta^t$ and $\lambda$, with probability $1-o_d(1)$ over the randomness of dataset and parameter initialization,  Algorithm \ref{alg:main} outputs $f_{\hat{\Theta}}(x)$ such that
    \begin{align}
       \mathbb{E}_{x\sim \mathcal{N}(0,I_d)} \left[\big|f_*(x) - f_{\hat{\Theta}}(x)\big| \right]\leq \varepsilon.
    \end{align}
    % with high probability. 
\end{theorem}
Due to the online SGD update, the runtime complexity $T_1+T_2$ directly translates to a sample complexity of $n = \tilde{\mathcal{O}}(Md^{p-1})$ (focusing on the term whose exponent depends on the information exponent $p$). 
%As for the required width $J$, note that LASSO yields a better polynomial dependence than ridge regression.
We make the following remarks on the obtained sample complexity. 
\begin{itemize}[leftmargin=*]
    \item \textbf{Comparison with prior results.} For single-index model with known link function, \citet{arous2021online} proved that the online SGD algorithm learns the degree-$p$ Hermite polynomial with $\tilde{\mathcal{O}}(d^{p-1})$ samples; whereas for the \textit{misspecified} setting (generic and unknown $f_m$), a sample complexity of $n=\tilde{\mathcal{O}}(d^p)$ has been established in \citep{bietti2022learning}. Our bound with $M=1$ matches the $\tilde{\mathcal{O}}(d^{p-1})$ complexity in \citep{arous2021online} despite the link misspecification (under the additional restriction that $f_m$ is polynomial as in \cite{damian2022neural}). 
    \item \textbf{Superiority over kernel methods.} By a standard dimension argument \citep{kamath2020approximate,hsu2021approximation,abbe2022merged}, we know that kernel methods (including neural networks in the lazy regime \cite{jacot2018neural}) require $n\gtrsim d^q$ samples to learn degree-$q$ polynomials in $\R^d$. 
    Since $M<d$ and $p\le q$, the sample complexity of gradient-based learning is always better than that achieved by kernels, due to the presence of \textit{feature learning}.
    \item \textbf{The role of additive structure.} Without structural assumptions, learning an $M$-index model of degree $q$ requires $\Omega(M^{q})$ samples. 
    In contrast, in our bound the exponent in the dimension does not depend on $q$, and the exponent of $M$ is independent of $p$ and $q$; hence we achieve better complexity when $q$ is large and $M$ diverging with dimensionality. This illustrates the benefit of the additive structure in Assumption~\ref{assump:additive}. 

\end{itemize}

\subsection{Outline of Theoretical Analysis}

\subsubsection{Training of the first layer}\label{subsection:Main-FirstLayerTraining}

We provide a proof sketch of the first-layer training, where the goal is to show that starting from random initialization, a subset of neurons will achieve significant alignment with one of the target directions $v_m$. The following lemma establishes that while the initial correlation with any target direction is small, student neurons may have a constant-factor difference in the magnitude of alignment. 

\begin{lemma}\label{lemma:Appendix-GB-Initialization}
    Consider the network per Algorithm \ref{alg:main} and the hyperparameters per Theorem~\ref{theorem:NN-main}.
    Then with high probability, for each $m$, there exist at least $J_{\min} =\tilde{\Omega}(M^\frac12\varepsilon^{-1})$ neurons at random initialization satisfying
    \begin{align}
        w_j^\top v_m\geq \max_{m'\ne m} %C^{\frac{2}{p-2}}
        \left|\frac{\beta_{m,p}}{\beta_{m',p}}\right|^{\frac{1}{p-2}}
        |w_j^\top v_{m'}| + \tilde{\Omega}(d^{-1/2}). 
        \label{eq:Main-InitializationDifference}
    \end{align}
    % with high probability. 
\end{lemma}
% We denote by $J_m$ the neurons that are most aligned with $v_m$ (though its scale is $O(\frac{1}{\sqrt{d}})$).
Next we show that this small difference in the initial alignment is amplified during the online SGD update, so that the student neurons will eventually specialize to the target direction $v_m$ with largest overlap at random initialization. 
Consider the dynamics of one neuron $w_j$, the update of which is written as
\begin{align}\textstyle
    w^{t+1}_j = \frac{w^t_j+ \eta^t_j y^t a_j \sigma' ({w_j^t}^\top x^t + b) (I - w_j^t {w_j^t}^\top)x^t}{\|w^t_j+ \eta^t_j y^t a_j \sigma' ({w_j^t}^\top x^t + b) (I - w_j^t {w_j^t}^\top)x^t\|}
    % \\ & =w^t_j + \eta^t_j \mathbb{E}\bigg[\frac{1}{\sqrt{M}}\sum_{m=1}^M f_m(v_m^\top x)a_j \sigma' ({w_j^t}^\top x + b^t) x\bigg] +  Z^t
     \approx w^t_j+\eta^t_j 
    \frac{1}{\sqrt{M}}
    \sum_{m=1}^{M} \alpha_{j,p}\beta_{m,p} (v_m^\top w_j^t)^{p-1}v_m+Z^t,
\end{align}
where $Z^t$ is a mean-zero random variable corresponding to the SGD noise, and the approximation step is due to Hermite expansion in which we ignored the effect of normalization for simplicity.
For projection onto a specific target direction $v_m$, we have the following estimate, 
\begin{align}
    v_m^\top w^{t+1}_j \approx v_m^\top w^t_j + \eta^t_j \frac{\alpha_{j,p}\beta_{m,p}}{\sqrt{M}} (v_m^\top w_j^t)^{p-1}+\eta^t_j  v_m^\top Z^t.
\end{align}
Following \citet{arous2021online}, by using the optimal step size $\eta_j^t  = \tilde{\Theta}(M^{-\frac12}d^{-\frac{p}{2}})$, 
the SGD dynamics approximately follows the ODE: 
\begin{align}
   \frac{\mathrm{d}}{\mathrm{d}t} v_m^\top w^{t}_j \approx \eta\frac{\alpha_{j,p}\beta_{m,p}}{\sqrt{M}} (v_m^\top w_j^t)^{p-1} \quad\Rightarrow\quad v_m^\top w^{t}_j \approx \frac{v_m^\top w^{0}_j}{\left(1-\frac{\eta(p-2)\alpha_{j,p}\beta_{m,p}(v_m^\top w^{0}_j)^{p-2}}{\sqrt{M}}t\right)}.
\end{align}
This means that $v_m^\top w_j^t\approx d^{-1/2}$ for a long period, but the SGD dynamics rapidly escapes the high-entropy equator and achieves nontrivial overlap (i.e., weak recovery) just before the critical time $t_j := \sqrt{M}\eta^{-1}(p-2)^{-1}\alpha_{j,p}^{-1}\beta_{m,p}^{-1}(v_m^\top w^{0}_j)^{-(p-2)} = \tilde{O}(Md^{p-1})$, as specified by the following lemma. 
\begin{lemma}\label{lemma:AlignmentFirstPhase}
    %Consider the case when $p\geq 3$.
    Consider the SGD dynamics with $\eta^t_j = \tilde{\Theta}(M^{-\frac12}d^{-\frac{p}{2}})$.
    There exists some time $t_0=\tilde{O}(Md^{p-1})$ and small constant $c=\tilde{\Theta}(1)$, such that 
    $\langle w_j^{t_0}, v_m\rangle\ge c$, and $|v_m^\top w_j^{t_0}|\leq c M^{-\frac12}$ for all $m=2,\dots,M$. 
\end{lemma}
After a nontrivial overlap with $v_m$ is obtained, we continue to run online SGD to further amplify the alignment. A technical challenge here is that as the student neuron start to develop alignment with one direction $v_m$, the influence from the other $m-1$ directions is no longer negligible, while the signal from $v_m$ gets smaller because of the projection $(1-w^t{w^t}^\top)$; this complication is addressed in Lemma~\ref{lemma:AlignmentSecondPhase}.

Finally, in Lemma~\ref{lemma:AlignmentFinalPhase} we prove a local convergence result (analogous to \cite{zhou2021local,akiyama2021learnability}) which entails that for each task direction $v_m$, there exist some student neurons that that localize to the task, i.e., $\langle v_m, w_j\rangle > 1-\varepsilon$. Here we exploit the local convexity (more specifically, {\L}ojasiewicz condition) of the loss landscape due to the small overlap between the single-index tasks as specified in Assumption~\ref{assumption:diversity}. 
The following proposition describes the configuration of parameters after first-layer training. 
Here, $\tilde{\varepsilon}$ indicates the level of alignment, which is later set to $\tilde{\varepsilon}=\tilde{\Theta}(M^{-\frac12}\varepsilon)$, where $\varepsilon$ is the desired final generalization error $\mathbb{E}_{x}[|f_{\hat{a}}(x) - f_*(x)|] \lesssim \varepsilon$.
\begin{lemma}[Informal]\label{lemma:first-layer-training}
    Take $T_{1,1}=\tilde{\Theta}(Md^{p-1})$, $T_{1,2}=\tilde{\Theta}(Md^{\frac{p}{2}})$, $T_{1,3}=\tilde{\Theta}(\tilde{\varepsilon}^{-2}Md\lor M\tilde{\varepsilon}^{-3})$, and the number of neurons as $J\gtrsim J_{\rm min}M^{C_p}\log d$.
    Suppose that $|v_{m'}^\top v_m|= \tilde{O}(M^{-1})$ for all $m'\ne m$, and $M= \tilde{O}(d^{1/2})$.
    Then under appropriate learning rate $\eta^t$, with high probability, %there exists some time $T_1=\tilde{O}(\tilde{\varepsilon}^{-1}M^\frac12\eta^{-1})$ such that 
    for each class $m$, there exist at least $J_{\rm min}$ neurons that achieve strong recovery $v_m^\top w_j^{T_1} \geq 1-\tilde{\varepsilon}$.
\end{lemma}

\subsubsection{Training of the second layer}
Once localization to each task is achieved, we optimize the second-layer parameters with $L^2$- or $L^1$-regularization. Because the objective is convex, gradient-based optimization can efficiently minimize the empirical loss. The learnability guarantee follows from standard analysis analogous to that in \citet{abbe2022merged,damian2022neural,ba2022high}, where we construct a ``certificate’' second-layer $a^*\in\R^J$ that achieves small loss
$$
\textstyle \E_{x\sim\mathcal{N}(0,I_d)}\left(f_*(x) - \frac{1}{J}\sum_{j=1}^J a^*_j\sigma(w_j^\top x+b_j)\right)^2 \le \varepsilon
$$
for $r=1,2$. Then, the population loss of the regularized empirical risk minimizer can be bounded via standard Rademacher complexity argument, using the norm $\|a^*\|_r^r$.

Observe that in Theorem~\ref{theorem:NN-main}, by using $L^1$-regularization, we can avoid the dependency on $C_p$ determined by the Hermite coefficients of $f_m$. 
This is because if some $f_m$ is too large, more students neurons are required to satisfy the initial alignment condition (Lemma~\ref{lemma:Appendix-GB-Initialization}). Since not all neurons are guaranteed to align with the target directions, 
the sparsity-inducing $L^1$-regularization allows us to ignore the redundant neurons and hence obtain better generalization error rate.  

\paragraph{Implication for fine-tuning.} $L^1$-regularization is also useful for efficient fine-tuning on a different downstream task, in which a subset of learned features of size $\tilde{M}$ is used to define the target function: $\tilde{f}_*(x) = \tilde{M}^{-1/2}\sum_{m=1}^{\tilde{M}} g_m(v_m^\top x)$, where $g_1,g_2,...,g_m$ are link functions that may differ from the training observations. In this setting, 
retraining the second-layer with $L^1$-regularization requires $\tilde{O}(\tilde{M}\varepsilon^{-2})$ sample to achieve an $L^1$-error of $\eps$, which is especially useful when the number of the downstream features is small  $\tilde{M}\ll M$. 
This indeed has connections to practical fine-tuning where sparsity is induced to extract relevant features, as seen in localization of skills  \citep{dai2021knowledge,wang2022finding,panigrahi2023task} (see also \cite{elhage2022toy,bricken2023towards} in the context of mechanistic interpretability) and LoRA \citep{hu2021lora}. 

\begin{remark}   
Prior theoretical results on fine-tuning \& transfer learning for downstream tasks typically showed that gradient-based pretraining ``localizes'' the neural network parameters to a low-dimensional subspace spanned by the target functions (tasks), which enables efficient adaptation \cite{du2020few,damian2022neural,collins2023provable}; for instance, fine-tuning for a degree-$q$ polynomial task in $M$-dimensional subspace requires $n\gtrsim M^q$ samples \cite{damian2022neural}.
However, when the collection of tasks becomes sufficiently diverse (i.e.~$M$ diverges with dimensionality), which is the relevant regime for large-scale pretraining, the benefit of low-dimensional representation diminishes. In contrast, in our setting we prove a different kind of \textit{localization}, where each neuron aligns with a different $v_m$ depending on the initialization; this leads to efficient fine-tuning despite $M$ being large. 
\end{remark}

\subsection{Numerical Illustration}  

We conduct numerical experiments to illustrate the feature learning process for additive models via gradient descent on a two-layer network. 
The target function is an additive model~\eqref{eq:intro-f*}, where we set $M=16$ and $d=64$, $f_m(x)=\He_3(x)$, and $\{v_m\}_{m=1}^M$ are the canonical basis vectors. 
% The true signal $v_m$ is set to be the vector whose $m$-th entry is one and the others are zero.
The student network is a two-layer ReLU network~\eqref{eq:TwoLayerNN}, where $J=8192$ and weights are initialized as $w_j^0\sim \mathbb{S}^{d-1}$, $a_j^0\sim \mathrm{Unif}\{\pm1\}$, and $b_j^0\sim \mathrm{Unif}([-1,1])$, respectively.
First, we train the first-layer parameters via online SGD for $T=10^6$ steps. The initial step size is chosen as $\eta^0=0.3$, and from $T'=T/2$, we anneal the step size as $\eta^t=\frac{\eta^0}{(t/T')^2}$.

\begin{figure}[!htb]
% \vskip 0.2in
\vspace{-1mm}
\centering
\begin{minipage}[t]{0.42\textwidth}
        \centering
\includegraphics[width=0.84\textwidth]{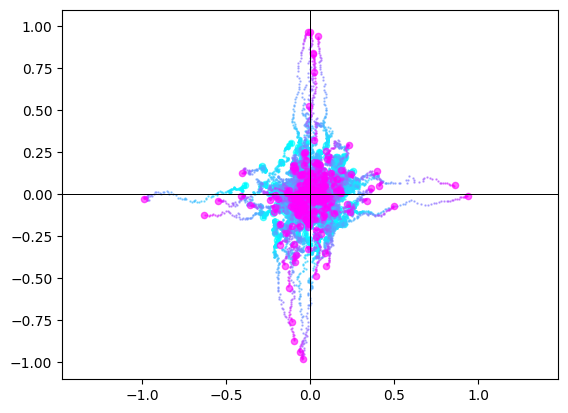}
    \end{minipage}
    % \hfill 
    \begin{minipage}[t]{0.42\textwidth}
        \centering
        \includegraphics[width=0.84\textwidth]{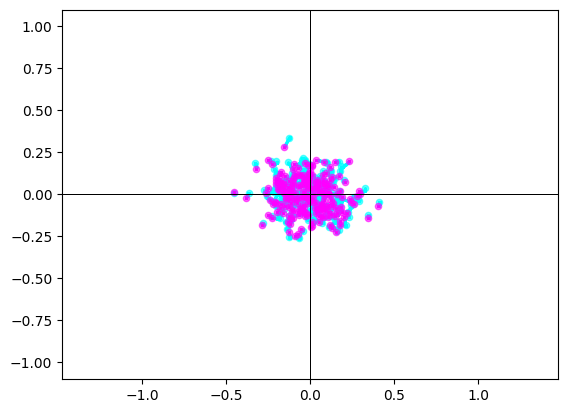}
    \end{minipage}
    \vspace{-1mm}
    \caption{\small Alignment between student neurons and true signals $v_1$ and $v_2$, before (blue) and after (purple) SGD training.  Left: neural network optimized by online SGD (Algorithm~\ref{alg:main}), Right: neural network in the NTK regime.}
    \label{fig:experiment}
\vspace{1.5mm}
\end{figure}

In Figure~\ref{fig:experiment} we plot the alignment between the student neurons and the true signal: the horizontal axis represents the alignment between $w_j$ and $v_1$, that is, $\langle w_j,v_1\rangle/\|w_j\|$, and the vertical axis represents $\langle w_j,v_2\rangle/\|w_j\|$.  
For comparison, we also train a network in the NTK regime~(\citet{jacot2018neural}), where the magnitude of second layer is $\Theta(1/\sqrt{J})$ and feature learning is suppressed \citep{yang2020feature}. 
We observe that, via Algorithm \ref{alg:main}, a subset of student neurons almost perfectly aligned with one of the true signal directions (see Left figure). 
This is in sharp contrast with the NTK model (Right figure) where no alignment is observed.

\section{Statistical Query Lower Bounds}
\label{sec:SQ}

In Section~\ref{sec:GD} we showed that two-layer ReLU network optimized via online SGD can learn a representative subset of the additive model class \eqref{eq:intro-f*} with polynomial sample complexity. 
This being said, due to the introduced diversity condition (Assumption~\ref{assumption:diversity}), it is not clear if the resulting function class is still intrinsically hard to learn. 
To understand the fundamental complexity of the problem we are addressing, in this section we present several computational lower bounds.
Specifically, we consider the statistical query learner \citep{kearns1998efficient,bshouty2002using,diakonikolas2020algorithms}, which submits a function of $x$ and $y$, and receives its expectation within a tolerance $\tau$.
The question we ask is how many accurate queries are needed to learn the target function.

\subsection{Correlational Statistical Query}
First, we derive lower bounds for the correlational statistical query (CSQ) learner, which is an important subclass of SQ learners often discussed in the context of feature learning of neural networks \citep{damian2022neural,abbe2023sgd}.
For a function $g\colon \mathcal{X}\to \mathbb{R}$ with $\|g\|_{L^2}=1$ and parameter $\tau$,
a correlational SQ oracle $\mathrm{CSQ}(g,\tau)$ returns 
$
    \mathbb{E}_{x,y}[yg(x)]+\varepsilon,
$
where $\varepsilon$ is an arbitrary (potentially adversarial) noise that takes value in  $\varepsilon\in [-\tau,\tau]$.

The CSQ learner can model gradient-based training as follows: for a neural network $f_{\Theta}$, the SGD update of the parameters with the minibatch size $b$ using the squared loss is written as 
\begin{align}
    \Theta^{t+1} \leftarrow \Theta^{t}-\eta^t \nabla_{\Theta}\frac{1}{b}\!\sum_{i=1}^b(y^t_i\!-\!f_{\Theta^{t}}(x^t_i))^2 
    = \Theta^{t}-\frac{2\eta}{b}\sum_{i=1}^by^t_i\nabla_{\Theta}f_{\Theta^{t}}(x^t_i)
    +\frac{2\eta}{b}\sum_{i=1}^bf_{\Theta^{t}}(x^t_i)\nabla_{\Theta}f_{\Theta^{t}}(x^t_i).
\end{align}
Here, the network gains information of the true function from the correlation term $\frac{2\eta}{b}\sum_{i=1}^by^t_i\nabla_{\Theta}f_{\Theta^{t}}(x^t_i)$. 
Roughly speaking, since each coordinate of the gradient concentrates around the expectation with $O(\frac{1}{\sqrt{b}})$ fluctuation, the noisy correlational query can be connected to gradient-based training by matching the tolerance $\tau$ with the concentration error $\frac{1}{\sqrt{b}}$. 
We note that there is a gap between CSQ and SGD update due to the different noise structure, see \cite[Remark 6]{abbe2023sgd}; nevertheless, a lower bound on CSQ learner serves as a baseline comparison for the sample complexity of SGD learning in recent works \citep{damian2022neural,abbe2022non,damian2023smoothing}. 

We obtain the following CSQ lower bounds with different dependencies on the target $L^2$ error and the number of queries.
Note that both bounds imply the sample complexity of $\tilde{\Omega}(M d^{\frac{p}{2}})$ under the standard $\tau\approx n^{-1/2}$ concentration heuristic. 
\begin{theorem}[CSQ lower bound]\label{theorem:CSQ-Appendix}
For any $p\geq 1$, $\varsigma>0$ and $C>0$, there exists a problem class $\mathcal{F}\subset \mathcal{F}_{d,M,\varsigma}^{p,p}$ satisfying Assumptions \ref{assump:additive} and \ref{assumption:diversity} such that a CSQ learner using $Q$ queries (outputting $\hat{f}$) requires the following to learn a random choice $f_*\sim \mathcal{F}$: 
\begin{itemize}
    \item[(a)] tolerance of 
    \vspace{-3mm}
    \begin{align}
        \tau \lesssim \frac{(\log Qd)^{\frac{p}{4}}}{M^{\frac12}d^{\frac{p}{4}}},
    \end{align}
    % \vspace{-1mm}
    otherwise, we have $\|f_* - \hat{f}\|_{L^2}^2 \gtrsim 1/M$ with probability at least $1-O(d^{-C})$. 
    \item[(b)] tolerance of 
    \vspace{-3mm}
    \begin{align}
       \tau \lesssim \frac{Q^{\frac12} (\log dQ)^{\frac{p}{4}+1}}{M^{\frac12}d^{\frac{p}{4}}},
    \end{align}
    % \vspace{-1mm}
    otherwise, we have $\|f_* - \hat{f}\|_{L^2}^2 \gtrsim 1$ with probability at least $1-O(d^{-C})$. 
\end{itemize}
\end{theorem}
Recall that the CSQ lower bound for single-index model translates to a sample complexity of $n\gtrsim d^{\frac{p}{2}}$; intuitively, 
our theorem thus implies that for a CSQ algorithm, the difficulty of learning one additive model with $M$ tasks is the same as learning $M$ single-index models separately. 
To establish this lower bound, we construct a function class $\mathcal{F}$ by choosing $f_m = \He_p$ and randomly sampling the directions $v_m$ from some set on sphere $S\in \mathbb{S}^{d-1}$. 
Here we briefly explain the difference between the two statements, both implying the same dependence on the number of tasks $M$ and ambient dimensionality $d$, but the target $L^2$ error and dependence on the number of queries $Q$ differ. 
\begin{remark}
The target error of the first lower bound is $\Omega(\frac{1}{M})$, and the dependency on the number of queries is logarithmic; this is obtained by a straightforward extension of \citet{damian2022neural} which handles single-index $f_*$. 

On the other hand, establishing the latter bound with $\Omega(1)$ error, which implies the failure to identify a constant fraction of $\{v_1,\cdots,v_M\}$, is nontrivial. This is because in the additive setting $\frac{1}{\sqrt{M}}\sum_{m=1}^M\frac{1}{\sqrt{p!}}\He_p(v_m^\top x)$, the query is not known beforehand, and hence the (adversarial) oracle need to simultaneously prevent the learning of as many directions from $v_1,v_2,...,v_M$ as possible; whereas in the single-index setting $\frac{1}{\sqrt{p!}}\He_p(v_1^\top x)$, the oracle only need to ``hide'' one direction. Consequently, we cannot directly connect the identification of $\Omega(1)$-fraction of target directions in the additive setting to the CSQ lower bound for single-index model. 
To overcome this issue, we introduce a sub-class of CSQ termed noisy correlational statistical query, which adds a random (instead of adversarial) noise to the expected query. 
For this query model, proving the lower bound for learning single-index models can indeed be translated to our additive setting.
However, because the noise is random, the dependency on the number of queries no longer logarithmic but $Q^{\frac12}$. 
\end{remark}

\subsection{Full Statistical Query}
Going beyond CSQ, in this subsection we present a lower bound for full SQ algorithms. 
For a function $g\colon \mathcal{X}\times \mathcal{Y}\to [-1,1]$ and parameter $\tau$,
a full statistical query oracle $\mathrm{SQ}(g,\tau)$ returns 
$
    \mathbb{E}_{x,y}[g(x,y)]+\varepsilon,
$
where $\varepsilon$ is an arbitrary noise that is bounded by the tolerance $\tau$. 
Note that unlike the CSQ learner, here the query function can apply transformations to the target label $y$ and hence reduce the computational complexity. 

The full SQ setting is partially motivated by the observation that for sparse polynomials with a constant number of relevant dimensions, a more sample-efficient SQ algorithm that departs from gradient-based training is available: 
\citep{chen2020learning} showed that general multi-index polynomials of degree $q$ can be learned with $\tilde{O}(d)$ samples. 
The core of their analysis is a transformation of the label $y$ that reduces the leap complexity \citet{abbe2023sgd} to at most $2$ (similar transformations also appeared in context of phase retrieval \citep{mondelli2018fundamental,barbier2019optimal}); this enables a warm-start from the direction that overlaps with the relevant dimensions of $f_*$, from which point projected gradient descent can efficiently learn the target.

One might therefore wonder if the polynomial dimension dependence in Theorem~\ref{theorem:CSQ-Appendix} is merely an artifact of restriction to correlational queries. 
However, we demonstrate that for general SQ algorithms, our additive model class is also challenging to learn despite the additive structure and near-orthogonality constraint. 
We consider the scaling where $M$ increases with $d$, i.e., $M\asymp d^\gamma$ with some $\gamma>0$, and show that the dimension exponent in the statistical complexity can be arbitrarily large for any fixed $\gamma$, by varying $p,q=O_d(1)$.
\begin{theorem}[SQ lower bound]\label{theorem:SQ-Appendix}
    Fix $0<\gamma\leq p$ and $\varsigma>0$ arbitrarily, and consider the number of tasks $M=\Theta(d^\gamma)$.
    For any $\rho>0$, there exist constants $p,q=O_d(1)$ depending only on $\gamma$ and $\rho$, and a problem class $\mathcal{F}\subset\mathcal{F}_{d,M,\varsigma}^{p,q}$ satisfying Assumptions \ref{assump:additive} and \ref{assumption:diversity} for which 
    an SQ learner outputting $\hat{f}$ from $Q$ queries requires the following tolerance to learn a random choice of $f_*\sim \mathcal{F}$, 
    \begin{align}
        \tau \lesssim d^{-\rho}, 
    \end{align}
    otherwise, we have $\|f_* - \hat{f}\|_{L^2}^2 \gtrsim 1$ with probability at least $1 - Q e^{-\Omega(\sqrt{d})}$.
\end{theorem}

This lower bound illustrates the intrinsic difficulty of our large-$M$ setting compared to the well-studied low-dimensional polynomial regression. 
For the latter, gradient-based training of polynomial-width network can achieve low loss using $n=\Omega(d^{p-1})$ samples \citep{abbe2023sgd}, and the CSQ lower bound gives a $\Omega(d^{p/2})$ complexity \citep{damian2022neural}; however, by applying nonlinear transformation to the labels, an efficient SQ algorithm can achieve $\tilde{O}(d)$ complexity \citep{chen2020learning}. In contrast, while our additive model has $O(Md)$ total parameters to be estimated, the dimension dependence in the SQ lower bound can be arbitrarily larger than  $M$ times the single-index complexity. Therefore, unlike the CSQ lower bound (Theorem~\ref{theorem:CSQ-Appendix}), Theorem~\ref{theorem:SQ-Appendix} suggests that for an SQ algorithm, learning one additive model with $M$ tasks is more difficult than learning $M$ single-index models separately. 
See Section~\ref{sec:stat-comp-gap} for more discussions. 

\begin{remark}
A concurrent work \citep{damian2024computational} established an SQ lower bound for single-index model that implies a sample complexity of $n\gtrsim d^{k^*/2}$, where $k^*$ is a prescribed \textit{generative exponent} that can be made arbitrarily large. While this also entails the existence of link functions not learnable by SQ algorithms with $\tilde{O}(d)$ samples, the source of computational hardness is fundamentally different than our setting: \cite{damian2024computational} constructed ``hard'' non-polynomial link functions that preserve high information exponent after arbitrary transformations; 
in contrast, we consider \textit{polynomial} link functions which have generative exponent $k^*\le 2$, so the lower bound relies on the additive structure with a large number of tasks $M=\omega(1)$. 
\end{remark}

\subsubsection{Outline of Theoretical Analysis} 
% In order to prove the theorem, the argument is reduced into the following lemma.
In the proof of the CSQ lower bounds, Hermite polynomials were used to construct the ``worst-case'' target functions \citet{damian2022neural}. Intuitively, this is because higher-order Hermite polynomials have small $L^2$-correlation under Gaussian input, and hence correlational queries are ineffective.
For the SQ learner, which goes beyond correlational queries, a hard function class should also ``hide'' information from nonlinear transformations to the target labels. As we will see, for our additive model, the class of target functions should maintain orthogonality with respect to queries after polynomial transformations. 
We refer to this property as \textit{superorthogonality}, and the following proposition shows the existence of such functions.
\begin{proposition}[Superorthogonal polynomials]\label{prop:superorthogonality}
    For any $K$ and $L$, 
    there exists a polynomial $f:\R\to\R$ that is not identically zero and satisfies the following:
    \begin{align}
        \int (f(x))^k \He_l(x) e^{-x^2/2}\mathrm{d}x = 0,
    \end{align}
    for every $1\leq k\leq K$ and $1\leq l\leq L$.
\end{proposition}
Because the proof is not constructive, it is difficult to determine the specific form of $f$ for general $K$ and $K$. Below, we present specific examples for small $K,L$:
\begin{example}
    Examples of link function $f$ in Proposition~\ref{prop:superorthogonality} are given as follows.
    \begin{itemize}
        \item[(i)] For $K=1$ and $L\in \mathbb{N}$, $f(x)=\He_{L+1}(x)$.
        \item[(ii)] For $K=L=2$, $f(x) = \He_4(x)- \frac{4}{15}\He_6(x) + \frac{11}{280}\He_8(x)-\frac{19}{4725}\He_{10}(x)+\frac{311}{997920}\He_{12}(x)-\frac{719}{37837800}\He_{14}(x)+\frac{14297}{15567552000}\He_{16}(x)-\frac{35369}{1042053012000}\He_{18}(x)+
        % (\frac{562273}{703964701440000}-\frac{1}{9262693440000}\sqrt{\frac{11993978}{11181942120404201}})
        (\frac{35369}{41682120480000}-\frac{1}{83364240960000}\sqrt{\frac{11163552839}{38}})
        \He_{20}(x)$.
    \end{itemize}
\end{example}
\begin{remark}
While the $K=1$ case follows from the orthogonality of Hermite polynomials, the $K=L=2$ case is already nontrivial --- it implies the existence of link functions with information exponent $p>2$, and furthermore, squaring the function cannot reduce the information exponent to 1 or 2; note that for $f(x) = \He_k(x)$, the information exponent of $f^2$ is at most 2. 
\end{remark}

In the lower bound construction we use the polynomial $f$ from the above lemma with suitably large $K$ and $L$.
Specifically, we let $f_m=f$ for all $m$ and sample $v_m$ from a set $S\subset \mathbb{S}^{d-1}$ with almost orthogonal elements.
We prove that the two target functions $\frac{1}{\sqrt{M}}\sum_{m=1}^M f(v_m^\top x)$ and $\frac{1}{\sqrt{M}}\sum_{m=1}^M f({v_m'}^\top x)$, where $\{v_m'\}_{m=1}^M$ is an independent copy of $\{v_m\}_{m=1}^M$, cannot be distinguished by SQ unless the tolerance is below $d^{-\rho}$.
To show this, we Taylor expand the target function and perform an entry-wise swapping of $f(v_m^\top x)$ with $f({v_m'}^\top x)$, in which smoothness of the function is entailed by additive Gaussian noise.
Specifically, for $\delta \ll 1$, we have
    \begin{align}
        \mathbb{E}_{\eps\sim\mathcal{N}(0,\varsigma^2)}[g(x,z+\delta+\varepsilon)]
        =
        \sum_{k=0}^K a_{k}(x,z) \delta^k + O(\delta^{K+1}),
    \end{align}
    where $a_{k}(x,z)=\frac{1}{k!}\int g(x,w) \left(\frac{\mathrm{d}^k}{\mathrm{d} z^k} e^{-(w-z)^2/2\varsigma^2}\right)\mathrm{d}w=O(1)$
(see Lemma~\ref{lemma:NoiseSmoothing} for derivation). 
Below we heuristically demonstrate how swapping is conducted for $m=1$:
\begin{align}
&\textstyle  \mathbb{E}\left[g\left(x,\frac{1}{\sqrt{M}}\sum_{m=1}^{M}f(v_{m}^\top x)+\varepsilon\right)\right]
   \\ \overset{(i)}{\approx}& \textstyle
\mathbb{E}\left[g\left(x,\frac{1}{\sqrt{M}}\sum_{m=2}^{M}f(v_m^\top x)+\varepsilon\right)\right]
    +
    \sum_{k=1}^K \mathbb{E}\left[a_{k}\left(x,\frac{1}{\sqrt{M}}\sum_{m=2}^{M}f(v_m^\top x)\right) \frac{f^k(v_1^\top x)}{M^{\frac{i}{2}}}\right]
\\ \overset{(ii)}{\approx}& \textstyle
\mathbb{E}\left[g\left(x,\frac{1}{\sqrt{M}}\sum_{m=2}^{M}f(v_m^\top x)+\varepsilon\right)\right]
    +
    \sum_{k=1}^K \mathbb{E}\left[a_{k}\left(x,\frac{1}{\sqrt{M}}\sum_{m=2}^{M}f(v_m^\top x)\right) \right] \frac{\mathbb{E}\left[f^k(v_1^\top x)\right]}{M^{\frac{i}{2}}}
% \\ & \approx
% \mathbb{E}_{\eps}\left[q\left(x,\frac{1}{\sqrt{M}}\sum_{m=1}^{M-1}f(v_1^\top x)+\varepsilon\right)\right]
%     +
%     \sum_{k=1}^K \mathbb{E}\left[a_{k}\left(x,\frac{1}{\sqrt{M}}\sum_{m=1}^{M-1}f(v_m^\top x)\right) \right] \frac{\mathbb{E}\left[f^k({v_1'}^\top x)\right]}{M^{\frac{i}{2}}}
 \\  \overset{(iii)}{\approx}& \textstyle
    \mathbb{E}\left[g\left(x,\frac{1}{\sqrt{M}}\sum_{m=1}^{M-1}f(v_{m}^\top x)+\frac{1}{\sqrt{M}}f({v_{1}'}^\top x)+\varepsilon\right)\right], 
\end{align}
where $(i)$ follows from Taylor expansion, $(ii)$ is due to $f^k$ being orthogonal to $\He_1,\dots,\He_L$ and hence its correlation to $a_k$ (which does not contain information of $v_1$) is approximated by its $\He_0$ component, and $(iii)$ is obtained by swapping $\mathbb{E}[f^k({v_{1}}^\top x)]$ and $\mathbb{E}[f^k({v_{1}'}^\top x)]$ and using the argument in reverse.

\subsubsection{Statistical-to-computational Gap}
\label{sec:stat-comp-gap}

The additive model \eqref{eq:intro-f*} contains $M$ directions $v_1,v_2,...v_M\in\R^d$ and $M$ univariate link functions $f_1,f_2,...f_M:\R\to\R$ to be estimated; therefore, we intuitively expect a sample size of $n\gtrsim Md$ to be information-theoretically sufficient to learn this function class. 
Indeed, following \cite{suzuki2018adaptivity}, it is easy to check that the covering number of width-$J$ neural network is $\log \mathcal{N}(\delta,\{\frac1J\sum_{j=1}^M a_j\sigma_j(w_j^\top \cdot +b_j)\},\|\cdot\|_\infty) \lesssim Jd(\log J + \log d)$.
If we take $J=Mq$ and let $\sigma_j=\He_i(\cdot ) \ ((q-1)M<j\leq qM)$, the network $\frac1J\sum_{j=1}^M a_j\sigma_j(w_j^\top \cdot +b_j)$ can perfectly approximate the additive model \eqref{eq:intro-f*}.
Therefore, by applying a standard generalization error bound (e.g., see \cite[Lemma 4]{schmidt2020nonparametric}), we can upper bound the squared loss by $\frac{1}{n}\log\mathcal{N}(\delta,\{\frac1J\sum_{j=1}^M a_j\sigma_j(w_j^\top \cdot +b_j)\},\|\cdot\|_\infty)$, which yields the following proposition (the detailed derivation of which we omit due to the standard proof).  
\begin{proposition}
    For the additive model \eqref{eq:intro-f*} with any $p$ and $q$, 
    there exists an (computationally inefficient) algorithm that returns a function $\hat{f}$ with $L^2$-error of
    $\varepsilon$ using 
    $n = \tilde{O}(Md\varepsilon^{-2})$ samples.
    \label{prop:optimal-algo}
\end{proposition}
The procedure in Proposition~\ref{prop:optimal-algo} involves finding the empirical risk minimizer of a neural network, which can be computationally infeasible in polynomial time \cite{blum1988training,arora2016understanding}. Nevertheless, the existence of a statistically efficient algorithm that learns \eqref{eq:intro-f*} in $n\gtrsim Md$ samples suggests a statistical-to-computational gap under the SQ class (with polynomial compute; see \cite{dudeja2022statistical,damian2024computational} for related discussions), since the lower bound in Theorem~\ref{theorem:SQ-Appendix} implies a worse sample complexity of $n\gtrsim (Md)^{\rho}$ for SQ learners, where $\rho$ can be made arbitrarily large by varying $p$ and $q$.

Importantly, the same statistical-computational gap is \textit{not present} in the finite-$M$ setting (in terms of dimension dependence), due to the restriction of $f_*$ being polynomial; specifically, when $M=O_d(1)$, there exists an efficient SQ algorithm that learns arbitrary multi-index polynomials using $n=\tilde{O}_d(d)$ samples \cite{chen2020learning}, and in the single-index setting, polynomial link functions have \textit{generative exponent} at most 2 \cite{damian2024computational}, and hence the SQ lower bound only implies that a sample size of $n\gtrsim d$ is necessary.  
In contrast, in our large-$M$ setting, the sample complexity of SQ algorithms may have large polynomial dimension dependence, despite the link functions being polynomial. 
The key observation is that when the number of tasks $M$ is diverging, the nonlinear label transformation cannot obtain higher-order exponentiation of the individual single-index tasks, which is employed by SQ learners such as \cite{mondelli2018fundamental,chen2020learning}; hence the information exponent of the link function may still be large after the nonlinear transformation.

\section{Conclusion and Future Directions}

In this work, we studied the learning of additive models, where the number of (diverse) single-index tasks $M$ grows with the dimensionality $d$.  We showed that a layer-wise SGD algorithm achieves $\tilde{O}(Md^{p-1})$ sample complexity, and a subset of first-layer neurons \textit{localize} into each additive task by achieving significant alignment. 
We also investigated the computational and statistical barrier in learning the additive model class by establishing lower bounds for both Correlational SQ and Full SQ learners. Our lower bound suggests a computational-to-statistical gap under the SQ class, which highlights the fundamental difference between our large $M$ setting and the previously studied finite-$M$ multi-index regression. 

We highlight a few future directions. 
First observe that there is a gap between our complexity upper bound for gradient-based training and the computational lower bounds, and one would hope to design more efficient algorithms that close such gap. It is also worth noting that in certain regimes, SGD on the squared loss may outperform the CSQ complexity due to the presence of higher-order (non-correlational) information, as demonstrated in very recent works \cite{dandi2024benefits,lee2024neural,arnaboldi2024repetita}; an interesting direction is to examine if similar mechanisms can improve the statistical efficiency of SGD in our setting. It is also important to relax the near-orthogonality condition for the true signals, and explore how our derived upper and lower bounds are affected.
Last but not least, we may theoretically analyze task localization in more complicated architectures beyond two-layer neural network, such as multi-head attention \cite{chen2024training}.

\bigskip

\subsection*{Acknowledgement}
The authors thank Alberto Bietti, Satoshi Hayakawa, Isao Ishikawa, Jason D.~Lee, Eshaan Nichani, and Atsushi Nitanda for insightful discussions. KO was partially supported by JST, ACT-X Grant Number JPMJAX23C4. TS was partially supported by JSPS KAKENHI (24K02905) and JST CREST (JPMJCR2015).

\bigskip

{

\fontsize{10}{11}\selectfont     

\bibliography{ref}
\bibliographystyle{alpha}

}

\newpage
{
\renewcommand{\contentsname}{Table of Contents}
\tableofcontents
}
 
\newpage
\appendix

\allowdisplaybreaks

\newcommand{\alowerbound}{\underline{A}}

\section{Preliminaries} 

\label{app:prelim}

\subsection{Hermite Polynomials}
For $k\in \Z_+$, the $k$-th Hermite polynomial $\He_k\colon\R\to\R$ is a univariate function defined as $\He_k(t) = e^{\frac{t^2}{2}}\frac{\mathrm{d}^k}{\mathrm{d}t^k} e^{-\frac{t^2}{2}}$.
The Hermite polynomials form an orthogonal basis for the Hilbert space of square-integrable functions. 
We provide several basic properties of the Hermite polynomials.
\begin{lemma}
    The Hermite polynomials satisfy the following properties:
    \begin{itemize}
        \item {\bf Derivatives:} $$\frac{\mathrm{d}}{\mathrm{d}t}\He_k(t)=k\He_{k-1}(t).$$
        \item {\bf Integration by parts (I):}
        $$\int \He_k(t)f(t)\frac{1}{\sqrt{2\pi}}e^{-\frac{t^2}{2}}\mathrm{d}t=\int \He_{k-1}(t)f'(t)\frac{1}{\sqrt{2\pi}}e^{-\frac{t^2}{2}}\mathrm{d}t.$$
        \item {\bf Integration by parts (II):}
        For $u\in \mathbb{S}^{d-1}(1)$ and $v\in \mathbb{R}^d$, 
                $$\int \He_k(u^\top x)f(v^\top x)\frac{1}{\sqrt{(2\pi)^d}}e^{-\frac{\|x\|^2}{2}}\mathrm{d}x=(u^\top v)\int \He_{k-1}(u^\top t)f'(v^\top x)\frac{1}{\sqrt{(2\pi)^d}}e^{-\frac{\|x\|^2}{2}}\mathrm{d}x.$$
                
                \item {\bf Orthogonality (I):} $$\int \He_k(t)\He_l(t)\frac{1}{\sqrt{2\pi}}e^{-\frac{t^2}{2}}\mathrm{d}t=k!\delta_{k,l}.$$
            \item {\bf Orthogonality (II):}
        For $u,v\in \mathbb{S}^{d-1}(1)$, 
                $$\int \He_k(u^\top x)\He_l(v^\top x)\frac{1}{\sqrt{(2\pi)^d}}e^{-\frac{\|x\|^2}{2}}\mathrm{d}x=k!(u^\top v)^k\delta_{k,l}.$$
            \item {\bf Hermite expansion:} if $f\in \R\to\R$ is square-integrable with respect to the standard Gaussian, $$f(t)\overset{L^2}{=}\sum_{k=0}^\infty\frac{\alpha_{k}}{k!}\He_k(t),\quad \alpha_{k} = \int f(t)\He_k(t)\frac{1}{\sqrt{2\pi}}e^{-\frac{t^2}{2}}\mathrm{d}t.$$
    \end{itemize}
\end{lemma}

\subsection{High Probability Bound on $y$}
\label{subsection:HPB-Y}
Recall that in the definition of $f_*$, we used a scaling factor of $\frac{1}{\sqrt{M}}$ instead of $\frac{1}{M}$ in order to ensure that the output is of order $\Theta(1)$ with high probability.
This subsection proves the $\Theta(1)$ output scale rigorously.
Since $v_m$ are almost orthogonal and $f_m(t)$ are mean-zero, $f_m(v_m^\top x)$ are almost independent mean-zero variables. 
Thus we expect $\frac{1}{\sqrt{M}}\sum_{m=1}^Mf_m(v_m^\top x)\to \mathcal{N}(0,1)$ by central limit theorem.
% The following lemma makes the above intuition rigorous.
\begin{lemma}\label{lemma:HPB-Y}
    Let $\|v_m\|=1$ for all $m=1,2,\cdots,M$. 
    If $M \leq \max_{m\ne m'}|v_m^\top v_{m'}|^{-p}$ and $S$ is even, we have
    \begin{align}\label{eq:pNormofy-1}
        \mathbb{E}\left[\left\|\frac{1}{\sqrt{M}}\sum_{m=1}^Mf_m(v_m^\top x)\right\|_S\right]
        \leq
        \mathbb{E}\left[\left\|\frac{1}{\sqrt{M}}\sum_{m=1}^Mf_m(v_m^\top x)\right\|_S^S\right]^{1/S}
        \lesssim
        S^{q+1}
        %((q!)^q(q-p+1)\max_{m,i} |c_{m,i}|)S^{(q-p+3)}
        .
    \end{align}
    Thus, by applying Lemma~\ref{lemma:HPBwithpNorm}, we have
    \begin{align}
        \left|\frac{1}{\sqrt{M}}\sum_{m=1}^Mf_m(v_m^\top x)\right| 
        \lesssim (\log 1/\delta)^{q+1}%\max_{m,i} |c_{m,i}|
        %\lesssim ((q!)^q(q-p+1)(\max_{m,i} |c_{m,i}|)(q-p+3)^{(q-p+3)})(\log 1/\delta)^{(q-p+3)}
        = \tilde{O}(1),\label{eq:pNormofy-5}
    \end{align}
    with probability at least $1-\delta$.
\end{lemma}
\begin{proof}
    Recall the Hermite expansion of $f_m$:  $f_m(t)=\sum_{i=p}^q \alpha_{m,i}\He_i(t)$.
    We decompose the LHS of \eqref{eq:pNormofy-1} as
    \begin{align}
        &\mathbb{E}\left[\left|\frac{1}{\sqrt{M}}\sum_{m=1}^Mf_m(v_m^\top x)\right|^S\right]
        \\
        &\leq\frac{1}{\sqrt{M}^S}\sum_{(m_1,\cdots,m_S)\in [M]^S}\mathbb{E}\left[f_{m_1}(v_{m_1}^\top x)\cdots f_{m_S}(v_{m_S}^\top x)\right].
        \\ &\leq 
        \!\frac{1}{\sqrt{M}^S}\!\!\!\sum_{(m_1,\cdots,m_S)\in [M]^S}\sum_{(i_1,\cdots,i_S)\in [q-p+1]^S}\!(\alpha_{m_1,i_1}\cdots \alpha_{m_S,i_S})
       \mathbb{E}\!\left[\He_{i_1}(v_{m_1}^\top x)\cdots \He_{i_S}(v_{m_S}^\top x)\right].
        \label{eq:pNormofy-2}
    \end{align}
    We evaluate each term of \eqref{eq:pNormofy-2}.
    If $i_1+\cdots+i_S$ is odd, the term is $0$.
    Otherwise, we bound the value of $ \mathbb{E}[\He_{i_1}(v_{m_1}^\top x)\cdots \He_{i_S}(v_{m_S}^\top x)]$ recursively.
    
    For $T\in \N$, $\boldsymbol{i}=(i_1,\cdots,i_{T})$, and $\boldsymbol{m}=(m_1,\cdots,m_{T})$, let us define
    \begin{align}
        A(T,\boldsymbol{i},\boldsymbol{m}):=\mathbb{E}[\He_{i_1}(v_{m_1}^\top x)\cdots \He_{i_{T}}(v_{m_{T}}^\top x)],
    \end{align}
    and
    \begin{align}
        B(T-1,\boldsymbol{i}) = \left\{\boldsymbol{j}=(j_1,\cdots,j_{T-1})\in \Z_+^{T-1}\left|\ \sum_{t=1}^{T-1}j_t= i_T,\ j_t\leq i_t\ (t=1,\cdots,T-1)\right.\right\}.
    \end{align}
    For $\boldsymbol{j}\in B(T-1,\boldsymbol{i})$, the multinomial coefficient for permuting a multiset of $\sum_{t=1}^{T-1}j_t$ elements (where $j_t$ is the multiplicity of each of the $t$-th element) is denoted by   $$c_{\boldsymbol{j}}:=\frac{(\sum_{t=1}^{T-1}j_t)!}{(j_1)!\cdots (j_{T-1})!}.$$
    
    By the basic property of Hermite polynomials, we have
    \begin{align}
        A(T,\boldsymbol{i},\boldsymbol{m})&=
        \mathbb{E}\left[\frac{\partial^{i_{T}}}{\partial (v_{m_{T}}^\top x)^{i_{T}}}(\He_{i_1}(v_{m_1}^\top x)\cdots \He_{i_{T-1}}(v_{m_{T-1}}^\top x))\right]
        \\ &=
        \sum_{\boldsymbol{j}\in B(T-1,\boldsymbol{i})}c_{\boldsymbol{j}}\prod_{t=1}^{T-1}\left[(v_{m_{T}}^\top v_{m_{t}}^\top)^{j_t}{}_{i_t}P_{j_t}
        \right] A(T-1,\boldsymbol{i}_{1:T-1}-\boldsymbol{j},\boldsymbol{m}_{1:T-1}).
    \end{align}
    We can bound the coefficients as $c_{\boldsymbol{j}} \leq i_T!\leq   q!$ and $\prod_{t=1}^{T-1}{}_{i_t}P_{j_t}\leq q^q$.
    We also bound the size of $B(T-1,\boldsymbol{i})$ by $S^q$.
    We define $I(m_1,\cdots,m_S)$ as the number of $m$ distinct in $(m_1,\cdots,m_S)$.
    Then, by recursively bounding $A$, we have
    \begin{align}
        A(S,(i_1,\cdots,i_S),(m_1,\cdots,m_S))
        \leq \left(q!q^qS^q\right)^S\left(\max_{m\ne m'}|v_m^\top v_{m'}|\right)^{\frac{pI(m_1,\cdots,m_S)}{2}}.
        \label{eq:pNormofy-3}
    \end{align}
    Moreover, 
    \begin{align}
    &\sum_{(m_1,\cdots,m_S)\in [M]^S}
    \left(\max_{m\ne m'}|v_m^\top v_{m'}|\right)^{\frac{pI(m_1,\cdots,m_S)}{2}}
    \\ &=
        \sum_{i=0}^S\sum_{(m_1,\cdots,m_S)\in [M]^S}\mathbbm{1}[I(m_1,\cdots,m_S)=i]\left(\max_{m\ne m'}|v_m^\top v_{m'}|\right)^{\frac{pi}{2}}
        \\ & \leq \sum_{i=0}^S\mathbbm{1}\left[i\geq \frac{S-i}{2}, \frac{S-i}{2}\in \Z_+\right]{}_MP_{M-i}\cdot {}_{i}P_{ \frac{S-i}{2}}\left( \frac{S-i}{2}\right)^{ \frac{S-i}{2}}\left(\max_{m\ne m'}|v_m^\top v_{m'}|\right)^{\frac{pi}{2}}
          \\ & \leq \max_{0\leq i\leq S}M^iS^{S} \left(\max_{m\ne m'}|v_m^\top v_{m'}|\right)^{\frac{pi}{2}}
          \label{eq:pNormofy-4}.
    \end{align}

    If $M\leq (\max_{m\ne m'}|v_m^\top v_{m'}|)^{-p}$, applying \eqref{eq:pNormofy-3} and \eqref{eq:pNormofy-4} to \eqref{eq:pNormofy-2}, we obtain
    \begin{align}
    &\mathbb{E}\left[\left|\frac{1}{\sqrt{M}}\sum_{m=1}^Mf_m(v_m^\top x)\right|^S\right]
   \\ &
   \leq \frac{1}{\sqrt{M}^S}
   \!\!\!\sum_{(m_1,\cdots,m_S)\in [M]^S}\sum_{(i_1,\cdots,i_S)\in [q-p+1]^S}\!\!|\alpha_{m_1,i_1}\cdots \alpha_{m_S,i_S}|\left(q!q^qS^q\right)^S\left(\max_{m\ne m'}|v_m^\top v_{m'}|\right)^{\frac{pI(m_1,\cdots,m_S)}{2}}
   \\ & \leq \frac{\left((q-p+1)q!q^qS^q\max|\alpha_{m_s,i_s}|\right)^S}{\sqrt{M}^S}\sum_{(m_1,\cdots,m_S)\in [M]^S}\left(\max_{m\ne m'}|v_m^\top v_{m'}|\right)^{\frac{pI(m_1,\cdots,m_S)}{2}}
    \\ & \leq  \frac{\left((q-p+1)q!q^qS^q\max|\alpha_{m_s,i_s}|\right)^S}{\sqrt{M}^S}\max_{0\leq i\leq S}M^iS^{S} \left(\max_{m\ne m'}|v_m^\top v_{m'}|\right)^{\frac{pi}{2}}
    \\ &
    \leq \left((q-p+1)q!q^qS^{q+1}\max|\alpha_{m_s,i_s}|\right)^S
        % \sum_{(m_1,\cdots,m_P)\in [M]^P}\left(\max_{m\ne m'}(v_m^\top v_{m'})\right)^{I''p}
        % &= \sum_{t=0}^P\sum_{(m_1,\cdots,m_P)\in [M]^P}\mathbbm{1}[I''(i_1,\cdots,i_P,m_1,\cdots,m_P)=t]\left(\max_{m\ne m'}(v_m^\top v_{m'})\right)^{tp}
        % \\ & \leq \sum_{t=0}^P M^tt^{P-t}\left(\max_{m\ne m'}(v_m^\top v_{m'})\right)^{tp}
        % \\ & \leq 
        % (P+1)P^P\max_{t=0,P}\left(M\max_{m\ne m'}(v_m^\top v_{m'})^{p}\right)^t
        ,
    \end{align}
    which yields \eqref{eq:pNormofy-1}.
    
    By applying Lemma~\ref{lemma:HPBwithpNorm} with $c_1=q+1$, 
    we have
    \begin{align}
        \left|\frac{1}{\sqrt{M}}\sum_{m=1}^Mf_m(v_m^\top x)\right|
        \lesssim (\log 1/\delta)^{q+1}, %((q!)^q(q-p+1)\max |c_{m,i}|)P^{(q-p+3)},
    \end{align}
    with probability $1-\delta$, 
    which yields \eqref{eq:pNormofy-2}.
\end{proof}

The above proof makes use of the following classical inequality. 
\begin{lemma}\label{lemma:HPBwithpNorm}
    Let $\delta>0$ and $X$ be a mean-zero random variable satisfying
    \begin{align}
        \mathbb{E}[|X|^S]\leq C_1S^{c_1}\quad\text{for }S=\frac{\log 1/\delta}{c_1}
    \end{align}
    for $C_1,c_1>0$.
    Then, with probability at least $1-\delta$, we have
    \begin{align}
        |X|\leq C_1(eS)^{c_1}.
    \end{align}
\end{lemma}
\begin{proof}
    The proof follows from \citet{damian2023smoothing}.
    \begin{align}
        \mathbb{P}[|X|\geq C_1(eS)^{c_1}]
        = 
        \mathbb{P}[|X|^S\geq C_1^S(eS)^{c_1S}]
        \leq \frac{\mathbb{E}[|X|^S]}{C_1^S(eS)^{c_1S}}
        \leq \frac{C_1^SS^{c_1S}}{C_1^S(eS)^{c_1S}}
        \leq e^{-c_1S}=\delta,
    \end{align}
    which concludes the proof. 
\end{proof}

\subsection{Bihari–LaSalle Inequality and Gronwall Inequality}
   For later use, we provide the proofs of the Bihari–LaSalle inequality and the Grönwall's inequality for completeness. 
    The proof of the Bihari–LaSalle inequality is borrowed from \citet{arous2021online}.
    \begin{lemma}[Bihari–LaSalle inequality and Gronwall inequality]\label{lemma:Bihari–LaSalle}
     For $p\geq 3$ and $c>0$, consider a positive sequence $(a^t)_{t\geq 0}$ such that
     \begin{align}
         a^{t+1} = a^t +c (a^t)^{p-1}.
     \end{align}
     Then, we have
     \begin{align}\label{eq:BihariLaSalle}
         a^t \geq \frac{a^0}{\left(1-c (p-2) (a^0)^{(p-2)}t\right)^\frac{1}{p-2}}.
     \end{align}
     Moreover, when $a^t\leq 1$ holds for all $t \leq T-1$, we have
     \begin{align}\label{eq:Gronwall}
         a^t \leq \frac{a^0}{\left(1-  c(1+c)^{p-1} (p-2) (a^0)^{(p-2)}t\right)^\frac{1}{p-2}}
     \end{align}
     for all $t\leq T$.
    \end{lemma}
    \begin{proof}
        From definition, we have
        \begin{align}
             c = \frac{a^{t+1}-a^t}{(a^t)^{p-1}}
            \leq \int_{t=a^{t}}^{a^{t+1}}\frac{1}{x^{p-1}}\leq \frac{1}{p-2}\left[\frac{1}{(a^t)^{p-2}}- \frac{1}{(a^{t+1})^{p-2}}\right].
        \end{align}
        Taking the summation and re-arranging the terms yield
        \begin{align}
            (a^t)^{-(p-2)} \leq (a^0)^{-(p-2)} -c (p-2) t,
            \\ \therefore a^t \geq \frac{a^0}{\left(1- c (p-2) (a^0)^{(p-2)}t\right)^\frac{1}{p-2}},
        \end{align}
        which gives the lower bound.

        On the other hand, when $a^t \leq 1$, we have $a^{t+1} \leq (1+c)a^t$, and therefore
        \begin{align}
           c = \frac{a^{t+1}-a^t}{(a^t)^{p-1}}
            \geq (1+c)^{-(p-1)}\int_{t=a^{t}}^{a^{t+1}}\frac{1}{x^{p-1}}= (1+c)^{-(p-1)}\frac{1}{p-2}\left[\frac{1}{(a^t)^{p-2}}- \frac{1}{(a^{t+1})^{p-2}}\right].
        \end{align}
        By taking the summation and re-arranging the terms yield
        \begin{align}
            (a^t)^{-(p-2)} \leq (a^0)^{-(p-2)} -c(1+c)^{p-1} (k-1) t,
            \\ \therefore a^t \leq \frac{a^0}{\left(1- c(1+c)^{p-1} (p-2)(a^0)^{(p-2)}t\right)^\frac{1}{p-2}},
        \end{align}
        which gives the upper bound.
    \end{proof}

\subsection{Orthonormal Basis from Nearly Orthogonal Vectors}
In case where the feature vectors $v_m$ are not orthogonal, the following lemma shows an orthonormal basis can be constructed as a linear combination of $v_m$.
\begin{lemma}\label{lemma:OrthogonalBasisfromVM}
    Let $\{v_m\}_{m=1}^M\subset \mathbb{S}^{d-1}$ be a set of unit vectors in $\mathbb{R}^d$, and suppose that 
    $\max_{m\ne m'}|v_m^\top v_{m'}|\leq \frac12 M^{-1}$
    holds.
    Then, we can construct an orthonormal basis $\{\tilde{v}_m\}_{m=1}^{M}\subset\mathbb{S}^{d-1}$ so that $\tilde{v}_m$ is a linear combination of $v_1,\cdots,v_m$.
    Specifically, we can take $\tilde{v}_1=v_1$ and 
    \begin{align}\label{eq:OrthogonalBasisfromVM-1}
        \tilde{v}_m=\sum_{m'=1}^m c_{m,m'}v_{m'},
    \end{align}
    so that 
    $|c_{m,m'}| \leq
    4\max_{m',m''}|v_{m'}^\top v_{m''}|$ for $m'=1,\cdots,m-1$, $|1-c_{m,m}|\leq 20 M\max_{m'\ne m''}|v_{m'}^\top v_{m''}|$, and $\tilde{v}_m^\top v_{m'}=0$ for $m'=1,\cdots,m-1$ hold.
    % $\sum_{m'=1}^{m-1} |c_{m,m'}| \leq 3 M \max_{m\ne m'}|v_m^\top v_{m'}|$, $|1 - c_{m,m}| \leq 36M\max_{m',m''}|v_{m'}^\top v_{m''}|$,  and $\tilde{v}_m^\top v_{m'}=0\ (m'=1,\cdots,m-1)$ hold.
\end{lemma}
\begin{proof}
    First, we show that $\mathrm{dim}(\mathrm{span}\{v_1,\cdots,v_m\})=m$ for all $m$. 
    Assume the opposite and then we have
    \begin{align}
        v_m = \sum_{m'=1}^{m-1} a_{m,m'}v_{m'}
    \end{align}
    for some $m$ and $\{a_{m,m'}\}_{ 1\leq m'\leq m-1}$.
    Then, we have
    \begin{align}
        1=\|v_m\|^2 &\geq \sum_{m'=1}^{m-1} a_{m,m'}^2+ \sum_{m'\ne m''} a_{m,m'}a_{m,m''}v_{m'}^\top v_{m''}
       % \\ & \geq \sum_{m'=1}^{m-1} a_{m,m'}^2- \sum_{m'\ne m''} |a_{m,m'}||a_{m,m''}|\max_{m'\ne m''}|v_{m'}^\top v_{m''}|
       \\ & \geq \sum_{m'=1}^{m-1} a_{m,m'}^2- \sum_{m'\ne m''} \frac{|a_{m,m'}||a_{m,m''}|}{2M}
       \\ & \geq \sum_{m'=1}^{m-1} a_{m,m'}^2- \frac{1}{2M}\left(\sum_{m'=1}^{m-1} |a_{m,m'}|\right)^2
       \\ & \geq \sum_{m'=1}^{m-1} a_{m,m'}^2 - \frac12\sum_{m'=1}^{m-1} a_{m,m'}^2
       =\frac12 \sum_{m'=1}^{m-1} a_{m,m'}^2,
       \label{eq:OrthogonalBasisfromVM-2}
    \end{align}
    where we used the Cauchy–Schwarz inequality for the last step.
    Thus, we have $|a_{m,m'}|\leq \sqrt{2}$ and therefore
    \begin{align}
        1\!=\!\|v_m\|^2\!=\! 
        \sum_{m'=1}^{m-1} a_{m,m'}v_{m'}^\top v_m
        \!\leq\! (M-1)\!\max_{m'\ne m} |a_{m,m'}|\cdot\!\max_{m'\ne m} |v_{m'}^\top v_m|\!\leq\!(M-1)\!\cdot\!\sqrt{2}\!\cdot\!\frac{1}{2M}\!<\!1,
    \end{align}
    which yields the contradiction.

    Therefore, we have an orthogonal basis $\{\tilde{v}_m\}_{m=1}^M$ such that
    \eqref{eq:OrthogonalBasisfromVM-1} and $\tilde{v}_m^\top v_{m'}=0\ (1\leq m\leq M, 1\leq m'\leq m-1)$ hold, with some $\{c_{m,m'}\}_{1\leq m\leq M, 1\leq m'\leq m}$.
    Similarly to \eqref{eq:OrthogonalBasisfromVM-2}, we have $|c_{m,m'}|\leq \sqrt{2}$ for all $1\leq m\leq M$ and $1\leq m'\leq m$.
    What remains is to bound the coefficients.
    Since we have
    \begin{align}
        \tilde{v}_m=\sum_{m'=1}^m c_{m,m'}v_{m'}
    \end{align}
    for all $m$, taking an inner product with $\sum_{m'=1}^{m-1}\mathrm{sign}(c_{m,m'})v_{m'} $, we get
    \begin{align}
        0&=\tilde{v}_m^\top \sum_{m'=1}^{m-1}\mathrm{sign}(c_{m,m'})v_{m'}
        = \left(\sum_{m'=1}^m c_{m,m'}v_{m'}\right)^\top \sum_{m'=1}^{m-1}\mathrm{sign}(c_{m,m'})v_{m'}
        \\ & =\sum_{m'=1}^{m-1}|c_{m,m'}| + \sum_{m'=1}^{m-1} \sum_{m''\ne m',m}c_{m,m'}\mathrm{sign}(c_{m,m''})v_{m'}^\top v_{m''}+\sum_{m'=1}^{m-1}c_{m,m}\mathrm{sign}(c_{m,m'})v_{m}^\top v_{m'}
        \\ & \geq \sum_{m'=1}^{m-1}|c_{m,m'}|-\sum_{m'=1}^{m-1}m \sqrt{2}\frac{1}{2M}-\sqrt{2}m\frac{1}{2M}
        \\ & \geq \left(1-\frac{1}{\sqrt{2}}\right)\sum_{m'=1}^{m-1}|c_{m,m'}|-\frac{1}{\sqrt{2}}.
    \end{align}
    Thus, $\sum_{m'=1}^{m-1}|c_{m,m'}|$ is bounded by $\sqrt{2}+1\leq 3$
     and $\sum_{m'=1}^{m}|c_{m,m'}|$ is bounded by $\sqrt{2}+1+\sqrt{2}\leq 4$.

    Also considering an inner product with $\mathrm{sign}(c_{m,m'})v_{m'}\ (1\leq m'\leq m-1)$, we get
    \begin{align}
        0&=\tilde{v}_m^\top \mathrm{sign}(c_{m,m'})v_{m'}
        = \left(\sum_{m''=1}^m c_{m,m''}v_{m''}\right)^\top \mathrm{sign}(c_{m,m'})v_{m'}
        \\ & =|c_{m,m'}| + \sum_{m''\ne m'}c_{m,m''}\mathrm{sign}(c_{m,m'})v_{m'}^\top v_{m''}
        \\ & \geq |c_{m,m'}|-\sum_{m'=1}^{m}|c_{m,m'}| \max_{m'\ne m''}|v_{m'}^\top v_{m''}|.
        % -\sqrt{2}m\frac{1}{2M}
    \end{align}
    Thus, $|c_{m,m'}|\ (1\leq m'\leq m-1)$ is bounded by $\sum_{m'=1}^{m}|c_{m,m'}| \max_{m'\ne m''}|v_{m'}^\top v_{m''}| \leq 4\max_{m'\ne m''}|v_{m'}^\top v_{m''}|$.
    Finally, we get
    \begin{align}
       1=\|\tilde{v}_m\|^2 &= \sum_{m'=1}^m c_{m,m'}^2 + \sum_{m'=1}^m \sum_{m''\ne m'}c_{m,m'}c_{m,m''}v_{m'}^\top v_{m''}
       \\  \therefore
       |1-c_{m,m}^2| &\leq \sum_{m'=1}^{m-1} |c_{m,m'}| + \sum_{m'=1}^m |c_{m,m'}|\sum_{m''\ne m'}|c_{m,m''}||v_{m'}^\top v_{m''}|
      \\ & \leq 4 M\max_{m'\ne m''}|v_{m'}^\top v_{m''}| + 16 \max_{m'\ne m''}|v_{m'}^\top v_{m''}|
      \leq 20 M\max_{m'\ne m''}|v_{m'}^\top v_{m''}|,
    \end{align}
    which implies that $|1-c_{m,m}|\leq 20 M\max_{m'\ne m''}|v_{m'}^\top v_{m''}|$.
\end{proof}

\bigskip

\section{Proof of Gradient-based Training}

\paragraph{Overview of analysis.} 
We define polylogarithmic constants with the following order of strength:
    \begin{align}
        % (\log d)^{(p-2)}\lesssim 
        \CB 
        \lesssim \cD^{-1}
        \lesssim \cA^{-1} 
        \lesssim \CA
        \lesssim 
        \cB^{-1}%\lesssim c_r^{-1}
        % \lesssim \CA 
        \lesssim \cC^{-1}=\tilde{O}(1). 
    \end{align}
Here $c_1$ and $C_1$ are different from those in Lemma~\ref{lemma:HPBwithpNorm}.
$\cC$ is the same as $c_v$ used in Assumption~\ref{assumption:diversity}.
$\CB$ will be used to represent any polylogarithmic factor that comes from high probability bounds.
Also, Section~\ref{subsection:Initialization} will introduce another constant $C_p$.

In Algorithm \ref{alg:main} we first train the first-layer parameters, where we aim to show that a for each class $m$, there exist sufficiently many neurons that almost align with $v_m$ to approximate each $f_m(v_m^\top x)$. 
We define the alignment for the $m$-th task at time $t$ as $\kappa_m^t=v_m w_j^t$.
The goal of first-layer training is to prove the following. 
We introduce the error   
$\tilde{\varepsilon}=\tilde{\Theta}(M^{-\frac12}\varepsilon)$, where $\varepsilon$ is the desired final generalization error $\mathbb{E}_{x}[|f_{\hat{a}}(x) - f_*(x)|] \lesssim \varepsilon$.
\newtheorem*{lemma:Goal-1}{\rm\bf Lemma~\ref{lemma:first-layer-training} (formal)}
\begin{lemma:Goal-1} 
\label{lemma:Goal-1}
    Let $T_{1,1}=\tilde{\Theta}(Md^{p-1})$, $T_{1,2}=\tilde{\Theta}(Md^{\frac{p}{2}})$, $T_{1,3}=\tilde{\Theta}(\tilde{\varepsilon}^{-2}Md\lor \tilde{\varepsilon}^{-3}M)$, and $T_{1}=T_{1,1}+T_{1,2}+T_{1,3}$.
    Take the step size as $\eta^t = \tilde{\Theta}(M^{-\frac12}d^{-\frac{p}{2}})$ for $0\leq t\leq T_{1,1}+T_{1,2}-1$ and $\eta^t = \tilde{\Theta}(\tilde{\varepsilon}M^{-\frac12}d^{-1}\land \tilde{\varepsilon}^2M^{-\frac12})$ for $T_{1,1}+T_{1,2}\leq t\leq T_{1,1}+T_{1,2}+T_{1,3}$, and the number of neurons as $J\gtrsim J_{\rm min}M^{C_p}\log d$.
    Suppose that $|v_{m'}^\top v_m|= \tilde{O}(M^{-1})$ for all $m'\ne m$, and $M= \tilde{O}(d^\frac12)$.
    Then, with high probability, %there exists some time $T_1=\tilde{O}(\tilde{\varepsilon}^{-1}M^\frac12\eta^{-1})$ such that 
    for each class $m$ there exist at least $J_{\rm min}$ neurons that achieves $v_m^\top w_j^{T_1} \geq 1-3\tilde{\varepsilon}$.
\end{lemma:Goal-1}

We let $T_1'=T_1+T_2$ in Theorem~\ref{theorem:NN-main}. 
The proof for Lemma~\ref{lemma:Goal-1} will be divided into the following parts. First we consider the initialization and activation functions. 
\begin{itemize}[leftmargin=*]
\item In Section~\ref{subsection:Initialization} we analyze the random initialization.
We show that, at the time of initialization, for each class $m$, there exist $J_{\rm min}$ neurons classified into some class $J_m$ (i.e., with slightly higher overlap with $v_m$). 
\item  Section~\ref{subsection:DescentPath} discusses the assumptions on the target/activation functions.
We show that Assumption~\ref{assumption:NeuronPositiveCorrelation} is satisfied with a constant probability (i.e., $\Omega(1)$ fraction of neurons). 
Since neurons do not interact in the correlation loss update, in the subsequent sections of the first-layer training, we focus on the dynamics of one neuron in $J_1$ that satisfies Assumption~\ref{assumption:NeuronPositiveCorrelation} and omit the subscript $w^t=w_j^t$ without loss of generality. 
\end{itemize}

Next we consider the training of the first-layer weights via a \textit{power-method like dynamics}. 
\begin{itemize}[leftmargin=*]
\item In Section~\ref{subsection:alignment-0}, we decompose the gradient update into the population dynamics and noise fluctuation.
The training dynamics consist of following three different phases. 
\item The first phase corresponds to Section~\ref{subsection:alignment-1}.
Here, we show that neurons in $J_m$ will obtain a small $\tilde{\Omega}(1)$ alignment to the class $v_m$ after $t_1(\leq T_{1,1})$ iterations.  
\item The second phase corresponds to Section~\ref{subsection:alignment-2}, where we show that neurons continue to grow in the direction of $v_m$ and achieves $1-\tilde{O}(1)$ alignment within $t_2(\leq T_{1,2})$ iterations, while remaining almost orthogonal to other directions.
\item Finally, in Section~\ref{subsection:alignment-3}, 
we show that, after $(T_{1,1}-t_1)+(T_{1,2}-t_2)+T_{1,3}$ iterations, neurons in $J_m$ will eventually achieve $w_j^\top v_m \geq 1-\tilde{\varepsilon}$ with high probability.
\end{itemize}

After first-layer training is completed, in 
Section~\ref{subsection:Expressivity} we prove the existence of suitable second-layer parameters with small norm that can approximate $f_*$ (Lemma~\ref{lemm:ApproximationByAstar}). 
Section~\ref{subsection:FittingSecondLayer} concludes the generalization error analysis, which is stated as follows:
\begin{lemma}\label{lemma:Generalization}
 Suppose that %$J = \tilde{\Theta}(d^{\frac{p-2}{2}}M^{C}d^{\frac{p-2}{2}}\varepsilon^{-2})$
 $J = \tilde{\Theta}(J_{\mathrm{min}}M^{C_p})$, and $\sigma$ be  either of the ReLU activation or any univariate polynomial with degree $q$. 
 There exists $\lambda > 0$ such that for ridge regression ($p=2$), we have
$$
\mathbb{E}_{x}[|f_{\hat{a}}(x) - f_*(x)|] 
\lesssim 
M^\frac12 (|J_{\mathrm{min}}|^{-1}+\tilde{\varepsilon})+\sqrt{\frac{dM^{C_p}}{T_2}}
$$
with probability $1-o_d(1)$.
Therefore, by taking $T_2=\tilde{\Theta}(dM^{C_p}\varepsilon^{-2})$, $\tilde{\varepsilon}=\tilde{\Theta}(M^{-\frac12}\varepsilon)$, $J_{\mathrm{min}}=\tilde{\Theta}(M^\frac12 \varepsilon^{-1})$, and $J=\tilde{\Theta}(M^{C_p+\frac12}\varepsilon^{-1})$, we have $\mathbb{E}_{x}[|f_{\hat{a}}(x) - f_*(x)|]\lesssim \varepsilon$.

On the other hand, for LASSO ($p=1$), we have
$$
\mathbb{E}_{x}[|f_{\hat{a}}(x) - f_*(x)|] 
\lesssim
M^\frac12 (|J_{\mathrm{min}}|^{-1}+\tilde{\varepsilon})+\sqrt{\frac{dM}{T_2}}
% \lesssim \frac{M^\frac12\sqrt{d}}{\sqrt{T_2}}  
% + \tilde{O}(\varepsilon), %\frac{1}{T_2}
$$
with probability $1-o_d(1)$. 
Therefore, by taking $T_2=\tilde{\Theta}(dM\varepsilon^{-2})$, $\tilde{\varepsilon}=\tilde{\Theta}(M^{-\frac12}\varepsilon)$, $J_{\mathrm{min}}=\tilde{\Theta}(M^\frac12 \varepsilon^{-1})$, and $J=\tilde{\Theta}(M^{C_p+\frac12}\varepsilon^{-1})$, we have $\mathbb{E}_{x}[|f_{\hat{a}}(x) - f_*(x)|]\lesssim \varepsilon$. 
\end{lemma}
Combining Lemma~\ref{lemma:Generalization} and Lemma~\ref{lemma:Goal-1} completes the proof of Theorem~\ref{theorem:NN-main}.

\subsection{Initialization}\label{subsection:Initialization}

To begin with, we need diversity of the neurons at initialization.
Note that we do not require neurons to achieve $\Theta(1)$ alignment with $v^m$, as this requires exponential width.
Instead, we prove that, for each class $v^m$, there exist some neurons $w_j^0$ that are more aligned to this direction $v^m$ than others --- in other words, $v_m^\top w_j^0 \gtrsim \max_{m'\ne m}|v_{m'}^\top w_j^0|$. 
The statement is formalized as follows. In the proof, the subscript $t$ is omitted because we only consider the initialization $t=0$.
\newtheorem*{lemma:Appendix-GB-Initialization}{\rm\bf Lemma~\ref{lemma:Appendix-GB-Initialization} (formal)}
\begin{lemma:Appendix-GB-Initialization}
    Suppose that the first layer weight of each neuron is initialized as an independent sample from the uniform distribution over $\mathbb{S}^{d-1}(1)$.
    Let $p>2$, 
$\delta = (\log d)^{-\frac{p-2}{2}}$, and  
$C_p=\left(\frac{\max_{m}|\beta_{m,p}|}{\min_{m'}|\beta_{m',p}|}\right)^{\frac{2}{p-2}}$.
    We define a set of indexes of neurons $J_m$ that have the highest alignment with $f_m(v_m^\top x)$ as
    \begin{align}
        J_m := \left\{j\in [J] \ \left|\ w_j^\top v_m\geq \frac{1}{\sqrt{d}}, \ (w_j^\top v_m)^{p-2}\geq \max_{m'\ne m} C_p^{\frac{p-2}{2}}|w_j^\top v_{m'}|^{p-2} + \delta\left(\frac{1}{\sqrt{d}}\right)^{p-2}\right.\right\}.
    \end{align}
    For $J_{\mathrm{min}}>0$, if 
    \begin{align}
        J \!\geq\! AJ_{\mathrm{min}}M^{C_p}\!(\log d)^\frac32,\! \text{ with }A = 
        \exp\bigg(
       O\Big(\max_{m\ne m'}|v_m^\top v_{m'}|\log d+1\Big)\bigg),
        % \exp\left(\!
        % \tilde{O}((\max_{m\ne m'}|v_m^\top v_{m'}|+\delta^\frac1k)\!\sqrt{\log M}\!+\!\max_{m\ne m'}|v_m^\top v_{m'}|^2\!+\!\delta^\frac2k\!+\!1)\!\right)\!,
    \end{align}
    then $|J_m|\geq J_{\mathrm{min}}$ for all $m$
    with high probability.
\end{lemma:Appendix-GB-Initialization}
Since $w_j^\top v_m=O(\sqrt{\log d})$ with high probability, we have the following corollary, where we use a small constant $\cD\lesssim ( \log d)^{(p-2)}$.
\begin{corollary}\label{corollary:Initialization}
    When $J\gtrsim J_{\mathrm{min}}M^{C_p}\log d$, for each class $m$, we have 
    at least $J_{\mathrm{min}}$ neurons $w_j$ such that 
\begin{align}
    w_j^\top v_m\geq \frac{1}{\sqrt{d}}, \text{ and } |\beta_{m,p}|(w_j^\top v_m)^{p-2}\geq \max_{m'}|\beta_{m',p}|\max_{m''\ne m} |w_j^\top v_{m''}|^{p-2} + \cD(w_j^\top v_m)^{p-2}.
    \label{eq:smalldiversity}
\end{align}
\end{corollary}

To prove the lemma, we make use of the following upper and lower bounds.
\begin{lemma}[Theorems 1 and 2 of \citet{chang2011chernoff}]\label{lemma:chang2011chernoff}
    For any $\beta>1$ and $x\in \mathbb{R}$, we have
    \begin{align}
        \frac{\sqrt{2e(\beta-1)}}{2\beta\sqrt{\pi}}
        e^{-\frac{\beta x^2}{2}}
    \leq
        \int_{x}^\infty \frac{1}{\sqrt{2\pi}}e^{-\frac{t^2}{2}}\mathrm{d}t
    \leq
        \frac12e^{-\frac{x^2}{2}}
    \end{align}
\end{lemma}
\begin{proofof}[Lemma~\ref{lemma:Appendix-GB-Initialization}]
    % We consider the case when $w_j \sim \mathrm{Unif}(\mathbb{S}^{d-1}(1))$ for simplicity, because multiplying a scalar yields the general case.
    Note that $w_j \sim \mathrm{Unif}(\mathbb{S}^{d-1}(1))$ can be obtained by sampling $\tilde{w}_j \sim \mathcal{N}(0,\frac1d I_d)$ and setting $w_j = \frac{\tilde{w_j}}{\|\tilde{w}_j\|}$.
    With high probability, we have $\|\tilde{w}_j\|\approx 1 \leq 2^\frac{1}{p-2}$.
    Thus we will instead show that 
    \begin{align}
        \tilde{w}_j^\top v_m\geq \frac{2}{\sqrt{d}}, \text{ and } (\tilde{w}_j^\top v_m)^{p-2}\geq \max_{m'\ne m} C_p^{\frac{p-2}{2}}|\tilde{w}_j^\top v_{m'}|^{p-2} + \frac{2\delta}{d^\frac{p-2}{2}}.
    \end{align}
    
    Fix $m\in [M]$. For each $m'\ne m$, consider the value of $\tilde{w}_j^\top (I-v_m v_m^\top)v_{m'}$ and $\tilde{w}_j^\top v_m v_m^\top v_{m'}$.
    The distribution of $\tilde{w}_j^\top (I-v_m v_m^\top)v_{m'}$ follows $\mathcal{N}(0,\|(I-v_m v_m^\top)v_{m'}\|)$, therefore by Lemma~\ref{lemma:chang2011chernoff},
    \begin{align}
        &\mathbb{P}\left[\text{for all $m'\ne m$,\ }|\tilde{w}_j^\top (I-v_m v_m^\top)v_{m'}|\leq t\right]
        \\ &\leq 1- (M-1)\exp\left(-\frac{dt^2}{2\max_{m'\ne m}\|(I-v_m v_m^\top)v_{m'}\| }\right).\label{eq:Appendix-GB-Initialization-1}
    \end{align}
    By taking 
    \begin{align}\label{eq:Appendix-GB-Initialization-2}
        t=t_1:=\left(2d^{-1}\max_{m'\ne m}\|(I-v_m v_m^\top)v_{m'}\|\log 2M\right)^\frac12,
    \end{align}
    \eqref{eq:Appendix-GB-Initialization-1} is bounded by $\frac{M+1}{2M}$.

    Note that $\tilde{w}_j^\top v_m v_m^\top v_{m'}=O(d^{-\frac12}\sqrt{\log d}\max_{m\ne m'}|v_m^\top v_{m'}|)$ with high probability.
    When $|\tilde{w}_j^\top (I-v_m v_m^\top)v_{m'}|\leq t_1$ for all $m'\ne m$ and $\tilde{w}_j^\top v_m v_m^\top v_{m'}=O(d^{-\frac12}\sqrt{\log d}\max_{m\ne m'}|v_m^\top v_{m'}|)$,
    we have $|\tilde{w}_j^\top v_{m'}|\leq t_1 + O(d^{-\frac12} \sqrt{\log d}$ $\max_{m\ne m'}|v_m^\top v_{m'}|)$.
    To satisfy $(\tilde{w}_j^\top v_m)^{p-2}\geq C_p^{\frac{p-2}{2}}|\tilde{w}_j^\top v_{m'}|^{p-2} + \frac{2\delta}{d^\frac{p-2}{2}}$ and $\tilde{w}_j^\top v_m\geq \frac{2}{\sqrt{d}}$, 
     it suffices for $\tilde{w}_j^\top v_m$ to be $w_j^\top v_m
     \geq t_2:=
         C_p^{\frac{1}{2}} t_1+C_p^{\frac{1}{2}}O(d^{-\frac12}\sqrt{\log d}\max_{m\ne m'}|v_m^\top v_{m'}|)+\frac{2^\frac{1}{p-2}\delta^\frac{1}{p-2}}{\sqrt{d}}$. 
     
     We lower bound the probability that $w_j^\top v_m
     \geq t_2$ holds.
     For any $\beta>1$, Lemma~\ref{lemma:chang2011chernoff} implies that
     \begin{align}
         &\mathbb{P}\left[\tilde{w}_j^\top v_{m}\geq t_2\right]
      \\ &\geq
       \frac{\sqrt{2e(\beta-1)}}{2\beta\sqrt{\pi}}\exp\left(-\frac{d\beta t_2^2}{2}\right)
       \\ &\geq
       \frac{\sqrt{2e(\beta-1)}}{2\beta\sqrt{\pi}}\exp\bigg(-C_p\beta \bigg[\max_{m'\ne m}\|(I-v_m v_m^\top)v_{m'}\|\log 2M
     \\ & \hspace{20mm}+O(\sqrt{\log d}\max_{m\ne m'}|v_m^\top v_{m'}|+\delta^\frac{1}{p-2})\sqrt{\log M}  + O(\log d\max_{m\ne m'}|v_m^\top v_{m'}|^2+\delta^\frac{2}{p-2})\bigg]
       \Bigg)
       \\ &\geq
       \frac{1}{(\sqrt{\log 2M}+\frac{1}{\sqrt{\log 2M}})\sqrt{\pi}}\exp\bigg(\!\!\!-\!C_p\bigg[\log 2M 
   +O\Big(\max_{m\ne m'}|v_m^\top v_{m'}|\log d+\delta^\frac{2}{p-2} \sqrt{\log d}+1\Big)
       \bigg]\bigg),
      %  \frac{1}{(1+\sqrt{\log M})\sqrt{\pi}}\exp\bigg(-C\bigg[\max_{m'\ne m}\|(I-v_m v_m^\top)v_{m'}\|(\log 2M + 1)
      % \\ & \hspace{50mm}+O(\max_{m\ne m'}|v_m^\top v_{m'}|\log d+\delta^\frac1k \sqrt{\log d}+ \delta^{\frac2k})
      %  \bigg]\bigg),
     \end{align}
     where we took $\beta = 1+\frac{1}{\log 2M}$ and used $\log M \lesssim \log d$, $\delta< 1$, and $\max_{m'\ne m}|v_m^\top v_{m'}|\leq 1$ for the last inequality.
    To simplify the notation, by letting 
    \begin{align}
        %% \frac12
        % A = \exp\left(
        % \tilde{O}((\max_{m\ne m'}|v_m^\top v_{m'}|+\delta^\frac1k)\sqrt{\log M}+\max_{m\ne m'}|v_m^\top v_{m'}|^2+\delta^\frac2k+1)\right)
        A = \exp\bigg(
       O\Big(\max_{m\ne m'}|v_m^\top v_{m'}|\log d+\delta^\frac{1}{p-2} \sqrt{\log d}+1\Big)\bigg)
       = \exp\bigg(
       O\Big(\max_{m\ne m'}|v_m^\top v_{m'}|\log d+1\Big)\bigg)
       ,
    \end{align}
     we have
     \begin{align}
         \mathbb{P}\left[\tilde{w}_j^\top v_{m}\geq t_2\right]\geq 
         \frac{2M^{-C_p}}{A\sqrt{\log 2M}}.
     \end{align}
     Note that this argument is independent from the one for \eqref{eq:Appendix-GB-Initialization-1}, because $(I-v_m v_m^\top)v_{m'}$ and $v_m$ are orthogonal.
     
    To sum up, \eqref{eq:Appendix-GB-Initialization-1} ensures that $|\tilde{w}_j^\top (I-v_m v_m^\top)v_{m'}|\leq t_1$ for all $m'\ne m$ with probability at least $\frac{M+1}{2M}$, under which $\tilde{w}_j^\top v_m v_m^\top v_{m'}=O(d^{-\frac12}\sqrt{\log d}\max_{m\ne m'}|v_m^\top v_{m'}|)$  for all $m'\ne m$ with high probability and 
    $\tilde{w}_j^\top v_{m}\geq t_2$ holds with probability at least $\frac{2M^{-C_p}}{A\sqrt{\log M}}$.
     $\tilde{w}_j^\top v_{m}\geq t_2$ implies that
$(\tilde{w}_j^\top v_m)^{p-2}\geq C|\tilde{w}_j^\top v_{m'}|^{p-2} + \frac{2\delta}{d^\frac{p-2}{2}}$ and $\tilde{w}_j^\top v_m\geq \frac{2}{\sqrt{d}}$.
    % $(w_j^\top v_m)^k\geq C^\frac2k|w_j^\top v_{m'}|^k +\delta$ holds for all $m'\ne m$ with high probability because $\|\tilde{w}_j\|\lesssim 2$ with high probability.
    Therefore, over the randomness of initialization of $w_j$, $w_j^\top v_m\geq \frac{1}{\sqrt{d}}$ and  $(w_j^\top v_m)^{p-2}\geq C|w_j^\top v_{m'}|^{p-2} +\delta(\frac{1}{\sqrt{d}})^{p-2}$ for all $m'\ne m$ with probability at least $\frac{M^{-C_p}}{A\sqrt{\log M}}$.
    Taking the uniform bound over all $v_m$, we know that the number of required neurons to satisfy $|J_m|\geq J_{\mathrm{min}}$ for all $m$ is at most $J\geq AJ_{\mathrm{min}}M^{C_p}(\log d)^\frac32$ (up to a constant factor, which can be absorbed in the definition of $A$).
\end{proofof}

\subsection{Descent Path for Population Gradient}\label{subsection:DescentPath}

This section discusses the assumption on the target/activation functions.
To translate \textit{weak recovery} to \textit{strong recovery} and establish alignment, we require that at any (positive) level of alignment $u=v_1^\top w^t>0$, the population correlation loss between the neuron and the target sub-problem $f_1(v_1^\top x)$ has the descent path on the sphere.
In other words, the population correlation loss as a function of $u$ should be monotonically decreasing with respect to $u \in (0,1]$. Observe that without such monotonicity, if the alignment $\alpha$ becomes large, higher-order terms in the Hermite expansion may generate a repulsive force that prevents the neuron from further aligning with the target direction $v_m$. 
For the well-specified setting (matching activation), this condition is automatically satisfied as shown in \cite{arous2021online}; whereas for the misspecified scenario, such a condition appeared in \cite{mousavi2024gradient} as an assumption, which we restate below, and then verify for specific choices of student activation functions.

Recall the Hermite expansion of one neuron $a^0\sigma(z+b^0)$: 
\begin{align}
    a^0\sigma(z+b^0) = \sum_{i=0}^\infty \frac{\alpha_i}{\sqrt{i!}} \He_i(z)
\end{align}
and 
the Hermite expansion of each sub-problem of the target function
\begin{align}
    f_m(z) = \sum_{i=p}^p \frac{\beta_{m,i}}{\sqrt{i!}}\He_i(z).
\end{align}
\begin{assumption}\label{assumption:NeuronPositiveCorrelation}
    The neuron satisfies $\alpha_i\beta_{1,i}>0$ for all $p\leq i\leq q$.
\end{assumption}
This assumption ensures that 
\begin{align}
    \mathbb{E}_{x\sim\mathcal{N}(0,I_d)}[a^0\sigma(x^\top w^t+b^0)f_1(v_1^\top x)]=\sum_{i=p}^q (\alpha_i \beta_{m,i})^i
\end{align}
is monotonically increasing with respect to $u=\alpha_i \beta_{m,i}\in (0,1]$.
We show that for certain choices of (randomized) activation function, a $\tilde{\Omega}(1)$ fraction of student neurons satisfy Assumption~\ref{assumption:NeuronPositiveCorrelation}. 
This indicates that Assumption~\ref{assumption:NeuronPositiveCorrelation} only affects the required width up to constant factor.
In the subsequent sections, we focus on the training dynamics of individual neurons that satisfy Assumption~\ref{assumption:NeuronPositiveCorrelation} (without explicitly mentioning so).

\subsubsection{ReLU Activation}\label{subsubsection:DescentPath-1}
For the ReLU activation, we verify this condition for target function in which at each task $f_m$, the non-zero Hermite coefficients have the same sign, i.e., for $\beta_{m,i},\beta_{m,j}\neq0$, we have $\mathrm{sign}(\beta_{m,i})=\mathrm{sign}(\beta_{m,j})$. For example, this condition is met when the link functions are pure Hermite polynomial. 
Then, the following lemma adapted from \cite{ba2023learning} ensures that $\alpha_i>0$ for all $i$ with probability at least $\frac14$. 
Because the distribution of $a^0$ is symmetric, Assumption \ref{assumption:NeuronPositiveCorrelation} holds with probability at least $\frac18$.
\begin{lemma}
    Given degree $q\geq 0$ and $b\sim [-C_b,C_b]$, 
    the $i$-th Hermite coefficient of $\mathrm{ReLU}(z+b)$ is positive with probability $\frac14$ for all $p\leq i\leq q$, if $C_b$ is larger than some constant that only depends on $q$.
    % %positive for $i=1,2,\cdots,p$, i.e., 
    % \begin{align}
    %     \mathbb{E}_{z\sim \mathcal{N}(0,1)}[\He_i(z)\mathrm{ReLU}(z+b)]>0\quad (i=1,2,\cdots,p).
    % \end{align}
\end{lemma}
\begin{proof}
    First note that for $i=1$, we have
    \begin{align}
        \mathbb{E}_{z\sim \mathcal{N}(0,1)}[\mathrm{ReLU}(z+b)\He_1(z)]
        % =\int_{z=-b}^\infty z(z+b)e^{-\frac{z^2}{2}}\mathrm{d}z
        =\frac12\int_{z=b}^\infty \frac{1}{\sqrt{2\pi}}e^{-\frac{z^2}{2}}\mathrm{d}z+\frac12>0.
    \end{align}
    Moreover, for $i\geq 2$, we have
    \begin{align}
        \mathbb{E}_{z\sim \mathcal{N}(0,1)}[\mathrm{ReLU}(z+b)\He_i(z)]=\frac{(-1)^i}{\sqrt{2\pi}}e^{-\frac{b^2}{2}}\He_{i-2}(b),
        \label{eq:Hermite-shifted-ReLU}
    \end{align}
    Because $\lim_{b\to-\infty}\He_i(b)=\infty$ for even $i$ and $-\infty$ for odd $i$, there exists some $b_0$ such that if $b\leq C_b'$ then $(-1)^pe^{\frac{b^2}{2}}\He_i(b)>0$.
    By taking $C_b\geq 2C_b'$, $b\leq C_b'$ (and thus the assertion) holds with probability $\frac14$.
\end{proof}

\subsubsection{General Polynomial Link Functions}\label{subsubsection:DescentPath-2}
To deal with general polynomial link functions, we randomize the student activations as follows, 
\begin{align}
    \sigma_j(z) = \sum_{i=p}^q \frac{\varepsilon_{i,j}}{\sqrt{i}}\He_i(z),
\end{align}
where $\varepsilon_{i}$ are independent Rademacher variables (taking $1, -1,$ and $0$ with equiprobability).
\begin{lemma}
    Given degree $q\geq 0$ and $b\sim [-C_b,C_b]$, for each $p_{\rm min}\leq p'\leq p_{\rm max}$, 
    the $i$-th Hermite coefficient of $a^0\sigma(z+b^0)$ is non-zero with probability $\Omega(C_b^{-1})$, for all $p'\leq i\leq q$. Here $\Omega$ hides constant only depending on $q$.
\end{lemma}
\begin{proof}
    Because of the randomized Hermite coefficient of the activation function, 
    $a^0\sigma(z)$ has positive coefficients with probability $2^{-(q-p+1)}$.
    As long as the bias $b^0$ is small, $a^0\sigma(z+b^0)$ also have positive coefficients.
\end{proof}

% \subsection*{Alignment of the First Layer via a Power Method Like Dynamics}
% \label{subsection:alignment--1}

% We assume that $\eta \leq \cC M^{-\frac12} d^{-\frac{p}{2}}$.

\subsection{Decomposition of Gradient Update}\label{subsection:alignment-0}
From now, we discuss the training dynamics of the first layer.
% This subsection shows the alignment of the first layer weight.
We focus on one neuron in $J_1$ that satisfies Assumption~\ref{assumption:NeuronPositiveCorrelation}.
To track the alignment during the dynamics, we define $\kappa^t_m = v_m^\top w_j^t$. 
We first consider the decomposition of the update into population and stochastic terms:
\begin{lemma}\label{lemma:OneStep}
    Suppose that $\eta^t=\eta \leq \cC d^{-1}$ and $\kappa^t_1 \geq \frac12 d^{-\frac12}$.
    With high probability, the update of $\kappa^{t}_m$ can be bounded as
    \begin{align}
        \kappa^{t+1}_1 \geq 
        \kappa_1^t 
        +\frac{\eta}{\sqrt{M}}\sum_{m=1}^M
        \sum_{i=p}^q
        \left[i\alpha_i\beta_{m,i} (\kappa^t_m)^{i-1} (v_1^\top v_m - \kappa_1^t\kappa_m^t)\right]
        -
        \kappa_1^t\eta^2 \CB^2 d + \eta v_1^\top (I-w^t{w^t}^\top)  Z^t.
        \label{eq:OneStep-102}
    \end{align}
    Moreover, $ \kappa^{t+1}_m$ is evaluated as
    \begin{align}
     &  \kappa_m^t 
        +\frac{\eta}{\sqrt{M}}\sum_{m'=1}^M
        \sum_{i=p}^q
        \left[i\alpha_i\beta_{m',i} (\kappa^t_{m'})^{i-1} (v_m^\top v_{m'} - \kappa_m^t\kappa_{m'}^t)\right]-\frac{|\kappa_m^t|+\eta \CB d^\frac12}{2}\eta^2 \CB^2 d 
        \\ & \quad+ \eta v_m^\top (I-w^t{w^t}^\top) Z^t
    \\ &  \leq   \kappa^{t+1}_m  \\ &\leq   \kappa_m^t 
        +\frac{\eta}{\sqrt{M}}\sum_{m'=1}^M
        \sum_{i=p}^q
        \left[i\alpha_i\beta_{m',i} (\kappa^t_{m'})^{i-1} (v_m^\top v_{m'} - \kappa_m^t\kappa_{m'}^t)\right]
        +
        \frac{|\kappa_m^t|+\eta \CB d^\frac12}{2}\eta^2 \CB^2 d 
         \\ & \quad+ \eta v_m^\top (I-w^t{w^t}^\top) Z^t.
        \label{eq:OneStep-103}
    \end{align}
    Here $Z^t$ is a mean-zero random variable satisfying $\|Z^t\|=\tilde{O}(1)$ with high probability.
    For any $v\in \mathbb{R}^{d}$ with $\|v\|=O(1)$ that is independent from $Z^t$, $|v^\top Z^t|=\tilde{O}(1)$ with high probability.
    Also, $|\kappa_m^t-\kappa_m^{t+1}|=\tilde{O}(\eta)$ with high probability.
\end{lemma}
To prove Lemma~\ref{lemma:OneStep}, we first establish the following characterization of the stochastic gradient.
\begin{lemma}\label{lemma:OneStep-1}
    The stochastic gradient $-\nabla_w y^ta^t\sigma({w^t}^\top x^t+b^t)$ is decomposed as
    \begin{align}
     &   -\nabla_w y^ta^t\sigma({w^t}^\top x^t+b^t)
    \\ &    = 
        -\frac{1}{\sqrt{M}}\sum_{m=1}^M
        \sum_{i=p}^q
        \left[i\alpha_i\beta_{m,i} (\kappa_m^t)^{i-1} v_m
        +\sqrt{(i+2)(i+1)}\alpha_{i+2}\beta_{m,i} (\kappa_m^t)^{i}w^t\right]+Z^t,
    \end{align}
    where $Z^t$ is a mean-zero random variable such that $\|Z^t\|=\tilde{O}(d^\frac12)$ and $|v^\top Z^t|=\tilde{O}(1)$ for any fixed $v\in \mathbb{S}^{d-1}$ with high probability.
    Also, $\|\nabla_w y^ta^t\sigma({w^t}^\top x^t+b^t)\|=\tilde{O}(d^\frac12)$ with high probability, and for any fixed $v$ with $\|v\|=O(1)$, $\left(\nabla_w y^ta^t\sigma({w^t}^\top x^t+b^t)\right)^\top v = \tilde{O}(1)$ with high probability.
\end{lemma}
\begin{proof}
    For $i$-th Hermite polynomial $\He_i$ and $u\in \mathbb{S}^{d-1}$, we have that
    \begin{align}
        &\mathbb{E}_{x\sim \mathcal{N}(0,I_d)}[\He_i(x_1)f(u^\top x)x_1]
        = i u_1^{i-1}\mathbb{E}_{x\sim \mathcal{N}(0,I_d)}[f^{(i-1)}(u^\top x)]+u_1^{i+1}\mathbb{E}_{x\sim \mathcal{N}(0,I_d)}[f^{(i+1)}(u^\top x)],\\
        &\mathbb{E}_{x\sim \mathcal{N}(0,I_d)}[\He_i(x_1)f(u^\top x)x_2]
        = u_1^{i}u_2\mathbb{E}_{x\sim \mathcal{N}(0,I_d)}[f^{(i+1)}(u^\top x)].
    \end{align}
    Therefore, 
    \begin{align}
        \mathbb{E}_{x\sim \mathcal{N}(0,I_d)}[\He_i(x_1)f(u^\top x)x]
        =
        \begin{pmatrix}
            iu_1^{i-1}
            \\ 0
            \\ \vdots
            \\ 0
        \end{pmatrix}
        \mathbb{E}_{x\sim \mathcal{N}(0,I_d)}[f^{(i-1)}(u^\top x)]
        +
        u_1^{i}u
        \mathbb{E}_{x\sim \mathcal{N}(0,I_d)}[f^{(i+1)}(u^\top x)]
        .
    \end{align}
    Using this fact, the population gradient is computed as
    \begin{align}
    &\nabla_w \mathbb{E}\left[y^t a^t \sigma({w^t}^\top x+b^t)\right]
    \\&=\mathbb{E}\left[\frac{1}{\sqrt{M}}\sum_{m=1}^M f_m(v_m^\top x)a^t \sigma'({w^t}^\top x+b^t)x\right]
    \\ & = 
    \mathbb{E}\left[\frac{1}{\sqrt{M}}\left(\sum_{m=1}^M\sum_{i=p}^p \frac{\beta_{m,i}}{\sqrt{i!}}\He_i(z)\right)
    \left(\sum_{i=0}^\infty \frac{i\alpha_i}{\sqrt{i!}} \He_{i-1}(z)\right)
    x\right]
    \\ &=
    \frac{1}{\sqrt{M}}\sum_{m=1}^M
        \sum_{i=p}^q
        \left[i^2\alpha_i\beta_{m,i} (v_m^\top w^t)^{i-1}v_m
        +\sqrt{(i+2)(i+1)}\alpha_{i+2}\beta_{m,i} (v_m^\top w^t)^{i}w^t\right]
        \\ &= \frac{1}{\sqrt{M}}\sum_{m=1}^M
        \sum_{i=p}^q
        \left[i\alpha_i\beta_{m,i} (\kappa_m^t)^{i-1} v_m
        +\sqrt{(i+2)(i+1)}\alpha_{i+2}\beta_{m,i} (\kappa_m^t)^{i}w^t\right]
        .
    \end{align}
    We define $Z^t$ as the difference between the population gradient and the empirical gradient:
    \begin{align}
        Z^t = -\nabla_w y^ta^t\sigma({w^t}^\top x^t+b^t) + \nabla_w \mathbb{E}\left[y^t a^t \sigma({w^t}^\top x+b^t)\right].
    \end{align}
    We have
    \begin{align}
      &  -\nabla_w y^ta^t\sigma({w^t}^\top x^t+b^t)
      \\ &  = 
        -\frac{1}{\sqrt{M}}\sum_{m=1}^M
        \sum_{i=p}^q
        \left[i\alpha_i\beta_{m,i} (\kappa_m^t)^{i-1} v_m
        +\sqrt{(i+2)(i+1)}\alpha_{i+2}\beta_{m,i} (\kappa_m^t)^{i}w^t\right]+Z^t
    \end{align}
    It is easy to see that $\mathbb{E}_{x^t,y^t}[Z^t]=0$, and
    $\|Z^t\|=\tilde{O}(d^\frac12)$  %$\|\nabla_w y^t\sigma({w^t}^\top x^t+b^t)\|=\tilde{O}(d^\frac12)$
    with high probability by \eqref{eq:pNormofy-5} of Lemma~\ref{lemma:HPB-Y}.
    \eqref{eq:pNormofy-5} also yields that $|v^\top Z^t|=\tilde{O}(1)$ with high probability for any fixed $v\in \mathbb{S}^{d-1}$.
    Finally, the norm of $\nabla_w y^ta^t\sigma({w^t}^\top x^t+b^t)=y^t a^t \sigma'({w^t}^\top x^t+b^t)x^t$ is of order $\tilde{O}(d^\frac12)$ with high probability, and for any fixed $v$ with $\|v\|=O(1)$, $\left(\nabla_w y^ta^t\sigma({w^t}^\top x^t+b^t)\right)^\top v = \tilde{O}(1)$ with high probability.
\end{proof}

\begin{proofof}[Lemma~\ref{lemma:OneStep}]
    First, we consider $\kappa_1^t$. We have
    \begin{align}
     &   \|w^t - \eta (I-w^t{w^t}^\top) \nabla_w (-y^ta^t\sigma({w^t}^\top x^t+b^t))\|^{-1}
      \\ &=\left(1 + \eta^2\|(I-w^t{w^t}^\top) \nabla_w y^ta^t\sigma({w^t}^\top x^t+b^t)\|^2\right)^{-\frac12}
        \\ & \geq 1 - \frac{\eta^2}{2}\|(I-w^t{w^t}^\top) \nabla_w y^ta^t\sigma({w^t}^\top x^t+b^t)\|^2
        \\ & \geq 1 - \frac{\eta^2}{2}\|\nabla_w y^ta^t\sigma({w^t}^\top x^t+b^t)\|^2.
        \label{eq:OneStep-101}
    \end{align}
    By using this, with high probability, 
    \begin{align}
     v_1^\top w^{t+1}&=
       v_1^\top \frac{w^t - \eta (I-w^t{w^t}^\top)  \nabla_w (-y^ta^t\sigma({w^t}^\top x^t+b^t))}{\|w^t - \eta (I-w^t{w^t}^\top) \nabla_w (-y^ta^t\sigma({w^t}^\top x^t+b^t))\|}
        \\ & \geq \kappa_1^t 
        + \eta v_1^\top(I-w^t{w^t}^\top)\nabla_w y^ta^t\sigma({w^t}^\top x^t+b^t)- \frac{\kappa_1^t\eta^2}{2}\|\nabla_w y^ta^t\sigma({w^t}^\top x^t+b^t)\|^2
        \\ & \quad-\frac{\eta^3}{2}|v_1^\top (I-w^t{w^t}^\top) \nabla_w (-y^ta^t\sigma({w^t}^\top x^t+b^t))|\|\nabla_w y^ta^t\sigma({w^t}^\top x^t+b^t)\|^2
        \\ & \geq \kappa_1^t + \eta v_1^\top \nabla_w\mathbb{E}\left[ y^ta^t\sigma({w^t}^\top x^t+b^t)\right] + v_1^\top Z^t-\frac{\kappa_1^t\eta^2}{2}\|\nabla_w y^ta^t\sigma({w^t}^\top x^t+b^t)\|^2
        \\ &
        \quad -\frac{\eta^3}{2}\|\nabla_w y^ta^t\sigma({w^t}^\top x^t+b^t)\|^3
        \\ & \overset{(i)}{\geq} \kappa_1^t + \eta v_1^\top(I-w^t {w^t}^\top)\frac{1}{\sqrt{M}}\sum_{m=1}^M
        \sum_{i=p}^q
        \left[i\alpha_i\beta_{m,i} (\kappa^t_m)^{i-1} v_m\right] + \eta v_1^\top(I-w^t {w^t}^\top) Z^t
     \\ & \quad   -
        \frac{\kappa_1^t\eta^2 \CB^2 d}{2} 
        -\frac{\eta^3 \CB^3 d^\frac32}{2}
        \\ & \overset{(ii)}{\geq} \kappa_1^t 
        +\eta \frac{1}{\sqrt{M}}\sum_{m=1}^M
        \sum_{i=p}^q
        \left[i\alpha_i\beta_{m,i} (\kappa^t_m)^{i-1} (v_1^\top v_m - \kappa_1^t\kappa_m^t)\right]
        % + \frac{\eta}{\sqrt{M}}p\alpha_{p}\beta_{1,p} (1-(\kappa_1^t)^2)(\kappa_1^t)^{p-1}
        % -
        % \frac{\eta}{\sqrt{M}}\sum_{m=2}^M
        % \sum_{i=p}^q
        % \left[i\alpha_i\beta_{m,i} (\kappa^t_m)^{i-1}\max_{m\ne 1}|v_1^\top v_m|\right]
        -
        \kappa_1^t\eta^2 \CB^2 d 
         \\ & \quad+ \eta v_1^\top (I-w^t{w^t}^\top) Z^t, 
    \end{align}
    where we used Lemma~\ref{lemma:OneStep-1} in $(i)$, and $(ii)$ is due to the fact that we take $\eta \leq \cC d^{-1}$ and $\kappa_1^t\geq d^{-\frac12}$, and hence $\eta^3 \CB^3 d^\frac32 \leq \kappa_1^t\eta^2\CB^2 d$ and $-\frac{\kappa_1^t\eta^2 \CB^2 d}{2} 
    -\frac{\eta^3 \CB^3 d^\frac32}{2}\geq -\kappa_1^t\eta^2 \CB^2 d$.
    Thus we obtained \eqref{eq:OneStep-102}.

    In the same way, for the lower bound on $\kappa^{t}_m$,  we have
    \begin{align}
        \kappa^{t+1}_m &\geq  \kappa_m^t 
        + \eta v_m^\top(I-w^t{w^t}^\top)\nabla_w y^ta^t\sigma({w^t}^\top x^t+b^t)- \frac{|\kappa_m^t|\eta^2}{2}\|\nabla_w y^ta^t\sigma({w^t}^\top x^t+b^t)\|^2
        \\ & \quad-\frac{\eta^3}{2}|v_m^\top (I-w^t{w^t}^\top) \nabla_w (-y^ta^t\sigma({w^t}^\top x^t+b^t))|\|\nabla_w y^ta^t\sigma({w^t}^\top x^t+b^t)\|^2
        \\ & \geq 
       \kappa_m^t 
        +\frac{\eta}{\sqrt{M}}\sum_{m'=1}^M
        \sum_{i=p}^q
        \left[i\alpha_i\beta_{m',i} (\kappa^t_{m'})^{i-1} (v_m^\top v_{m'} - \kappa_m^t\kappa_{m'}^t)\right]
        -
        \frac{|\kappa_m^t|+\eta \CB d^\frac12}{2}\eta^2 \CB^2 d 
         \\ & \quad+ \eta v_m^\top (I-w^t{w^t}^\top) Z^t.
    \end{align}
    As for the upper bound, 
    \begin{align}
        \kappa^{t+1}_m &\leq  \kappa_m^t 
        + \eta v_m^\top(I-w^t{w^t}^\top)\nabla_w y^ta^t\sigma({w^t}^\top x^t+b^t)+ \frac{|\kappa_m^t|\eta^2}{2}\|\nabla_w y^ta^t\sigma({w^t}^\top x^t+b^t)\|^2
        \\ & \quad+\frac{\eta^3}{2}|v_m^\top (I-w^t{w^t}^\top) \nabla_w (-y^ta^t\sigma({w^t}^\top x^t+b^t))|\|\nabla_w y^ta^t\sigma({w^t}^\top x^t+b^t)\|^2
        \\ & \leq 
       \kappa_m^t 
        +\frac{\eta}{\sqrt{M}}\sum_{m'=1}^M
        \sum_{i=p}^q
        \left[i\alpha_i\beta_{m',i} (\kappa^t_{m'})^{i-1} (v_m^\top v_{m'} -\kappa_m^t\kappa_{m'}^t)\right]
        +
        \frac{|\kappa_m^t|+\eta \CB d^\frac12}{2}\eta^2 \CB^2 d 
         \\ & \quad+ \eta v_m^\top (I-w^t{w^t}^\top)  Z^t.
    \end{align}
    Finally, we check that $|\kappa_m^t-\kappa_m^{t+1}|=\tilde{O}(\eta)$.
    From the above argument, we have
    \begin{align}
        |\kappa_m^{t+1}-\kappa_m^t|
        &\leq \eta \left|v_m^\top(I-w^t{w^t}^\top)\nabla_w y^ta^t\sigma({w^t}^\top x^t+b^t)\right|
        +\frac{|\kappa_m^t|\eta^2}{2}\|\nabla_w y^ta^t\sigma({w^t}^\top x^t+b^t)\|^2
        \\ &\quad + \frac{\eta^3}{2}|v_m^\top (I-w^t{w^t}^\top) \nabla_w (-y^ta^t\sigma({w^t}^\top x^t+b^t))|\|\nabla_w y^ta^t\sigma({w^t}^\top x^t+b^t)\|^2.
    \end{align}
    The first term is bounded by $\tilde{O}(\eta)$ because $\left|v_m^\top(I-w^t{w^t}^\top)\nabla_w y^ta^t\sigma({w^t}^\top x^t+b^t)\right|=\tilde{O}(1)$.
    The second term is bounded by $\tilde{O}(\eta)$ when $\eta \leq \cC d^{-1}$, because $\|\nabla_w y^ta^t\sigma({w^t}^\top x^t+b^t)\|^2=\tilde{O}(d)$.
    The third term is bounded by $\tilde{O}(\eta)$ when $\eta \leq \cC d^{-1}$, because $|v_m^\top (I-w^t{w^t}^\top) \nabla_w (-y^ta^t\sigma({w^t}^\top x^t+b^t))|=\tilde{O}(1)$ and $\|\nabla_w y^ta^t\sigma({w^t}^\top x^t+b^t)\|^2=\tilde{O}(d)$.
    Therefore, we obtained that $|\kappa_m^t-\kappa_m^{t+1}|=\tilde{O}(\eta)$.
\end{proofof}

\subsection{Phase I: Weak Recovery for One Direction}\label{subsection:alignment-1}
Based on Lemma~\ref{lemma:OneStep}, we analyze the stochastic gradient update of the first-layer parameters. 
The goal of this subsection is to prove the formal version of Lemma \ref{lemma:AlignmentFirstPhase}. 
\newtheorem*{lemma:AlignmentFirstPhase-1}{\rm\bf Lemma~\ref{lemma:AlignmentFirstPhase} (formal)}
\begin{lemma:AlignmentFirstPhase-1} 
    Suppose that  $|v_{m'}^\top v_m|\leq \cC M^{-1}$ for all $m\ne m'$ and $M\leq \cC d^{\frac12}$, and $\eta^t=\eta \leq \cC M^{-\frac12}d^{-\frac{p}{2}}$.
    Then, with high probability, there exists some time $t_1\leq T_{1,1}=\Theta(\eta^{-1}M^{\frac12}d^{\frac{p-2}{2}})$ such that the following holds:
    \begin{itemize}
        \item[(i)] $\kappa_1^{t_1}> \cA$, and 
        \item[(ii)] $|\kappa_m^{t_1}|\leq 5\cB M^{-1}$, for all $m=2,\cdots,M$.
    \end{itemize}
\end{lemma:AlignmentFirstPhase-1} 
We start with the following lemma, which introduces (deterministic) auxiliary sequences that upper and lower bound the stochastic updates of $\kappa_m^t$.
\begin{lemma}\label{lemma:Lower}
    Suppose that $|v_{m'}^\top v_m|\leq \cC M^{-1}$ for all $m\ne m'$ and $M\leq \cC d^{\frac12}$.
    For all $s=0,1,\cdots,t$, we assume that
    \begin{itemize}
        \item[(a)] $\kappa^s_1\leq \cA$ (only required for (i): Lower bound), 
        \item[(b)] $|\kappa_m^s|\leq \kappa_1^s$ for all $m=2,\cdots,M$, and
        \item[(c)] $|\kappa_m^s| \leq \CA\cB M^{-1}$ for all $m=2,\cdots,M$.
    \end{itemize}
    Then, taking $\eta^t=\eta \leq  \cC M^{-\frac12}d^{-\frac{p}{2}}$, we have the following bounds.

     \begin{itemize}[leftmargin=*]
   \item[{\bf(i):}]{\bf Lower bound of $\kappa^s_1$:} For $\kappa^{t}_1$, we have
     \begin{align}\label{eq:ConnectODE-21}
         \kappa^{s+1}_1\geq (1-\cA)\kappa^{0}_1
         + \eta(1-\cA)\sum_{s'=0}^{s}\frac{p\alpha_p\beta_{1,p}}{\sqrt{M}} (\kappa_1^{s'})^{p-1},
     \end{align}
     for $s=0,1,\cdots,t$.
     Consequently, by introducing an auxiliary sequence $(P^s)_{s=0}^{t+1}$ with $P^0=(1-\cA)\kappa^0_1$, and
    \begin{align}\label{eq:DefineAuxiliaryA}
      P^{s+1} &= P^{s}
         +\eta\cA\frac{p\alpha_p\beta_{1,p}}{\sqrt{M}} (P^{s})^{p-1}\quad (s=0,1,\cdots,t),
    \end{align}
   $\kappa^s_1$ is lower-bounded by $P^s$ for all $s=0,1,\cdots,t+1$, with high probability. 
   
   \item[{\bf (ii):}]{\bf Upper bound of  $\max_{m\ne 1}|\kappa_m^s|$:} For an auxiliary sequence $(Q^s)_{s=0}^{t+1}$ with
 $Q^0 = (1+\cA)\max\{ \max_{m\ne 1}|\kappa_m^{0}|,$ $\frac12 d^{-\frac12}\}$, and
    \begin{align}\label{eq:ConnectODE-6}
        Q^{s+1} = Q^{s}
       +(1+\cA)\frac{\eta}{\sqrt{M}}p\max_{m}|\alpha_p\beta_{m,i}| (Q^{s})^{p-1}
       +\cB
       \frac{\eta}{M^\frac32}p\alpha_p\beta_{1,p} (\kappa^s_{1})^{p-1},
         % Q^{\tau+1}=\max\{\max_{m\ne 1}  | \kappa_m^{0}|, \cA \kappa_1^0\}+\sum_{s=0}^\tau \frac{\eta}{\sqrt{M}}p\alpha_p\beta_{m,p} (Q^s)^{p-1} + 3\cA\sum_{s=0}^\tau
       % \frac{\eta}{\sqrt{M}}p\alpha_p\beta_{1,p} (\kappa^s_{1})^{p-1}
    \end{align}
     $\max_{m\ne 1}|\kappa_m^s|$ is upper-bounded by $Q^s$ for all $s=0,1,\cdots,t+1$, with high probability.
    \end{itemize}
\end{lemma}
\begin{proof}
    \paragraph{(i) Lower bound of $\kappa^{s}_1$ by $P^s$.}
    % We assume that $\kappa^s_1 < \cA$.
    If $\kappa_1^t\geq \frac12 d^{-\frac12}$, we have
    $\kappa_1^t\geq \eta \CB d^\frac12$ by the choice of $\eta$.
    If $\kappa_1^s\geq \frac12 d^{-\frac12}$, 
    by Lemma~\ref{lemma:OneStep}, we have
    \begin{align}
         \kappa^{s+1}_1 &\geq 
        \kappa_1^s 
        +\frac{\eta}{\sqrt{M}}\sum_{m=1}^M
        \sum_{i=p}^q
        \left[i\alpha_i\beta_{m,i} (\kappa^s_m)^{i-1} (v_1^\top v_m - \kappa_1^s\kappa_m^s)\right]
        -
        \kappa_1^s\eta^2 \CB^2 d + \eta v_1^\top (I-w^s{w^s}^\top)  Z^s
        \\ & \geq 
        \kappa_1^s 
        + \frac{\eta}{\sqrt{M}}p\alpha_p\beta_{1,p}(1-(\kappa_1^s)^2)(\kappa_1^s)^{p-1}
        +\frac{\eta}{\sqrt{M}}\sum_{m=2}^M
        \sum_{i=p}^q
        \left[i\alpha_i\beta_{m,i} (\kappa^s_m)^{i-1} (v_1^\top v_m - \kappa_1^s\kappa_m^s)\right]
        \\ & \quad -
        \kappa_1^s\eta^2 \CB^2 d + \eta v_1^\top (I-w^s{w^s}^\top)  Z^s
        \\ & \geq 
        \kappa_1^s 
        + \frac{\eta}{\sqrt{M}}p\alpha_p\beta_{1,p}(1-(\kappa_1^s)^2)(\kappa_1^s)^{p-1}
    -q^2\eta\sqrt{M}\max_{m\ne 1}|\kappa_m^s|^{p-1}\max_{m\ne 1}|v_1^\top v_m|\max_{m,i}|\alpha_i\beta_{m,i} |
       \\ & \quad -q^2\eta\sqrt{M}\max_{m\ne 1}|\kappa_m^s|^{p}\max_{m,i}|\alpha_i\beta_{m,i}| -
        \kappa_1^s\eta^2 \CB^2 d + \eta v_1^\top (I-w^s{w^s}^\top)  Z^s 
        \label{eq:ConnectODE-201}
        \\ & \geq \kappa_1^s 
        + \frac{\eta}{\sqrt{M}}p\alpha_p\beta_{1,p}(1-(\kappa_1^s)^2)(\kappa_1^s)^{p-1}
    -q^2\eta\sqrt{M}(\kappa_1^s)^{p-1}\cC M^{-1} 
       \\ & \quad -q^2\eta\sqrt{M}(\kappa_1^s)^{p-1}\CA\cB M^{-1} -
        \eta\kappa_1^s\cC M^{-\frac12}d^{-\frac{p-2}{2}} \CB^2 + \eta v_1^\top (I-w^s{w^s}^\top)  Z^s 
        % \\ & \geq  \kappa_1^s 
        % + \frac{\eta}{\sqrt{M}}p\alpha_p\beta_{1,p}(1-\cA)(\kappa_1^s)^{p-1}
        % -  \cA\frac{\eta}{\sqrt{M}}p\alpha_p\beta_{1,p}(\kappa_1^s)^{p-1}
        . 
    \end{align}
    Note that $(\kappa_1^s)^2\leq \cA^2\leq \frac15\cA$, $q^2\eta\sqrt{M}(\kappa_1^s)^{p-1}\cC M^{-1} \leq \frac15\cA\frac{\eta}{\sqrt{M}}p\alpha_p\beta_{1,p}(\kappa_1^s)^{p-1}$, 
    $q^2\eta\sqrt{M}(\kappa_1^s)^{p-1}\CA\cB M^{-1} \leq \frac15\cA\frac{\eta}{\sqrt{M}}p\alpha_p\beta_{1,p}(\kappa_1^s)^{p-1}$, and
    $\eta\kappa_1^s\cC M^{-\frac12}d^{-\frac{p-2}{2}} \CB^2
    % \leq \eta\kappa_1^s \cC M^{-\frac12}(2\kappa_1^s)^{p-2} 
    \leq \frac15\cA\frac{\eta}{\sqrt{M}}p\alpha_p\beta_{1,p}(\kappa_1^s)^{p-1}$, (where we used $\kappa_1^s\geq \frac12 d^{-\frac12}$ for the last statement).
    Thus, we obtained that
    \begin{align}
        \kappa^{s+1}_1 &\geq\kappa_1^s 
        + (1-\frac45\cA)\frac{\eta}{\sqrt{M}}p\alpha_p\beta_{1,p}(\kappa_1^s)^{p-1}+ \eta v_1^\top (I-w^s{w^s}^\top)  Z^s .
        \label{eq:ConnectODE-104}
    \end{align}

    We prove the assertion by induction.
    % Suppose that $\kappa_1^{s}\geq (1-\cA)\kappa_1^{0}$ holds for all $s=0,\cdots,\tau$ for some $\tau\leq t$.
    Suppose that \eqref{eq:ConnectODE-21} holds for  $s=0,\cdots,\tau$ for some $\tau\leq t$.
    Note that this implies $\kappa_1^{s}\geq (1-\cA)\kappa_1^{0}$ and $\kappa_1^{s}\geq \frac12d^{-\frac12}$.
    By applying \eqref{eq:ConnectODE-104}, we have 
    \begin{align}
       \kappa^{\tau+1}_1   & \geq \kappa^{\tau}_1+(1-\frac45\cA)\frac{\eta}{\sqrt{M}}p\alpha_p\beta_{1,p}(\kappa_1^\tau)^{p-1}+ \eta v_1^\top (I-w^\tau{w^\tau}^\top)  Z^\tau
         \\ & 
     \geq  \kappa^{0}_1+ \sum_{s=0}^\tau (1-\frac45\cA)\frac{\eta}{\sqrt{M}}p\alpha_p\beta_{1,p}(\kappa_1^s)^{p-1}+ \sum_{s=0}^\tau \eta v_1^\top (I-w^s{w^s}^\top)  Z^s
         \label{eq:ConnectODE-1}.
    \end{align}    
    If $\tau\leq \CA M (\kappa_1^0)^{2-2p}$, then
    \begin{align}
        \sum_{s=0}^{\tau} \eta v_1^\top (I-w^s{w^s}^\top)Z^{s} 
         \leq \eta \CB \sqrt{\tau}
         \leq \cC \CB \kappa_1^0 \leq \cA \kappa_1^0
         ,\label{eq:Noise-101}
    \end{align}
    with high probability.
    If $\tau> \CA M (\kappa_1^0)^{2-2p}$,
    \begin{align}
        \sum_{s=0}^{\tau} \eta v_1^\top (I-w^s{w^s}^\top)Z^{s} 
         &\leq \eta \CB \sqrt{\tau}
         \leq \eta \tau \CB \CA^{-\frac12} M^{-\frac12} (\kappa_1^0)^{p-1}
         \leq \frac15\cA\frac{\eta}{\sqrt{M}} \tau p\alpha_p\beta_{1,p}((1-\cA)\kappa_1^0)^{p-1}
         \\ & \leq \frac15\sum_{s=0}^{\tau}\cA\frac{\eta}{\sqrt{M}}p\alpha_p\beta_{1,p} (\kappa_1^s)^{p-1}
         ,\label{eq:Noise-102}
    \end{align}
    with high probability.
    Applying the above evaluations to \eqref{eq:ConnectODE-1}, we have
    \begin{align}
   \eqref{eq:ConnectODE-1}&\geq (1-\cA)\kappa^{0}_1
     + \sum_{s=0}^{\tau}(1-\cA)\frac{\eta}{\sqrt{M}}p\alpha_p\beta_{1,p} (\kappa_1^s)^{p-1}.
          \label{eq:ConnectODE-2}
    \end{align}
    Thus the \eqref{eq:ConnectODE-21} holds also for $s=\tau+1$.
    The induction proves that \eqref{eq:ConnectODE-21} holds until $s=t$.
        
    By repeatedly using \eqref{eq:DefineAuxiliaryA}, the update of $(P^t)_{t=0}^\tau$ is equivalent to
    \begin{align}
        P^{t+1}&= P^{0}
         + \sum_{s=0}^{t}\frac{\eta}{\sqrt{M}}(1-\cA)p\alpha_p\beta_{1,p}(P^s)^{p-1}
         % \\ &=(1-\cA)\kappa_1^0
         % + \sum_{s=0}^{t}\eta(1-\cA)\frac{\alpha_p\beta_{1,p}}{\sqrt{M}} (\underline{P}^s)^{p-1}(1-\eta\lambda)^{t-s}.
    \end{align}
    By comparing this and \eqref{eq:ConnectODE-21}, we conclude that $\kappa_1^s$ is lower bounded by $P^{s}$ for $s=1,2,\cdots,t+1$.

    \paragraph{(ii) Upper bound of $\max_{m\ne 1}|\kappa^{s}_m|$ by $Q^s$:}
    According to Lemma~\ref{lemma:OneStep}, $|\kappa^{s+1}_{m}-\kappa^s_{m}|\leq \CB \eta$ with high probability.
    Thus, the sign of $\kappa^{s+1}_{m}$ is the same as that of $\kappa^s_{m}$, or $|\kappa^{s+1}_{m}|\leq \CB \eta$.
    Therefore,
    \begin{align}
      &  |\kappa^{s+1}_m|
      \\ & \leq \max\bigg\{\CB \eta,
        \bigg|\kappa_m^s 
        +\frac{\eta}{\sqrt{M}}\sum_{m=1}^M
        \sum_{i=p}^q
        \left[i\alpha_i\beta_{m,i} (\kappa^s_m)^{i-1} (v_1^\top v_m - \kappa_m^t\kappa_m^s)\right]
         \\ & \quad\quad\quad\quad\quad\quad\quad\quad\!\!\!\! +
        \frac{|\kappa_m^s|+\eta \CB d^\frac12}{2}\eta^2 \CB^2 d + \eta v_m^\top (I-w^s{w^s}^\top)  Z^s\bigg|\bigg\}
        \\ & \leq \max\bigg\{\CB \eta,
        \bigg|\kappa_m^s 
        +\frac{\eta}{\sqrt{M}}\sum_{i=p}^q
        \left[i\alpha_i\beta_{1,i} (\kappa^s_{1})^{i-1} (v_m^\top v_{1} - \kappa_m^s\kappa_{1}^s)\right]
       \\ & \quad\quad\quad\quad\quad\quad\quad\quad+\frac{\eta}{\sqrt{M}}\sum_{m'\ne 1}^M
        \sum_{i=p}^q
        \left[i\alpha_i\beta_{m',i} (\kappa^s_{m'})^{i-1} (v_m^\top v_{m'} - \kappa_m^s\kappa_{m'}^s)\right]  \\ & \quad\quad\quad\quad\quad\quad\quad\quad
        + 
        \frac{|\kappa_m^s|+\eta \CB d^\frac12}{2}\eta^2 \CB^2 d  +\eta v_m^\top (I-w^s{w^s}^\top)  Z^s\bigg|\bigg\}
        \\ & \leq \max\bigg\{\CB \eta,
        \bigg|\kappa_m^t 
        +\frac{\eta}{\sqrt{M}}\sum_{i=p}^q
       i\alpha_i\beta_{1,i} (\kappa^s_{1})^{i-1} v_m^\top v_{1} \\ &\quad\quad\quad\quad\quad\quad\quad\quad+\frac{\eta}{\sqrt{M}}\sum_{m'\ne 1}^M
        \sum_{i=p}^q
        \left[i\alpha_i\beta_{m',i} (\kappa^s_{m'})^{i-1} (v_m^\top v_{m'} - \kappa_m^s\kappa_{m'}^s)\right]
       \\ & \quad\quad\quad\quad\quad\quad\quad\quad+ \frac{|\kappa_m^s|+\eta \CB d^\frac12}{2}\eta^2 \CB^2 d+ \eta v_m^\top (I-w^s{w^s}^\top)  Z^s\bigg|\bigg\}.
    \end{align}
    
    We have
    \begin{align}
        &
     \bigg|\sum_{i=p}^q
       i\alpha_i\beta_{1,i} (\kappa^s_{1})^{i-1} v_m^\top v_{1}+\sum_{m'\ne 1}^M
        \sum_{i=p}^q
        \left[i\alpha_i\beta_{m',i} (\kappa^s_{m'})^{i-1} (v_m^\top v_{m'} - \kappa_m^s\kappa_{m'}^s)\right]\bigg|
        \\ & \leq 
     \bigg|\sum_{i=p}^q
       i\alpha_i\beta_{1,i} (\kappa^s_{1})^{i-1} v_m^\top v_{1}\bigg| + \bigg|
        \sum_{i=p}^q
        \left[i\alpha_i\beta_{m,i} (\kappa^s_{m})^{i-1} (1 - (\kappa_m^s)^2)\right]\bigg|
       \\ & \quad\quad +\bigg|\sum_{m'\ne 1}^M
        \sum_{i=p}^q
        \left[i\alpha_i\beta_{m',i} (\kappa^s_{m'})^{i-1} (v_m^\top v_{m'} - \kappa_m^s\kappa_{m'}^s)\right]\bigg|
     % \\&   =\sum_{i=p}^q
     %    \left[i\alpha_i\beta_{1,i} (\kappa^s_{1})^{i-1} (v_m^\top v_{1} - \kappa_m^s\kappa_{1}^s)\right]
     %    +p\alpha_p\beta_{m,p} (\kappa^s_{m})^{p-1} (1 - (\kappa_m^s)^2)
     %   \\ & \quad +\sum_{i=p+1}^q
     %    \left[i\alpha_i\beta_{m,i} (\kappa^s_{m})^{i-1} (1 - (\kappa_m^s)^2)\right]+\sum_{m'\ne 1,m}^M\sum_{i=p}^q
     %    \left[i\alpha_i\beta_{m',i} (\kappa^s_{m'})^{i-1} (v_m^\top v_{m'} - \kappa_m^s\kappa_{m'}^s)\right]
    \\ &\leq q^2 \max_{i}|\alpha_i\beta_{1,i}| (\kappa^s_{1})^{p-1} \max_{m'}|v_1^\top v_{m'}|+
    p\max_{m'}|\alpha_p\beta_{m',p}| |\kappa^s_{m}|^{p-1}
   +  q^2 \max_{m',i}|\alpha_i\beta_{m',i}||\kappa^s_{m}|^{p}
      \\ & \quad+Mq^2\max_{m',i}|\alpha_i\beta_{m',i}| \max_{m\ne 1}|\kappa^s_{m}|^{p-1}( \max_{m',m}|v_m^\top v_{m'}|+\max_{m'\ne 1}|\kappa_m^s|) 
     \\ &\leq q^2 \max_{i}|\alpha_i\beta_{1,i}| (\kappa^s_{1})^{p-1} \cC M^{-1}+
    p\alpha_p\beta_{m,p} (\kappa^s_{m})^{p-1}
   +  q^2 \max_{m,i}|\alpha_i\beta_{m,i}|(\kappa^s_{m})^{p-1}\CA\cB M^{-1}
      \\ & \quad+
   Mq^2\max_{m,i}|\alpha_i\beta_{m,i}| \max_{m\ne 1}(\kappa^s_{m})^{p-1}(\cC M^{-1}+\CA\cB M^{-1})
   \\ & \leq (1+\frac13\cA) p\max_{m}|\alpha_p\beta_{m,p}| \max_{m\ne 1}|\kappa^s_{m}|^{p-1}
   +\cB \frac{p\alpha_p\beta_{1,p}}{M}(\kappa^s_{1})^{p-1}.
    \end{align}
    Also, $\frac{|\kappa_m^s|+\eta \CB d^\frac12}{2}\eta^2 \CB^2 d\leq \frac{|\kappa_m^s|+\cC M^{-\frac12}d^{-\frac{p}{2}} \CB d^\frac12}{2}\eta \cC M^{-\frac12}d^{-\frac{p-2}{2}} \CB^2 \leq \frac13\cA \eta p\alpha_p\beta_{m,p} \max\{|\kappa_m^s|,\frac12d^{-\frac12}\}^{p-1}$.
    Therefore, $\kappa_m^{\tau+1}$ is bounded as
    \begin{align}
       | \kappa_m^{\tau+1}| &\leq \max\bigg\{\CB \eta,\bigg|\kappa_m^{\tau} + 
        (1+\frac13\cA)\frac{\eta}{\sqrt{M}}p\max_{m}|\alpha_p\beta_{m,p}| \max_{m\ne 1}|\kappa^s_{m}|^{p-1}
        +
        \cB\frac{\eta}{M^\frac32}p\alpha_p\beta_{1,p} (\kappa^\tau_{1})^{p-1} 
        \\ & \quad +\frac13\cA \eta p\alpha_p\beta_{1,p} \max\{|\kappa_m^s|,\frac12d^{-\frac12}\}^{p-1}+\eta v_m^\top (I-w^s{w^s}^\top) Z^\tau\bigg|\bigg\}
        \\ & \leq \max\bigg\{\CB \eta,\bigg|\kappa_m^{0}
        +(1+\frac23\cA)\sum_{s=0}^\tau \frac{\eta}{\sqrt{M}}p\max_{m}|\alpha_p\beta_{m,p}|\max\{\max_{m\ne 1}|\kappa^s_{m}|^{p-1},(\frac12 d^{-\frac12})^{p-1}\}
    \\ &\quad +\cB\sum_{s=0}^\tau
       \frac{\eta}{M^\frac32}p\alpha_p\beta_{1,p} (\kappa^s_{1})^{p-1} +\sum_{s=0}^\tau\eta v_m^\top (I-w^s{w^s}^\top) Z^s\bigg|\bigg\}
   \\ &\leq \max\bigg\{\CB \eta, \max_{m\ne 1}|\kappa_m^{0}|
        +(1+\frac23\cA)\sum_{s=0}^\tau \frac{\eta}{\sqrt{M}}p\max_{m}|\alpha_p\beta_{m,p}| \max\{\max_{m\ne 1}|\kappa^s_{m}|^{p-1},(\frac12 d^{-\frac12})^{p-1}\}
         \\ & \quad +\cB\sum_{s=0}^\tau
       \frac{\eta}{M^\frac32}p\alpha_p\beta_{1,p} (\kappa^s_{1})^{p-1}
      + \max_{m\ne 1}\bigg|\sum_{s=0}^\tau\eta v_m^\top (I-w^s{w^s}^\top) Z^s\bigg|\bigg\}.
        \label{eq:OneStep-104}
    \end{align}

    According to the update of $Q^t$, $Q^t\geq \frac12 d^{-\frac12}\geq \CB \eta$ for all $t$.
    Moreover, by $\eta \leq \cC M^{-\frac12}d^{-\frac{p}{2}}$, 
    \begin{align}
      &   \max_{m\ne 1}\left|\sum_{s=0}^\tau\eta v_m^\top (I-w^s{w^s}^\top) Z^s\right|
    \\  &   \leq \eta \CB \sqrt{\tau}
      \\ & \leq  \begin{cases}\displaystyle
         \frac12\cA d^{-\frac12} & (\tau \leq \CA Md^{p-1})
         \\ \displaystyle
         \frac{\tau\cA \eta}{3\sqrt{M}}p\alpha_p\beta_{1,p}(\frac12 d^{-\frac12})^{p-1} &(\tau > \CA Md^{p-1})
      \end{cases}
    \end{align}
    Therefore, \eqref{eq:OneStep-104} is further upper bounded as
    \begin{align}
     &   \max_{m\ne 1}  | \kappa_m^{\tau+1}| \leq (1+\cA)\max\{ \max_{m\ne 1}|\kappa_m^{0}|,\frac12 d^{-\frac12}\}
         \\ & \quad +(1+\cA)\sum_{s=0}^\tau \frac{\eta}{\sqrt{M}}p\max_{m}|\alpha_p\beta_{m,p}| \max\{\max_{m\ne 1}|\kappa^s_{m}|,Q^{s}\}^{p-1}
        +\cB\sum_{s=0}^\tau
       \frac{\eta}{M^\frac32}p\alpha_p\beta_{1,p} (\kappa^s_{1})^{p-1}
   \label{eq:OneStep-105}
    \end{align}
    On the other hand, $Q^{\tau+1}$ is written as
    \begin{align}
        Q^{\tau+1} &\leq (1+\cA)\max\{ \max_{m\ne 1}|\kappa_m^{0}|,\frac12 d^{-\frac12}\}
        +(1+\cA)\sum_{s=0}^\tau \frac{\eta}{\sqrt{M}}p\max_{m}|\alpha_p\beta_{m,p}| (Q^{s})^{p-1}
    \\ & \quad   +\cB\sum_{s=0}^\tau
       \frac{\eta}{M^\frac32}p\alpha_p\beta_{1,p} (\kappa^s_{1})^{p-1}
   \label{eq:OneStep-106}
    \end{align}
    Comparing \eqref{eq:OneStep-105} and \eqref{eq:OneStep-106} with the update of $Q^t$, we conclude that $\max_{m\ne 1}  | \kappa_m^{\tau}|\leq Q^{\tau}\ (\tau=0,1,\cdots,t+1)$ holds by induction.
\end{proof}

The previous lemma assumed (a)-(c). Next we show that (b) and (c) hold along the trajectory via induction.
    (here we use a different notation for the coefficient in (c).)
\begin{lemma}\label{lemma:VerifyBC}
    Take $\eta^t=\eta \leq \cC M^{-\frac12}d^{-\frac{p}{2}}$.
    Suppose that, for all $s=0,1,\cdots,t$, 
    \begin{itemize}
        \item[(a)] $\kappa^s_1\leq \cA$, 
        \item[(b)] $|\kappa_m^s|\leq \kappa_1^s$ for all $m=2,\cdots,M$, and
        \item[(c)'] $|\kappa_m^s| \leq 4\cB M^{-1}$ for all $m=2,\cdots,M$.
    \end{itemize}
    Then, if we have $M\leq \cC d^\frac12$ and (a) $\kappa^{t+1}_1\leq \cA$ for $t+1$, (b) and (c)' hold for $s=t+1$ with high probability.
    % Note that (a)-(c) trivially hold for $s=0$.
    % Thus, with high probability, there exists some $t$ such that $\kappa^t_1 > \cA$ while $\kappa_m^s \leq \cC M^{-\frac12}$ for all $m=2,\cdots,M$.
\end{lemma}
\begin{proof}
    First consider the case when
     \begin{align}
    (1+\cA)\frac{p\max_m|\alpha_p\beta_{m,p}|}{\sqrt{M}} (Q^{s})^{p-1}>\frac{\cB}{M^\frac32}p\alpha_p\beta_{1,p}(P^{s})^{p-1}
    \label{eq:Criteria-101}
   \end{align}
   holds for all $s=0,1,\cdots,t$. 
   Then, for $s=0,1,\cdots,t$,  
   \begin{align}
      \frac{\cB}{M^\frac32}p\alpha_p\beta_{1,p}(P^s)^{p-1}
      < \cA^{-1}\cB(1+\cA)\frac{p\max_m|\alpha_p\beta_{m,p}|}{\sqrt{M}} (Q^{s})^{p-1}
      <\cA\frac{p\max_m|\alpha_p\beta_{m,p}|}{\sqrt{M}} (Q^{t})^{p-1},
   \end{align}
   and therefore
   \begin{align}
       Q^{s+1}\leq Q^{s}+(1+2\cA)\frac{\eta}{\sqrt{M}}p\max_{m}|\alpha_p\beta_{m,i}| (Q^{s})^{p-1}
   \end{align}
   for all $s=0,1,\cdots,t$.
   According to Lemma~\ref{lemma:Bihari–LaSalle}, we have
   \begin{align}
       Q^{s}\!\leq\! 
       \frac{Q^0}{\left(1-\eta p(p-2)(1+3\cA)(\max_m|\alpha_p\beta_{m,p}|)M^{-\frac12}(Q^0)^{(p-2)}s\right)^{\frac{1}{p-2}}}
       \label{eq:Criteria-102}
   \end{align}
   for all $s=0,1,\cdots,t+1$ 
    (here $(1+c)^{p-1}$ in the original bound is absorbed in $(1+3\cA)$).

   On the other hand, $P^s$ is lower bounded by
   \begin{align}
       P^{s}\geq
       \frac{P^0}{\left(1-\eta p(p-2)(1-\cA)(\alpha_p\beta_{1,p})M^{-\frac12}(P^0)^{(p-2)}s\right)^{\frac{1}{p-2}}}
       \label{eq:Criteria-103}
   \end{align}
   for all $s=0,1,\cdots,t+1$.
   According to \eqref{eq:smalldiversity}, 
   $Q^0=(1+\cA)\max\{\max_{m\ne 1}|\kappa_m^0|,\frac12d^{-\frac12}\}\leq P^0=(1-\cA)\kappa_1^0$.
   Moreover, \eqref{eq:smalldiversity} implies that
   $\max_{m}|\alpha_p\beta_{m,p}|(1+3\cA)(\max_{m\ne 1}|\kappa_m^0|)^{p-2}\leq (1-\cA)\alpha_p\beta_{1,p}(\kappa_1^0)^{p-2}$, yielding that
   the denominator of \eqref{eq:Criteria-102} is larger than that of \eqref{eq:Criteria-103}.
   Therefore, $Q^{t+1}\leq P^{t+1}$ holds, which implies that $\max_{m\ne 1}|\kappa_m^{t+1}|\leq \kappa_1^{t+1}$.

   Next we check  $Q^{t+1}< \cB M^{-1}$ (RHS is smaller than (c)' by a factor of $4$). From $P^t\leq \kappa_1^t\leq \cA$ and \eqref{eq:Criteria-103},
   \begin{align}
      t\leq \frac{\eta^{-1}M^{\frac12}((P^0)^{-(p-2)}-(\cA)^{-(p-2)})}{p(p-2)(1-\cA)(\alpha_p\beta_{1,p})}
      \leq \frac{\eta^{-1}M^{\frac12}(1+5\cA)\frac{\max_m|\alpha_p\beta_{m,p}|}{\alpha_p\beta_{1,p}}(P^0)^{-(p-2)}}{p(p-2)(1+3\cA)(\max_m|\alpha_p\beta_{m,p}|)}
      .\label{eq:Criteria-104}
   \end{align}
    On the other hand, according to \eqref{eq:Criteria-102}, $Q^{t+1}> \cB M^{-1}$ holds only if
    \begin{align}
        t> \frac{\eta^{-1}M^{\frac12}((Q^0)^{-(p-2)}-(\cB M^{-1})^{-(p-2)}-\eta M^{-\frac12} p(p-2)(1+3\cA)(\max_m|\alpha_p\beta_{m,p}|))}{p(p-2)(1+3\cA)(\max_m|\alpha_p\beta_{m,p}|)}.
        \label{eq:Criteria-105}
    \end{align}
    If $M\leq \cC d^\frac12$, $(\cB M^{-1})^{-(p-2)}\leq (\cB\cC^{-1} d^{-\frac12})^{-(p-2)} \leq \cA(\kappa_1^0)^{-(p-2)}\leq \cA(P^0)^{-(p-2)}\leq \cA(Q^0)^{-(p-2)}$.
    Moreover, $\eta M^{-\frac12}p (p-2)(1+5\cA)(\max_m|\alpha_p\beta_{m,p}|)\leq \cA(Q^0)^{-(p-2)}$.
    Thus, \eqref{eq:Criteria-105} is further bounded by
    \begin{align}
        \eqref{eq:Criteria-105} \geq \frac{\eta^{-1}M^{\frac12}(1-2\cA)(Q^0)^{-(p-2)}}{p(p-2)(1+3\cA)(\max_m|\alpha_p\beta_{m,p}|)}
        \label{eq:Criteria-106}.
    \end{align}
   According to \eqref{eq:smalldiversity}, $(1+8\cA)\max_{m}|\alpha_p\beta_{m,p}|(\max_{m\ne 1}|\kappa_m^0|)^{p-2}< \alpha_p\beta_{1,p}(\kappa_1^0)^{p-2}$.
   By using this, we have $\text{(RHS of \eqref{eq:Criteria-106})}>\text{(RHS of \eqref{eq:Criteria-104})}$.
   Thus, $Q^{t+1}> \cB M^{-1}$ does not hold, and we have obtained $Q^{t+1}\leq  \cB M^{-1}$ (and (c)').

   Now we consider the case when \eqref{eq:Criteria-101} holds for $s=0,1,\cdots,\tau_1-1$ but 
\begin{align}
    (1+\cA)\frac{p\max_m|\alpha_p\beta_{m,p}|}{\sqrt{M}} (Q^{s})^{p-1}\leq \frac{\cB}{M^\frac32}p\alpha_p\beta_{1,p}(P^{s})^{p-1}
    \label{eq:Criteria-107}
   \end{align}
   holds for $s=\tau_1\leq t$.
   We show that \eqref{eq:Criteria-107} holds for all $s=\tau_1,\cdots,t+1$ in this case.
   Assuming the inequality holds for all $s=\tau_1,\cdots,\tau$ with $\tau\leq t$.
   For $s=\tau_1,\cdots,\tau$, the update of $Q^s$ is evaluated as
   \begin{align}
       Q^{s+1}\leq Q^{s} + 2\cB\frac{\eta}{M^\frac32}p\alpha_p\beta_{1,p}(P^{s})^{p-1}.
   \end{align}
   Thus, when $p\geq 3$, 
   \begin{align}
       Q^{\tau+1}&\leq Q^{\tau_1} + \sum_{s=\tau_1}^{\tau}2\cB\frac{\eta}{M^\frac32}p\alpha_p\beta_{1,p}(P^{s})^{p-1}
       \\ &\leq Q^{\tau_1} + \frac{2\cB(1-\cA)^{-1}}{M}\left(P^{\tau+1}-P^{\tau_1}\right)
       \label{eq:Criteria-108}
       \\ & \leq 
       \left(\frac{\alpha_p\beta_{1,p}}{\max_m|\alpha_p\beta_{m,p}|}(1+\cA)^{-1}M^{-1}\right)^{\frac{1}{p-1}}P^{\tau_1} + \frac{2\cB(1-\cA)^{-1}}{M}\left(P^{\tau+1}-P^{\tau_1}\right)
       \\ & \leq \left(\frac{\alpha_p\beta_{1,p}}{\max_m|\alpha_p\beta_{m,p}|}(1+\cA)^{-1}M^{-1}\right)^{\frac{1}{p-1}}P^{\tau+1},
   \end{align}
   which yields \eqref{eq:Criteria-107} for $s=\tau+1$.
   \eqref{eq:Criteria-107} with $s=t+1$ implies that $Q^{t+1}\leq P^{t+1}$, and hence $\max_{m\ne 1}|\kappa_m^{t+1}|\leq \kappa_1^{t+1}$.
   Similar to \eqref{eq:Criteria-108}, we have
   \begin{align}
       Q^{t+1}\leq Q^{\tau_1} + \frac{2\cB(1-\cA)^{-1}}{M}\left(P^{t+1}-P^{\tau_1}\right)
       \leq Q^{\tau_1} + \frac{2\cB(1-\cA)^{-1}}{M}P^{t+1}
      \leq Q^{\tau_1} + 3\cB M^{-1}
       .
   \end{align}
   For $Q^{\tau_1}$, as we proved $Q^{t+1}\leq \cB M^{-1}$ in the first case, we have $Q^{\tau_1}\leq \cB M^{-1}$.
   % Also, $3\cA \cB M^{-1}\leq \frac14 \cB M^{-1}$
   Thus, $  Q^{t+1}$ is bounded by $4\cB M^{-1}$, which yields (c)'.
\end{proof}

\begin{proofof}[Lemma~\ref{lemma:AlignmentFirstPhase}]
    Suppose that (a) holds for all $s=0,1,\cdots, T_{1,1} $, where 
    \begin{align}
   T_{1,1}= \left\lfloor\left(\eta p(p-2)(1-5\cA)(\alpha_p\beta_{1,p})M^{-\frac12}(P^0)^{(p-2)}\right)^{-1}\right\rfloor.
    \label{eq:Criteria-111}
    \end{align}
    According to Lemma~\ref{lemma:VerifyBC}, 
    if $M\leq \cC d^\frac12$, $\eta \leq \cC M^{-\frac12}d^{-\frac{p}{2}}$, and (a) holds for all $s=0,1,\cdots,T_{1,1}$, 
   (b) and (c) of Lemma~\ref{lemma:VerifyBC} holds with high probability for all $s=0,1,\cdots,T_{1,1}$ and the bounds of Lemma~\ref{lemma:Lower} holds for all $s=0,1,\cdots,T_{1,1}$.

   From Lemma~\ref{lemma:Lower} and Lemma~\ref{lemma:Bihari–LaSalle}, 
   \begin{align}
       \kappa_1^t \geq P^t \geq \frac{P^0}{\left(1-\eta p(p-2)(1-\cA)(\alpha_p\beta_{1,p})M^{-\frac12}(P^0)^{(p-2)}s\right)^{\frac{1}{p-2}}}
      .\label{eq:Criteria-109}
   \end{align}
   However, at $t =T_{1,1}$, 
   \begin{align}
       \text{(RHS of \eqref{eq:Criteria-109})}
      & \geq \frac{P^0}{\left(\eta p(p-2)(1-\cA)(\alpha_p\beta_{1,p})M^{-\frac12}(P^0)^{(p-2)}\right)^{\frac{1}{p-2}}}
      \\ &  = \frac{1}{\left(\eta p(p-2)(1-\cA)(\alpha_p\beta_{1,p})M^{-\frac12}\right)^{\frac{1}{p-2}}},
   \end{align}
   and thus RHS of \eqref{eq:Criteria-109} is clearly larger than $1$.
   This yields the contradiction because $\kappa_1^{T_{1,1}}\leq 1$. Therefore, with high probability, there exists some $t_1\leq T_{1,1}$ such that
   $\kappa_1^{t_1}> \cB$ and $\kappa_1^{t}\leq  \cB\ (t=0,1,\cdots,t_1-1)$.

   As for (ii), recall that $|\kappa_m^{t_1-1}|\leq 4\cB M^{-\frac12}$.
   Moreover, according to Lemma~\ref{lemma:OneStep}, 
   \begin{align}
   &|\kappa_1^{t_1}-\kappa_1^{t_1-1}|
     \\ &  \leq \CB \eta \leq \CB \cC M^{-\frac12} d^{-\frac{p}{2}}\leq \cB M^{-\frac12}.
     % \left|  \frac{\eta}{\sqrt{M}}\sum_{m'=1}^M
     %    \sum_{i=p}^q
     %    \left[i\alpha_i\beta_{m',i} (\kappa^t_{m'})^{i-1} (v_m^\top v_{m'} - \kappa_m^t\kappa_{m'}^t)\right]
     %    +
     %    \frac{|\kappa_m^t|+\eta \CB d^\frac12}{2}\eta^2 \CB^2 d + \eta v_m^\top (I-w^t{w^t}^\top) Z^t\right|
     %    \\ & \leq 2\eta M^\frac12 q^2 \max_{m,i}|\alpha_i\beta_{m,i}|
     %    + (\CB \eta d^\frac12)^2 + (\CB \eta d^\frac12)^3 + \CB\eta.\label{eq:Criteria-110}
   \end{align}
   % By taking $\eta \leq \cC M^{-\frac12} d^{-\frac{p}{2}}$, \eqref{eq:Criteria-110} is bounded by $\frac12\cB M^{-\frac12}$.
   Thus, $|\kappa_m^{t_1}|\leq |\kappa_m^{t_1-1}|+|\kappa_1^{t_1}-\kappa_1^{t_1-1}|\leq 4\cB M^{-\frac12}+\cB M^{-\frac12}\leq 5\cB M^{-\frac12}$.
\end{proofof}

\subsection{Phase II: Amplification of Alignment}\label{subsection:alignment-2}
In the previous section (Lemma~\ref{lemma:AlignmentFirstPhase}), we proved that 
neurons in $J_1$ achieve a small constant ($ \cA$) alignment.
However, as the alignment $\kappa_1^t$ becomes larger, the effect from other directions becomes non-negligible, while the signal from $v_1$ gets smaller because of the projection $(1-w^t{w^t}^\top)$. 
Hence the previous weak recovery analysis is not sufficient to show that the neurons will continue to grow in the direction of $v_1$.

The goal of this subsection is to prove the following lemma.
\begin{lemma}\label{lemma:AlignmentSecondPhase}
    Suppose $\eta=\eta^t \leq \cC M^{-\frac12}d^{-\frac{p}{2}}$, $|v_{m'}^\top v_m|\leq \cC M^{-1}$ for all $m\ne m'$, and $M\leq \cC d^{\frac12}$, and consider a neuron that satisfies (i) and (ii) of Lemma~\ref{lemma:AlignmentFirstPhase}. 
    Then, with high probability, there exists some time $t_2\leq T_{1,2}=\tilde{\Theta}(M^\frac12\eta^{-1})$ such that $\kappa_1^{t_2}> 1-\cA$.
\end{lemma}
Similar to the Phase I analysis, we bound the update by deterministic auxiliary sequences. To simplify the notation, we let $t \leftarrow t-t_1$ throughout this subsection.
\begin{lemma}\label{lemma:P2-Lower}
    Suppose that $|v_{m'}^\top v_m|\leq \cC M^{-1}$ for all $m\ne m'$ and $M\leq \cC d^{\frac12}$.
    For all $s=0,1,\cdots,t$, we assume that
    \begin{itemize}
        \item[(a)] $\kappa^s_1\leq 1-\cA$ (only required for (i): Lower bound), 
        \item[(b)] $|\kappa_m^s|\leq \kappa_1^s$ for all $m=2,\cdots,M$, and
        \item[(c)] $|\kappa_m^s| \leq \CA\cB M^{-1}$ for all $m=2,\cdots,M$.
    \end{itemize}
    Take $\eta^t=\eta \leq  \cC M^{-\frac12}d^{-\frac{p}{2}}$. 
     Then, we have the following bounds.

     \begin{itemize}
   \item[{\bf(i):}]{\bf Lower bound of $\kappa^s_1$:} For an auxiliary sequence $(P^s)_{s=0}^{t+1}$ with
     $P^0 = (1-\cA)\kappa_1^0$, and
    \begin{align}\label{eq:P2-Lower-1}
        P^{s+1} = P^{s}
       +\cA\frac{\eta}{\sqrt{M}}p\alpha_p\beta_{m,p} (P^{s})^{p-1},
    \end{align}
     $\kappa_1^s$ is lower-bounded by $P^s$ for all $s=0,1,\cdots,t+1$, with high probability.
   
   \item[{\bf (ii):}]{\bf Upper bound of  $\max_{m\ne 1}|\kappa_m^s|$:} For an auxiliary sequence $(Q^s)_{s=0}^{t+1}$ with
 $Q^0 = 6\cB M^{-1}$, and
    \begin{align}\label{P2-Lower-2}
        Q^{s+1} = Q^{s}
       +(1+\cA)\frac{\eta}{\sqrt{M}}p\max_{m}|\alpha_p\beta_{m,i}| (Q^{s})^{p-1}
       +\cB
       \frac{\eta}{M}p\alpha_p\beta_{1,p} (\kappa^s_{1})^{p-1},
         % Q^{\tau+1}=\max\{\max_{m\ne 1}  | \kappa_m^{0}|, \cA \kappa_1^0\}+\sum_{s=0}^\tau \frac{\eta}{\sqrt{M}}p\alpha_p\beta_{m,p} (Q^s)^{p-1} + 3\cA\sum_{s=0}^\tau
       % \frac{\eta}{\sqrt{M}}p\alpha_p\beta_{1,p} (\kappa^s_{1})^{p-1}
    \end{align}
     $\max_{m\ne 1}|\kappa_m^s|$ is upper-bounded by $Q^s$ for all $s=0,1,\cdots,t+1$, with high probability.
    \end{itemize}
\end{lemma}
\begin{proof}
    \paragraph{(i) Lower bound of $\kappa^{s}_1$ by $P^s$:}
    % We assume that $\kappa^s_1 < \cA$.
    If $\kappa_1^s\geq \frac12 d^{-\frac12}$, by following the argument for \eqref{eq:ConnectODE-201} in  Lemma~\ref{lemma:Lower}, we have
    \begin{align}
         \kappa^{s+1}_1 &\geq 
        \kappa_1^s 
        + \frac{\eta}{\sqrt{M}}p\alpha_p\beta_{1,p}(1-(\kappa_1^s)^2)(\kappa_1^s)^{p-1}
    -q^2\eta\sqrt{M}\max_{m\ne 1}|\kappa_m^s|^{p-1}\max_{m\ne 1}|v_1^\top v_m|\max_{m,i}|\alpha_i\beta_{m,i} |
       \\ & \quad -q^2\eta\sqrt{M}\max_{m\ne 1}|\kappa_m^s|^{p}\max_{m,i}|\alpha_i\beta_{m,i}| -
        \kappa_1^s\eta^2 \CB^2 d + \eta v_1^\top (I-w^s{w^s}^\top)  Z^s 
        \\ & \geq \kappa_1^s 
        + \frac{\eta}{\sqrt{M}}p\alpha_p\beta_{1,p}(1-(1-\cA)^2)(\kappa_1^s)^{p-1}
    -q^2\eta\sqrt{M}|\kappa_1^s|^{p-1}\cC M^{-1}
       \\ & \quad -q^2\eta\sqrt{M}(\kappa_1^s)^{p-1}\CA\cB M^{-1} -
        \eta\kappa_1^s\cC M^{-\frac12}d^{-\frac{p-2}{2}} \CB^2 + \eta v_1^\top (I-w^s{w^s}^\top)  Z^s 
        . 
    \end{align}
    Note that $1-(1-\kappa_1^s)^2\geq \frac{9}{5}\cA$, $q^2\eta\sqrt{M}(\kappa_1^s)^{p-1}\cC M^{-1} \leq\frac15 \cA\frac{\eta}{\sqrt{M}}p\alpha_p\beta_{1,p}(\kappa_1^s)^{p-1}$, $q^2\eta\sqrt{M}(\kappa_1^s)^{p-1}\CA\cB M^{-1} \leq\frac15 \cA\frac{\eta}{\sqrt{M}}p\alpha_p\beta_{1,p}(\kappa_1^s)^{p-1}$ and $\eta\kappa_1^s\cC M^{-\frac12}d^{-\frac{p-2}{2}} \CB^2 \leq \frac15\cA\frac{\eta}{\sqrt{M}}p\alpha_p\beta_{1,p}(\kappa_1^s)^{p-1}$
    (where we used $\kappa_1^s\geq \frac12 d^{-\frac12}$ for the last statement).

    If $\kappa_1^s\geq \frac12 d^{-\frac12}$ for all $s=0,1,\cdots,\tau$, 
    following \eqref{eq:Noise-101} and \eqref{eq:Noise-102}, we know that
    \begin{align}
        \sum_{s=0}^{\tau} \eta v_1^\top (I-w^s{w^s}^\top)Z^{s} 
       \leq \cA \kappa_1^0
       +\frac15 \sum_{s=0}^{\tau}\cA\frac{\eta}{\sqrt{M}}p\alpha_p\beta_{1,p} (\kappa_1^s)^{p-1}
        ,
    \end{align}
    with high probability.
    
    Given $\kappa_1^s\geq \frac12 d^{-\frac12}$ for all $s=0,1,\cdots,\tau$, we obtain
    \begin{align}
        \kappa^{\tau+1}_1 &\geq (1-\cA)\kappa_1^0 
        + \cA \sum_{s=0}^{\tau}\frac{\eta}{\sqrt{M}}p\alpha_p\beta_{1,p}(\kappa_1^s)^{p-1},
        \label{P2-Lower-3}
    \end{align}
    and $\kappa_1^{\tau+1}\geq \frac12 d^{-\frac12}$.
    Therefore, $\kappa_1^\tau\geq \frac12 d^{-\frac12}$ holds for all $\tau=0,1,\cdots,t$, and \eqref{P2-Lower-3} holds for all $\tau=0,1,\cdots,t$
    By comparing \eqref{P2-Lower-3} with the update of $P^\tau$, we obtain the desired bound.
    
    \paragraph{\bf (ii) Upper bound of $\max_{m\ne 1}|\kappa^{s}_m|$ by $Q^s$:}
    We have already established the upper bound with $Q^0 = (1+\cA)\max\{ \max_{m\ne 1}|\kappa_m^{0}|,\frac12 d^{-\frac12}\}$ in Lemma~\ref{lemma:Lower}.
    Because the choice of $Q^0 = 6\cB M^{-1}$ is larger than that when $M\leq \cC d^{\frac12}$, the desired bound follows.
\end{proof}

Now we show that assumptions (b) and (c) in the previous lemma can be verified along the trajectory via induction.
\begin{lemma}\label{lemma:P2-VerifyBC}
    Take $\eta^t=\eta \leq \cC M^{-\frac12}d^{-\frac{p}{2}}$.
    Suppose that, for all $s=0,1,\cdots,t$, 
    \begin{itemize}
        \item[(a)] $\kappa^s_1\leq 1-\cA$, 
        \item[(b)] $|\kappa_m^s|\leq \kappa_1^s$ for all $m=2,\cdots,M$, and
        \item[(c)'] $|\kappa_m^s| \leq \CA\cB M^{-1}$ for all $m=2,\cdots,M$.
    \end{itemize}
    Then, if $M\leq \cC d^\frac12$ and (a) $\kappa^{t+1}_1\leq 1-\cA$ for $t+1$, (b) and (c)' hold for $s=t+1$ with high probability.
\end{lemma}
\begin{proof}
    We only need to prove (c)': $\max_{m\ne 1}|\kappa_m^{t+1}|\leq 7\cB M^{-1}$.
    This is because if (a) and (c)' for $t+1$ hold, then $\kappa_1^0>\cA$ and Lemma~\ref{lemma:P2-Lower} yields that $\kappa_1^{t+1}\geq P^{t+1} \geq P^0 = (1-\cA)\kappa_1^0\geq 2\cB M^{-1}\geq \max_{m\ne 1}|\kappa_m^{t+1}|$, which proves (b) for $t+1$. 
    If
       \begin{align}
    (1+\cA)\frac{p\max_m|\alpha_p\beta_{m,p}|}{\sqrt{M}} (Q^{s})^{p-1}\leq \frac{\cB}{M^\frac32}p\alpha_p\beta_{1,p}(P^{s})^{p-1}
    \label{eq:Criteria-217}
   \end{align}
   holds for all $s=0,1,\cdots,\tau$, according to Lemma~\ref{lemma:P2-Lower}, the update of $Q^s$ is evaluated as
   \begin{align}
       Q^{s+1}\leq Q^{s} + 2\cB\frac{\eta}{M^\frac32}p\alpha_p\beta_{1,p}(P^{s})^{p-1}
   \end{align}
   for all $s=0,1,\cdots,\tau$.
   Note that \eqref{eq:Criteria-217} holds for $s=0$.
   Thus, if \eqref{eq:Criteria-217} holds for all $s=0,1,\cdots,\tau$, we have,
   \begin{align}
       Q^{\tau+1}&\leq Q^{\tau_1} + \sum_{s=\tau_1}^{\tau}2\cB\frac{\eta}{M^\frac32}p\alpha_p\beta_{1,p}(P^{s})^{p-1}
       \\ &\leq Q^{\tau_1} + \frac{2\cA^{-1}\cB}{M}\left(P^{\tau+1}-P^{\tau_1}\right)
       \label{eq:Criteria-208}
       \\ & \leq 
       \left(\frac{\alpha_p\beta_{1,p}}{\max_m|\alpha_p\beta_{m,p}|}(1+\cA)^{-1}M^{-1}\right)^{\frac{1}{p-1}}P^{\tau_1} + \frac{2\cA^{-1}\cB}{M}\left(P^{\tau+1}-P^{\tau_1}\right)
       \\ & \leq \left(\frac{\alpha_p\beta_{1,p}}{\max_m|\alpha_p\beta_{m,p}|}(1+\cA)^{-1}M^{-1}\right)^{\frac{1}{p-1}}P^{\tau+1},
   \end{align}
   which implies \eqref{eq:Criteria-217} for $s=\tau + 1$.
   Thus, \eqref{eq:Criteria-217} holds for all $s=0,1,\cdots,t+1$.
   
   Moreover, similar to \eqref{eq:Criteria-208},
   \begin{align}
       Q^{t+1}\leq Q^{0} + \frac{2\cA^{-1}\cB}{M}\left(P^{t+1}-P^{0}\right)
       \leq Q^{0} + \frac{2\cA^{-1}\cB}{M}P^{t+1}
      \leq Q^{0} + 2\CA \cB M^{-1}
       .
   \end{align} 
   $Q^{0}$ is bounded by $6\cB M^{-1}$
   % Also, $3\cA \cB M^{-1}\leq \frac14 \cB M^{-1}$
   Thus, $  Q^{t+1}$ is bounded by $\CA\cB M^{-1}$, which yields (c).
\end{proof}
\begin{proofof}[Lemma~\ref{lemma:AlignmentSecondPhase}]
    Suppose that (a) holds for all $s=0,1,\cdots, T_{1,2} $, where 
    \begin{align}
    T_{1,2}= \left\lfloor\left(\eta p(p-2)\cA(\alpha_p\beta_{1,p})M^{-\frac12}(P^0)^{(p-2)}\right)^{-1}\right\rfloor.
    \label{eq:Criteria-211}
    \end{align}
    According to Lemma~\ref{lemma:P2-VerifyBC}, 
    if $M\leq \cC d^\frac12$, $\eta \leq \cC M^{-\frac12}d^{-\frac{p}{2}}$, and (a) for all $s=0,1,\cdots,T_{1,2}2$ hold, 
   (b) and (c) of Lemma~\ref{lemma:P2-VerifyBC} holds with high probability for all $s=0,1,\cdots,T_{1,2}$ and the bounds of Lemma~\ref{lemma:P2-Lower} holds for all $s=0,1,\cdots,T_{1,2}$.

   According to Lemma~\ref{lemma:P2-Lower} and Lemma~\ref{lemma:Bihari–LaSalle}, 
   \begin{align}
       \kappa_1^t \geq P^t \geq \frac{P^0}{\left(1-\eta p(p-2)\cA(\alpha_p\beta_{1,p})M^{-\frac12}(P^0)^{(p-2)}s\right)^{\frac{1}{p-2}}}
      .\label{eq:Criteria-209}
   \end{align}
   However, at $t =T_{1,2}$, 
   \begin{align}
       \text{(RHS of \eqref{eq:Criteria-209})}
       \geq \frac{P^0}{\left(\eta p(p-2)\cA(\alpha_p\beta_{1,p})M^{-\frac12}(P^0)^{(p-2)}\right)^{\frac{1}{p-2}}}
       = \frac{1}{\left(\eta (p-2)\cA(\alpha_p\beta_{1,p})M^{-\frac12}\right)^{\frac{1}{p-2}}},
   \end{align}
   and thus RHS of \eqref{eq:Criteria-209} is clearly larger than $1$.
   This yields the contradiction because $\kappa_1^{T_{1,2}}$ should be smaller than $1$. Therefore, with high probability, there exists some $t_2\leq T_{1,2}$ such that
   $\kappa_1^{t_2}> 1-\cA$.
\end{proofof}

\subsection{Phase III: Strong Recovery and Localization}\label{subsection:alignment-3}
In the previous section (Lemma~\ref{lemma:AlignmentSecondPhase}), we proved that neurons can achieve alignment of $1-\cA$, which sets up the local convergence argument.
To simplify the notation, we let $t \leftarrow t-t_1-t_2$ throughout this subsection.
We write $\bar{v}_m=(I-v_1v_1^\top) v_m$ and
$\bar{\kappa}^t_m = \bar{v}_m^\top w_j^t$.

The goal of this subsection is to prove the following lemma.
% Putting it all together, we finally obtain Lemma~\ref{lemma:Goal-1}.
\begin{lemma}\label{lemma:AlignmentFinalPhase}
    Take $\eta^t = \eta_1 \leq \cC M^{-\frac12}d^{-\frac{p}{2}}$
    for $0\leq t\leq (T_{1,1}-t_1)+(T_{1,2}-t_2)-1$, and $\eta^t = \eta_2 \leq \min\{\cC \tilde{\varepsilon}M^{-\frac12} d^{-1},\cC \tilde{\varepsilon}^2M^{-\frac12}\}$
    for $(T_{1,1}-t_1)+(T_{1,2}-t_2)\leq t\leq T_{1,3}+(T_{1,1}-t_1)+(T_{1,2}-t_2)-1$.
    % $\eta^t = \eta \leq \min\{\cC M^{-\frac12}d^{-\frac{p}{2}},\cC \tilde{\varepsilon}M^{-\frac12} d^{-1},\cC \tilde{\varepsilon}^2M^{-\frac12}\}$
    Suppose $|v_{m'}^\top v_m|\leq \cC M^{-1}$ for all $m'\ne m$ and $T_{1,3}=\tilde{\Theta}(\tilde{\varepsilon}^{-1}M^\frac12\eta^{-1})$, 
    and consider a neuron that satisfies $\kappa_1^0\geq 1-\cA$. 
    
    Then, $\kappa_1^{(T_{1,1}-t_1)+(T_{1,2}-t_2)+T_{1,3}}> 1-3\tilde{\varepsilon}$ with high probability.
\end{lemma}
We bound the update by deterministic auxiliary sequences.
\begin{lemma}\label{lemma:P3-Lower}
    Let $0<\bar{\varepsilon}<\cA$. 
    If $\eta^t=\eta \leq\min\{ \cC \bar{\varepsilon}M^{-\frac12} d^{-1},\cC \bar{\varepsilon}^2M^{-\frac12}\}$, $1-2\cA \leq \kappa_1^0$, $\kappa_1^s\leq 1-\bar{\varepsilon}$ for $s=0,1,\cdots,t$, and $|v_{m'}^\top v_m|\leq \cC M^{-1}$ for all $m'\ne m$, we have the following bound:
     \begin{itemize}[leftmargin=*]
   \item{\bf Lower bound of $\kappa^s_1$:} 
   % For an auxiliary sequence $(P^s)_{s=0}^{t+1}$ with
     % $P^0 = (1-2\cA)$, and
    \begin{align}\label{eq:P2-Lower-1}
    \kappa^s_1 \geq \kappa^0_1 -\cA\bar{\varepsilon} + s\bar{\varepsilon}\frac{\eta}{\sqrt{M}}p\alpha_p\beta_{m,p}, 
       %  P^{s+1} = P^{s}
       % +\tilde{\varepsilon}\frac{\eta}{\sqrt{M}}p\alpha_p\beta_{m,p},
    \end{align}
     % $\kappa_1^s$ is lower-bounded by $P^s$ for all $s=0,1,\cdots,t+1$,
    for all $s=0,1,\cdots,t+1$, with high probability.
     \end{itemize}
\end{lemma}
\begin{proof}
    If $\kappa_1^s\geq \frac12 d^{-\frac12}$, 
    by Lemma~\ref{lemma:Lower}, 
    \begin{align}
         \kappa^{s+1}_1 &\geq 
        \kappa_1^s 
        +\frac{\eta}{\sqrt{M}}\sum_{m=1}^M
        \sum_{i=p}^q
        \left[i\alpha_i\beta_{m,i} (\kappa^s_m)^{i-1} (v_1^\top v_m - \kappa_1^s\kappa_m^s)\right]
        -
        \kappa_1^s\eta^2 \CB^2 d + \eta v_1^\top (I-w^s{w^s}^\top)  Z^s
        \\ & \geq 
        \kappa_1^s 
        + \frac{\eta}{\sqrt{M}}p\alpha_p\beta_{1,p}(1-(\kappa_1^s)^2)(\kappa_1^s)^{p-1}
        +\frac{\eta}{\sqrt{M}}\sum_{m=2}^M
        \sum_{i=p}^q
        \left[i\alpha_i\beta_{m,i} (\kappa^s_m)^{i-1} (v_1^\top v_m - \kappa_1^s\kappa_m^s)\right]
        \\ & \quad -
        \kappa_1^s\eta^2 \CB^2 d + \eta v_1^\top (I-w^s{w^s}^\top)  Z^s
            \\ & \geq 
        \kappa_1^s 
        + \frac{\eta}{\sqrt{M}}p\alpha_p\beta_{1,p}(1-(\kappa_1^s)^2)(\kappa_1^s)^{p-1}
        +\frac{\eta}{\sqrt{M}}\sum_{m=2}^M
        \sum_{i=p}^q
        \left[i\alpha_i\beta_{m,i} (\kappa^s_m)^{i-1} v_1^\top v_m(1 - (\kappa_1^s)^2)\right]
        \\ & \quad +\frac{\eta}{\sqrt{M}}\sum_{m=2}^M
        \sum_{i=p}^q
        \left[i\alpha_i\beta_{m,i} (\kappa^s_m)^{i-1}
        \bar{\kappa}_m^s\kappa_1^s\right]-
        \kappa_1^s\eta^2 \CB^2 d + \eta v_1^\top (I-w^s{w^s}^\top)  Z^s.
        \label{eq:P3-A1-1}
    \end{align}
    We bound each term of \eqref{eq:P3-A1-1} from now.
    If $\kappa_1^s \geq 1-3\cA$, the second term is bounded by  
    \begin{align}
       p\alpha_p\beta_{1,p} (1-(\kappa_1^s)^2)(\kappa_1^s)^{p-1}\geq p\alpha_p\beta_{1,p}(1-(\kappa_1^s)^2)(\kappa_1^s)^{p-1}
        \geq \frac{9}{5}p\alpha_p\beta_{1,p}(1-\kappa_1^s).
        \label{eq:P3-A1-2}
    \end{align}
    Next, the third term is bounded by 
    \begin{align}
        \bigg|\sum_{m=2}^M
        \sum_{i=p}^q
        \left[i\alpha_i\beta_{m,i} (\kappa^s_m)^{i-1} v_1^\top v_m(1 - (\kappa_1^s)^2)\right]\bigg|
        &\leq q^2M\max_{m,i}|\alpha_i\beta_{m,i}|\max_{m\ne 1}|v_1^\top v_m|(1-(\kappa_1^s)^2)
    \\ & \leq 2q^2M\max_{m,i}|\alpha_i\beta_{m,i}|\cC M^{-1}(1-(\kappa_1^s)^2)
    \\ & \leq \frac{1}{5}p\alpha_p\beta_{1,p}(1-\kappa_1^s).
    \label{eq:P3-A1-3}
    \end{align}
    Then we consider the fourth term, 
     \begin{align}
       & \bigg|\sum_{m=2}^M
        \sum_{i=p}^q
        \left[i\alpha_i\beta_{m,i} (\kappa^s_m)^{i-1}
        \tilde{\kappa}_m^s\kappa_1^s\right]\bigg|
     \\  &\leq  q^2\max_{i,m}|\alpha_i\beta_{m,i}| \sum_{m=2}^M|\kappa_m^s||\bar{\kappa}^s_m|
        \\ & \leq 2q^2\max_{i,m}|\alpha_i\beta_{m,i}|\sum_{m=2}^M |\bar{\kappa}^s_m|^3 + 2q^2\max_{i,m}|\alpha_i\beta_{m,i}|\sum_{m=2}^M|v_1^\top v_m|^2|\bar{\kappa}^s_m|.
        \label{eq:P3-A1-4}
    \end{align}
    To upper bound the above, we consider the value of $\sum_{m=2}^M |\bar{\kappa}^s_m|^2$. 
    This can be represented as $ \sum_{m=2}^M |\bar{\kappa}^s_m|^2={w^s}^\top (I-v_1v_1^\top)^\top  A (I-v_1v_1^\top)w^s$, where
    $A=\sum_{m=2}^M v_mv_m^\top.$

    Consider $\tilde{v}_m$ defined in Lemma~\ref{lemma:NearOrthogonalBasis} with coefficients $\{c_{m,m'}\}$.
    Let us define $B\in \mathbb{R}^{M\times M}$ as 
    \begin{align}
        B_{i,j} = \begin{cases}
            c_{i,j} & (j\leq i\leq M)
            \\ 0 & (\text{otherwise}).
        \end{cases}
    \end{align}
    Then we have
    \begin{align}
        \begin{pmatrix}
            \tilde{v_1}
            &\cdots 
            &\tilde{v_M}
        \end{pmatrix}
        =\begin{pmatrix}
         v_1
           &\cdots 
           &v_M
        \end{pmatrix}
        B^\top.
    \end{align}
    Since the non-diagonal terms are bounded by $\cC M^{-1}\leq \frac13 M^{-1}$  and the absolute value of diagonal terms is no smaller than $1-20\cA\geq \frac23$, we know that $B^\top$ is invertible and (the absolute value of) the maximum eigenvalue of $(B^\top)^{-1}$ is bounded by $3$.
    Each $v_m$ is computed as
    \begin{align}
        v_m = \begin{pmatrix}
            \tilde{v_1}
            &\cdots 
            &\tilde{v_M}
        \end{pmatrix}(B^\top)^{-1} {\sf e}_m^\top
    \end{align}
    and 
    \begin{align}
        A =\sum_{m=1}^M \begin{pmatrix}
            \tilde{v_1}
            &\cdots 
            &\tilde{v_M}
        \end{pmatrix}(B^\top)^{-1} {\sf e}_m {\sf e}_m^\top ((B^\top)^{-1})^\top\begin{pmatrix}
            \tilde{v_1}^\top
        \\\vdots 
        \\\tilde{v_M}^\top
        \end{pmatrix}.
    \end{align}
    Thus, 
    \begin{align}
       \lambda_{\rm max}(A) \leq \lambda_{\rm max}^2\bigg(\begin{pmatrix}
            \tilde{v_1}
            &\cdots 
            &\tilde{v_M}
        \end{pmatrix}\bigg) \lambda_{\rm max}^2((B^\top)^{-1}) \lambda_{\rm max}((\sum_{m=1}^M{\sf e}_m {\sf e}_m^\top))\leq 9.
    \end{align}
    Therefore, $ \sum_{m=2}^M |\bar{\kappa}^s_m|^2={w^s}^\top (I-v_1v_1^\top)^\top  A (I-v_1v_1^\top)w^s \leq 9\|(I-v_1v_1^\top)w^s\|^2 = 9(1-(\kappa_1^s)^2)$.
    Based on this, if $\kappa_1^s\geq 1-3\cA$, we have
    \begin{align}
        \eqref{eq:P3-A1-4} &\leq 54q^2\max_{i,m}|\alpha_i\beta_{m,i}|(1-(\kappa_1^s)^2)^{\frac32} + 6q^2\max_{i,m}|\alpha_i\beta_{m,i}|\max_{m\ne 1}|v_1^\top v_m|^2(1-(\kappa_1^s)^2)^{\frac12}
        \\ & \leq \frac{1}{5}p\alpha_p\beta_{1,p}(1-\kappa_1^s).
    \end{align}
    The fifth term $\kappa_1^s\eta^2 \CB^2 d$ of \eqref{eq:P3-A1-1} is bounded by $\frac{\eta}{5\sqrt{M}}p\alpha_p\beta_{1,p}(1-\kappa_1^s)$, when $\eta \leq \cC \bar{\varepsilon}M^{-\frac12}d^{-1}$.

    Putting things together, 
    if $1-3\cA\leq\kappa_1^s\leq 1-\bar{\varepsilon}$ for all $s=0,1,\cdots,\tau$, 
    we have
    \begin{align}
        \kappa_1^{\tau+1} &\geq \kappa_1^{\tau} 
        + \frac{6}{5}\frac{\eta}{\sqrt{M}}p\alpha_p\beta_{1,p}(1-\kappa_1^\tau)
        +\eta v_1^\top (I-w^\tau{w^\tau}^\top)  Z^\tau
        \\ & \geq \kappa_1^{0}+ \sum_{s=0}^\tau  \frac{6}{5}\frac{\eta}{\sqrt{M}}p\alpha_p\beta_{1,p}(1-\kappa_1^s)
        +\eta\sum_{s=1}^\tau v_1^\top (I-w^s{w^s}^\top) Z^s
        \label{eq:P3-A1-5}
        .
    \end{align}
     Moreover, by $\eta \leq \cC M^{-\frac12}\bar{\varepsilon}^2$, 
    \begin{align}
      &   \left|\sum_{s=0}^\tau\eta v_1^\top (I-w^s{w^s}^\top) Z^s\right|
    \\  &   \leq \eta \CB \sqrt{\tau}
      \\ & \leq  \begin{cases}\displaystyle
         \cA\bar{\varepsilon}  & (\tau \leq \CA M\bar{\varepsilon}^{-2})
         \\ \displaystyle
         \frac{\tau\eta}{5\sqrt{M}}p\alpha_p\beta_{1,p}(1-\kappa_1^s)  &(\tau > \CA M\bar{\varepsilon}^{-2})
      \end{cases}
    \end{align}
    Therefore, 
    if $1-3\cA\leq\kappa_1^s\leq 1-\bar{\varepsilon}$ for all $s=0,1,\cdots,\tau$, 
    \eqref{eq:P3-A1-5} is bounded by
    \begin{align}
      \kappa_1^{\tau+1}& \geq \kappa_1^{0}-\cA\bar{\varepsilon}+ \sum_{s=0}^\tau  \frac{\eta}{\sqrt{M}}p\alpha_p\beta_{1,p}(1-\kappa_1^s)
     \\ & \geq \kappa_1^{0}-\cA\bar{\varepsilon}+ \sum_{s=0}^\tau  \frac{\eta}{\sqrt{M}}p\alpha_p\beta_{1,p} \bar{\varepsilon}
      \\ &\geq \kappa_1^{0}-\cA\bar{\varepsilon}+ \tau \frac{\eta}{\sqrt{M}}p\alpha_p\beta_{1,p} \bar{\varepsilon}
      \label{eq:P3-A1-6}
      ,
    \end{align}
    and $1-3\cA\leq\kappa_1^s$ holds for $s=\tau+1$.
    By induction,  $1-3\cA\leq\kappa_1^s$ holds for all $s=0,1,\cdots,t+1$, and the bound \eqref{eq:P3-A1-6} holds for all $\tau=0,1,\cdots,t$.
\end{proof}

\begin{proofof}[Lemma~\ref{lemma:AlignmentFinalPhase}]
    First, consider $0\leq t\leq (T_{1,1}-t_1)+(T_{1,2}-t_2)$.
    By the choice of $\eta_1\leq \cC M^{-\frac12}d^{-\frac{p}{2}}$, $\eta \leq \min\{\cC M^{-\frac12}d^{-\frac{p}{2}},\cC \bar{\varepsilon}^2M^{-\frac12}\}$ is satisfied with $\bar{\varepsilon}=d^{-\frac12}$ in Lemma~\ref{lemma:P3-Lower}.
    Thus, according to Lemma~\ref{lemma:P3-Lower}, until $\kappa_1^t>1-\bar{\varepsilon}$ holds (we let $\tau$ be the earliest time this condition holds), $\kappa_1^t$ is lower bounded by $\kappa_1^0-\bar{\varepsilon}\geq 1-2\cA$.

    One can also see that $\kappa_1^t\geq 1-2\cA$ holds for all $\tau\leq t \leq (T_{1,1}-t_1)+(T_{1,2}-t_2)$.
   Suppose that there exists some $\tau'<(T_{1,1}-t_1)+(T_{1,2}-t_2)$ such that $\kappa_1^{t}< 1-\bar{\varepsilon}$.
   Among such $\tau'$, we focus on the earliest time.
   According to Lemma~\ref{lemma:OneStep}, $|\kappa_1^{\tau'}-\kappa_1^{\tau'-1}|\leq \CB \eta_1$, which implies that $\kappa_1^{\tau'}\geq 1-2\bar{\varepsilon}$.
   Then, According to Lemma~\ref{lemma:P3-Lower}, $\kappa_1^t\geq \kappa_1^{\tau'}-\cA\geq 1-2\cA$ until $\kappa_1^t$ gets larger than $1-\bar{\varepsilon}$ again. 
   By repeating this argument, we obtain that $\kappa_1^0-\bar{\varepsilon}\geq 1-2\cA$ for all $0\leq t\geq (T_{1,1}-t_1)+(T_{1,2}-t_2)$.

   Then, we consider $(T_{1,1}-t_1)+(T_{1,2}-t_2)\leq t \leq (T_{1,1}-t_1)+(T_{1,2}-t_2)+T_{1,3}$.
   By the choice of $\eta_2\leq \min\{\cC \tilde{\varepsilon}M^{-\frac12}d^{-1},\cC \tilde{\varepsilon}^2M^{-\frac12}\}$, $\eta\leq \min\{\cC \bar{\varepsilon}M^{-\frac12}d^{-1},\cC \bar{\varepsilon}^2M^{-\frac12}\}$ is satisfied with $\bar{\varepsilon}=\tilde{\varepsilon}$ in Lemma~\ref{lemma:P3-Lower}.
% $\eta \leq \min\{\cC M^{-\frac12}d^{-\frac{p}{2}},\cC \bar{\varepsilon}M^{-\frac12} d^{-1},\cC \bar{\varepsilon}^2M^{-\frac12}\}$
     Suppose that $\kappa_1^s\leq 1-\tilde{\varepsilon}$ holds for all $s=(T_{1,1}-t_1)+(T_{1,2}-t_2),\cdots,(T_{1,1}-t_1)+(T_{1,2}-t_2)+T_{1,3}$, where 
    \begin{align}
    T_{1,3}= \left\lfloor3\cA\left(\frac{\eta}{\sqrt{M}}\tilde{\varepsilon}p\alpha_p\beta_{1,p}\right)^{-1}\right\rfloor + 1.
    \end{align}
    According to Lemma~\ref{lemma:P3-Lower}, 
    the bound of Lemma~\ref{lemma:P3-Lower} holds for all $s=(T_{1,1}-t_1)+(T_{1,2}-t_2),\cdots,(T_{1,1}-t_1)+(T_{1,2}-t_2)+T_{1,3}$:
    \begin{align}
        \kappa_1^s \geq 1-2\cA-\cA\tilde{\varepsilon} + s \tilde{\varepsilon}\frac{\eta}{\sqrt{M}}p\alpha_p\beta_{1,p}
        \geq 1-3\cA+ s \tilde{\varepsilon}\frac{\eta}{\sqrt{M}}p\alpha_p\beta_{1,p}
        .
        \label{eq:Criteria-309}
    \end{align}
   However, at $t =(T_{1,1}-t_1)+(T_{1,2}-t_2)+T_{1,3}$, 
   \begin{align}
       \text{(RHS of \eqref{eq:Criteria-309})}
       \geq 1,
   \end{align}
   and thus RHS of \eqref{eq:Criteria-309} is clearly larger than $1-\tilde{\varepsilon}$.
   This yields the contradiction.
   Therefore, with high probability, there exists some $t_3\leq (T_{1,1}-t_1)+(T_{1,2}-t_2)+T_{1,3}$ such that
   $\kappa_1^{t_3}> 1-\tilde{\varepsilon}$.

   Finally, suppose that there exists some $t>t_3$ such that $\kappa_1^{t}< 1-\tilde{\varepsilon}$.
   Among such $t>t_3$, we focus on the smallest $\tau$.
   According to Lemma~\ref{lemma:OneStep}, $|\kappa_1^\tau-\kappa_1^{\tau-1}|\leq \CB \eta_2$, which implies that $\kappa_1^\tau\geq 1-2\tilde{\varepsilon}$.
   Then, According to Lemma~\ref{lemma:P3-Lower}, $\kappa_1^t\geq 1-3\tilde{\varepsilon}$ until $\kappa_1^t$ gets larger than $1-\tilde{\varepsilon}$ again. 
   By repeating this argument, we obtain that $\kappa_1^{t}> 1-3\tilde{\varepsilon}$ for all $t_3\leq t\leq (T_{1,1}-t_1)+(T_{1,2}-t_2)+T_{1,3}$, which concludes the proof. 
\end{proofof}

\subsection{Expressivity of the Trained Feature Map}\label{subsection:Expressivity}

In this section we discuss the expressivity of the feature map after first-layer training. 
First, we consider the approximation of single-index polynomials and show the existence of suitable second-layer parameters with small approximation error and low norm.
\paragraph{ReLU activation.}
For $\sigma=\mathrm{ReLU}$, we have the following result.
\begin{lemma}\label{lemma:DamianApproximation}
    Suppose that $b_j \sim \mathrm{Unif}([-C_b,C_b])$ with $C_b=\tilde{O}(1)$, and consider the approximation of degree-$q$ polynomial $h(s)$, where $q=O_d(1)$.
    Let $v\in \mathbb{S}^{d-1}(1)$ and $v_-=-v$.
    Then, there exists $a_1,\dots,a_{2N}$ such that
    \begin{align}
       \sup_{t=T_1+1,\cdots,T_1+T_2}\left| \frac{1}{2N}\sum_{j=1}^N a_j \sigma (v^\top x^t +b_j)
       -\frac{1}{2N}\sum_{j=N+1}^{2N} a_j \sigma (v_-^\top x^t +b_j) - h(v^\top x^t)\right|=\tilde{O}(N^{-1}),
    \end{align}
    Moreover, we have $\sum_{j=1}^{2N} a_j^2 = \tilde{O}(N)$ and $\sum_{j=1}^{2N}|a_j|=\tilde{O}(N)$.
\end{lemma}
\begin{proof}
    According to Lemma 9 of \citet{damian2022neural}, if $b\sim \mathrm{Unif}([-1,1])$ and $\delta\sim \mathrm{Unif}(\{-1,1\})$, for any $k\geq 0$, there exists $v_k(\delta,b)$ with $|v_k(\delta,b)| \lesssim 1$ such that for all $s$ with $|s|\leq 1$,
    \begin{align}
        \mathbb{E}[v_k(\delta ,b)\sigma(\delta s+b)]=s^k.
    \end{align}
    Thus, for $C_b=\tilde{O}(1)$, if $b\sim \mathrm{Unif}([-C_b,C_b])$ and $\delta \sim \mathrm{Unif}(\{-1,1\})$, there exists $\bar{v}(\delta,b;h)$ with $|\bar{v}(\delta,b;h)| =\tilde{O}(1)$ such that for all $s$ with $|s|\leq C_b$,
    \begin{align}
        \mathbb{E}[\bar{v}(\delta,b;h)\sigma(\delta s+b)]=h(s).
    \end{align}
    We take $C_b=\tilde{\Theta}(1)$ sufficiently large so that for all $x^t\ (t=T_1+1,\cdots,T_1+T_2)$,  $|v^\top x^t|\leq C_b$ holds, with high probability.
    For $A=\tilde{\Theta}(N)$,
    we consider $2A$ intervals $[-C_b, C_b(-1+\frac{1}{A})),[C_b(-1+\frac{1}{A}), C_b(-1+\frac{2}{A})),\cdots,[C_b(1-\frac{1}{A}), C_b]$. 
    By taking the hidden constant sufficiently small, for each interval there exists at least one $b_j$.
    Then, for $b_j$ corresponding to $[C_b(-1+\frac{i}{A}), C_b(-1+\frac{i+1}{A}))$, 
    we set $a_j = \frac{N}{2}\int_{C_b(-1+\frac{i}{A})}^{C_b(-1+\frac{i+1}{A}))}\bar{v}(1,b;h)\mathrm{d}b$ for $1\leq j\leq N$, and $a_j = \frac{N}{2}\int_{C_b(-1+\frac{i}{A})}^{C_b(-1+\frac{i+1}{A}))}\bar{v}(-1,b;h)\mathrm{d}b$ otherwise.
    Here we note that $|a_j| = \tilde{O}(1)$ holds for all $j$.
    If each interval contains more than one $b_j$, we ignore all but one component by letting $a_j=0$.
    In doing so, 
    since $\sigma(s+b)$ is $1$-Lipschitz with respect to $s$, 
    we have
    \begin{align}
       \sup_{t=T_1+1,\cdots,T_1+T_2}\left| \frac{1}{2N}\sum_{j=1}^N a_j \sigma (v^\top x^t +b_j)
       -\frac{1}{2N}\sum_{j=N+1}^{2N} a_j \sigma (v_-^\top x^t +b_j) - h(v^\top x^t)\right|=\tilde{O}(N^{-1}),
    \end{align}
    with high probability.
    Thus we obtain the assertion.
\end{proof}

\paragraph{Polynomial Activation}
If $\sigma$ is a degree-$q$ polynomial, we have the following result.
\begin{lemma}\label{lemma:DamianApproximationPolynomial}
     Suppose that $b_j \sim \mathrm{Unif}([-C_b,C_b])$ with $C_b=\tilde{O}(1)$, and consider the approximation of degree-$q$ polynomial $h(s)$, where $q=O_d(1)$. 
    Then, there exists $a_1,\dots,a_{N}$ such that
    \begin{align}
       \sup_{t=T_1+1,\cdots,T_1+T_2}\left| \frac{1}{N}\sum_{j=1}^N a_j \sigma (v^\top x^t +b_j)
       - h(v^\top x^t)\right|=\tilde{O}(N^{-1})
    \end{align}
    with high probability, where $v\in \mathbb{S}^{d-1}$.
    Moreover, we have $\sum_{j=1}^{N} a_j^2 = \tilde{O}(N)$ and $\sum_{j=1}^{N}|a_j|=\tilde{O}(N)$.
\end{lemma}
The lemma depends on the following result.
\begin{lemma}\label{lemma:ApproxPolynomialByHeq}
    Suppose that $C_b\geq q$.
    For any polynomial $h(s)$ with degree at most $q$, there exists $\bar{v}(b;h)$ with $|\bar{v}(b;h)| \lesssim C_b$ such that for all $s$,
    \begin{align}
        \mathbb{E}[\bar{v}(b;h)\sigma(\delta s+b)]=h(s).
    \end{align}
\end{lemma}
\begin{proof}
      When $g_{q}(s)=\sigma(s)$ is a degree-$q$ polynomial, 
    % for any distribution $\mu$ of $b$, 
    \begin{align}
        g_{q}(s)=\int_{b=-q}^0\sigma(s+b)\mathrm{d}b
    \end{align}
    is also a degree-$q$ polynomial.

    Let us repeatedly define
    \begin{align}
        g_{q-i}(s):=g_{q-(i-1)}(s+1)-g_{q-(i-1)}(s)\quad (i=1,2,\cdots,q),
    \end{align}
    and let $(c_{i,j})$ be coefficients so that $(s-1)^i = \sum_{j=0}^i c_{i,j}s^j$ holds for all $z$.
    Then, by induction, $g_{i}(s)$ is a degree-$i$ polynomial.
    Moreover, we have 
    \begin{align}
        g_{q-i}(s)&=\sum_{j=0}^ic_{i,j}\int_{b=-q}^0\sigma(s+b+j)\mathrm{d}b
        \\ & =2C_b\mathbb{E}_{b\sim \mathrm{Unif}([-C_b,C_b])}\bigg[\bigg(\sum_{j=0}^ic_{i,j}\mathbbm{1}[j-q\leq b\leq j]\bigg)\sigma(s+b)\bigg]
        ,
    \end{align}
    when $C_b\geq q$. 
    Therefore, for any polynomial $h(s)$ with its degree at most $q$, there exists $\bar{v}(b;h)$ with $|\bar{v}(b;h)| \lesssim C_b$ such that for all $s$,
    \begin{align}
        \mathbb{E}[\bar{v}(b;h)\sigma(\delta s+b)]=h(s).
    \end{align}
\end{proof}
\begin{proofof}[Lemma~\ref{lemma:DamianApproximationPolynomial}]
     We now discretize Lemma~\ref{lemma:ApproxPolynomialByHeq}.
    For $A=\tilde{\Theta}(N)$ (with a sufficiently small hidden constant),
    we consider $2A$ intervals $[-C_b, C_b(-1+\frac{1}{A})),[C_b(-1+\frac{1}{A}), C_b(-1+\frac{2}{A})),\cdots,[C_b(1-\frac{1}{A}), C_b]$.
    By taking the hidden constant sufficiently small, for each interval there exists at least one $b_j$.
    Then, for $b_j$ corresponding to $[C_b(-1+\frac{i}{A}), C_b(-1+\frac{i+1}{A}))$, 
    we set $a_j = \frac{N}{2}\int_{C_b(-1+\frac{i}{A})}^{C_b(-1+\frac{i+1}{A}))}\bar{v}(b;h)\mathrm{d}b$.
    Here we note that $|a_j| = \tilde{O}(1)$ holds for all $j$.
    Due to Lipschitzness of $\sigma$, 
    we have
    \begin{align}
        \left| \frac{1}{N}\sum_{j=1}^N a_j \sigma (s+b_j)
     - h(s)\right|= \tilde{O}(N)
    \end{align}
    for all $s=\tilde{O}(1)$. 
    Because $|v^\top x^t|=\tilde{O}(1)$ with high probability, we have
    \begin{align}
       \sup_{t=T_1+1,\cdots,T_1+T_2}\left| \frac{1}{N}\sum_{j=1}^N a_j \sigma (v^\top x^t +b_j)
       - h(v^\top x^t)\right|=\tilde{O}(N^{-1})
    \end{align}
    with high probability, which yields the assertion. 
\end{proofof}

Next, by using the expressivity results above, we show that there exists some $a^*$ that can approximate the additive target function $f_*$.
\begin{lemma}\label{lemm:ApproximationByAstar}
    % If $J = \tilde{\Theta}(M^{1+C_p}d^{\frac{p-2}{2}}\varepsilon^{-2})$, 
    If $J\gtrsim J_{\mathrm{min}}M^{C_p}\log d$, and $\sigma$ is either of the ReLU activation or any univariate polynomial with degree $q$, 
    there exists some parameters $a^* = (a_j^*)_{j=1}^J \in \R^{J}$ such that
    \begin{align}
       %\mathbb{E}_{x\sim \mathcal{N}(0,I_d)}\left[
       \frac{1}{T_2}\sum_{t=T_1+1}^{T_1+T_2}
       \left(\frac{1}{J}\sum_{j=1}^J a_j^* \sigma (\hat{w}_j^\top x^t+b_j) - \frac{1}{\sqrt{M}}\sum_{m=1}^Mf_m(v_m^\top x^t)\right)^2%\right] 
       \leq \CB M (|J_{\min}|^{-2} + \tilde{\varepsilon}^2),
    \end{align}
    %This $a^*$ satisfies that 
    where  
    $\|a^*\|_2^2 = \tilde{O}(J^2|J_{\mathrm{min}}|^{-1})$
    % $\|a^*\|_2^2 =\tilde{O}(M^{1+2C_p}d^{p-2} \varepsilon^{-2})$
    and 
    % $\|a^*\|_1 =\tilde{O}(M^{\frac32+C_p}d^{\frac{p-2}{2}}\varepsilon^{-2})$
    $\|a^*\|_1 = \tilde{O}(JM^{\frac12})$.
\end{lemma}
\begin{proof}
    We only discuss the case of the ReLU activation; the result for degree-$q$ polynomial link functions follows from the exact same analysis. 
    Let $\tilde{J}_{m,+}$ be a set of neurons satisfying $ \hat{w}_j\geq 1-3\tilde{\varepsilon}$ and $\tilde{J}_{m,-}$ be a set of neurons with $\hat{w}_j\leq -1+3\tilde{\varepsilon}$.
    Also, we let $\tilde{J}_{m}=\tilde{J}_{m,+}\cup \tilde{J}_{m,-}$.
    According to Lemma~\ref{lemma:Goal-1}, when $J\gtrsim J_{\mathrm{min}}M^{C_p}\log d$, we have $|\tilde{J}_{m,+}|, |\tilde{J}_{m,-}|\geq J_{\mathrm{min}}$ with high probability.
    
    If $j\notin \bigcup_{m}\tilde{J}_m$, we set $a_j=0$. 
    Also, if $\tilde{J}_{m,+}$ or $\tilde{J}_{m,-}$ contains more than $J_{\mathrm{min}}$ neurons, we ignore the rest by simply setting $a_j=0$. 
    Then, 
    \begin{align}
    &%\mathbb{E}_{x\sim \mathcal{N}(0,I_d)}\left[
    \frac{1}{T_2}\sum_{t=T_1+1}^{T_1+T_2}
    \left(\frac{1}{J}\sum_{j=1}^J a_j \sigma (\hat{w}_j^\top x^t+b_j) - \frac{1}{\sqrt{M}}\sum_{m=1}^Mf_m(v_m^\top x^t)\right)^2
    %\right]
    \\ &
    =\frac1M\sum_{m,m'=1}^M
    %\mathbb{E}_{x\sim \mathcal{N}(0,I_d)}\Bigg[
    \frac{1}{T_2}\sum_{t=T_1+1}^{T_1+T_2}
    \Bigg(\frac{\sqrt{M}}{J}\sum_{j\in \tilde{J}_m} a_j \sigma (\hat{w}_j^\top x^t+b_j) -f_m(v_m^\top x^t)\Bigg) \\ & \hspace{50mm}\Bigg(\frac{\sqrt{M}}{J}\sum_{j\in \tilde{J}_{m'}} a_j \sigma (\hat{w}_j^\top x^t+b_j) - f_{m'}(v_{m'}^\top x^t)\Bigg). %\Bigg].
    \label{eq:ApproximationError}
\end{align}
 For each $m$, we evaluate $|\frac{\sqrt{M}}{J}\sum_{j\in \tilde{J}_m} a_j \sigma (\hat{w}_j^\top x+b_j^*) -f_m(v_m^\top x)|$. 
For the $j$-th neuron in $\tilde{J}_{m,+}$, by using $a_j$ in Lemma~\ref{lemma:DamianApproximation} with  $N=|\tilde{J}_{\min}|$, we define $a_j^*\leftarrow \frac{J}{2|J_{\min}|\sqrt{M}}a_j$; similarly for the $j$-th neuron in $\tilde{J}_{m,-}$, we define $a_j^*\leftarrow \frac{J}{2|\tilde{J}_{\min}|\sqrt{M}}a_{N+j}$.
We obtain that
\begin{align}
    &\Bigg|\frac{\sqrt{M}}{J}\sum_{j\in \tilde{J}_m} a_j^* \sigma (\hat{w}_j^\top x+b_j^*) -f_m(v_m^\top x)\Bigg| 
   \\ &\leq \Bigg|\frac{\sqrt{M}}{J}\sum_{j\in \tilde{J}_m} a_j^* \sigma (\delta_j v_m^\top x+b_j^*) -f_m(v_m^\top x)\Bigg| + \Bigg|\frac{\sqrt{M}}{J}\sum_{j\in \tilde{J}_m} (a_j^* \sigma (\hat{w}_j^\top x+b_j^*) -a_j^* \sigma(\delta_j v_m^\top x+b_j^*))\Bigg|
  \\  &\leq \CB (|J_{\min}|^{-1} + \tilde{\varepsilon}).
\end{align}
% Then, \eqref{eq:ApproximationError} is bounded by $\tilde{O}(\varepsilon^2)$.
The norm can be calculated from the construction.
\end{proof}

\subsection{Fitting the Second Layer}\label{subsection:FittingSecondLayer}
This subsection proves the generalization error in Lemma~\ref{lemma:Generalization}, which concludes the proof of Theorem~\ref{theorem:NN-main} together with the guarantee for the first-layer training (Lemma~\ref{lemma:Goal-1}).

Let $\hat{a}$ be the regularized empirical risk minimizer with $L^1$ or $L^2$ norm regularization: 
$$
\hat{a} := \mathop{\arg\min}_{a \in \R^{J}} \hat{\mathcal{L}}(a) + \frac{\bar{\lambda}}{r} \|a\|^r_r, 
$$
where $\hat{\mathcal{L}}(a) := \frac{1}{T_2}\sum_{t=T_1+1}^{T_1+T_2-1} (y^t -\frac{1}{J}\sum_{j=1}^J a_j \sigma (\hat{w}_j^\top x^t+b_j)  )^2$ and $r \in \{1,2\}$. 
Let $f_a(x) := \frac{1}{J}\sum_{j=1}^J a_j \sigma (\hat{w}_j^\top x+b_j) $. 
Then, by the equivalence between convex regularizer and norm-constraint, 
there exists $\bar{\lambda} > 0$ (which can be data dependent) such that 
\begin{align}
& \hat{\mathcal{L}}(\hat{a}) \leq \hat{\mathcal{L}}(a^*),~~~~\|a\|_{r} \leq \|a^*\|_r\quad (r=1\text{ or }2). 
\label{eq:ConstraintOptimality}
\end{align}
Indeed, let $g_B: \R^{J} \to \R \cup \{\infty\}$ be the indicator function of the $L^p$-norm ball with radius $B > 0$: $g_B(x) = \begin{cases} 0 & (\|x\|_r \leq B) \\ \infty &( \|x\|_r > B) \end{cases}$. Then, the minimizer of the following norm constraint optimization problem 
$$
\hat{x} := \mathop{\arg\min}_{x \in \R^{J}} \hat{\mathcal{L}}(x) + g_B(x),
$$
should satisfy $\nabla \hat{\mathcal{L}}(\hat{x}) \in - \partial g_B(\hat{x})$ where $\partial g_B(\hat{x})$ is the subgradient of $g_B$ at $\hat{x}$. 
We notice that any element $v$ in $\partial g_B(\hat{x})$ can be expressed by $v = c v'$ for some $c \geq 0$ 
 and $v' \in \partial \|x\|_r|_{x=\hat{x}}$. This means that the minimizer $\hat{x}$ of the norm-constrained problem is also the minimizer of the regularized objective $\hat{\mathcal{L}}(x) + c \|x\|_r$.
 By resetting the value $c$ as $c \leftarrow c/(p\|\hat{x}\|^{r-1})$, it is also the minimizer of the objective $\hat{\mathcal{L}}(x) + c \|x\|_r^r$.

\newtheorem*{lemma:Generalization}{\rm\bf Lemma~\ref{lemma:Generalization}}
\begin{lemma:Generalization}
 Suppose that %$J = \tilde{\Theta}(d^{\frac{p-2}{2}}M^{C}d^{\frac{p-2}{2}}\varepsilon^{-2})$
 $J = \Theta(J_{\mathrm{min}}M^{C_p}\log(d))$, and $\sigma$ be  either of the ReLU activation or any univariate polynomial with degree $q$.
 There exists $\lambda > 0$ such that the ridge estimator $\hat{a}$ satisfies 
$$
\mathbb{E}_{x}[|f_{\hat{a}}(x) - f_*(x)|] 
\lesssim 
M^\frac12 (|J_{\mathrm{min}}|^{-1}+\tilde{\varepsilon})+
{\sqrt{\frac{M^{C_p}\log(d)}{T_2}}}
$$
with probability $1-o_d(1)$.
Therefore, by taking {$T_2=\tilde{\Theta}(M^{C_p}\varepsilon^{-2})$}, $\tilde{\varepsilon}=\tilde{\Theta}(M^{-\frac12}\varepsilon)$, $J_{\mathrm{min}}=\tilde{\Theta}(M^\frac12 \varepsilon^{-1})$, and $J=\tilde{\Theta}(M^{C_p+\frac12}\varepsilon^{-1})$, we have $\mathbb{E}_{x}[|f_{\hat{a}}(x) - f_*(x)|]\lesssim \varepsilon$.

On the other hand, for LASSO ($r=1$) we have
$$
\mathbb{E}_{x}[|f_{\hat{a}}(x) - f_*(x)|] 
\lesssim
M^\frac12 (|J_{\mathrm{min}}|^{-1}+\tilde{\varepsilon})+
{\sqrt{\frac{J_{\min}^{2/s}M^{2C_p/s+1}\log(d)^{2/s}}{T_2}}}
% \lesssim \frac{M^\frac12\sqrt{d}}{\sqrt{T_2}}  
% + \tilde{O}(\varepsilon), %\frac{1}{T_2}
$$
with probability $1-o_d(1)$, for arbitrary $s < \infty$ (where the hidden constant may depend on $s$). 
Therefore, by taking {$T_2=\tilde{\Theta}(M^{1 + \frac{2C_p +1}{s}}\varepsilon^{-2-\frac{2}{s}})$}, $\tilde{\varepsilon}=\tilde{\Theta}(M^{-\frac12}\varepsilon)$, $J_{\mathrm{min}}=\tilde{\Theta}(M^\frac12 \varepsilon^{-1})$, and $J=\tilde{\Theta}(M^{C_p+\frac12}\varepsilon^{-1})$, we have $\mathbb{E}_{x}[|f_{\hat{a}}(x) - f_*(x)|]\lesssim \varepsilon$
(here we ignore polylogarithmic factors). 
\end{lemma:Generalization}
\begin{proof}
Let $\mathcal{F}_{a^*} :=\{ f_a \mid \|a\|_r \leq \|a^*\|_r \}$ 
and $P_{T_2}$ be the empirical distribution of the second stage: $P_{T_2} := \frac{1}{T_2}\sum_{t=T_1 + 1}^{T_1 + T_2} \delta_{x_t}$.
%and $\mathcal{L}(f) := \mathbb{E}[(f(x) - f_*(x))^2]$.  
If we choose $\lambda$ as mentioned above, $\hat{a}$ satisfies the condition \eqref{eq:ConstraintOptimality}, which yields that $f_{\hat{a}}\in \mathcal{F}_{a^*}$.
Therefore, we have that 
\begin{align}
& \|f_{\hat{a}} - f_*\|_{L^1(P_x)} \\
%&\leq 
%2 \|f_{\hat{a}} - f_{a^*}\|_{L^2(P_x)}^2 
%+ 2\|f_{a^*} - f_*\|_{L^2(P_x)}^2 \\ 
& = 
\|f_{\hat{a}} - f_*\|_{L^1(P_x)} 
-\|f_{\hat{a}} - f_*\|_{L^1(P_{T_2})} 
+  \|f_{\hat{a}} - f_*\|_{L^1(P_{T_2})} \\
& \leq 
\sup_{a \in \R^{J}: \|a\| \leq \|a^*\|}
\left[ \|f_{a} - f_*\|_{L^1(P_x)} 
-\|f_{a} - f_*\|_{L^1(P_{T_2})} 
\right]
+  \|f_{\hat{a}} - f_*\|_{L^2(P_{T_2})}.
\label{eq:LoneLossFirstBound}
\end{align}
\paragraph{(1)} 
First, we bound the second term $\|f_{\hat{a}} - f_*\|_{L^2(P_{T_2})}$. 
Since $\hat{\mathcal{L}}(f_{\hat{a}}) \leq \hat{\mathcal{L}}(f_{a^*})$, we have that 
$$
\|f_{\hat{a}} - f_*\|_{L^2(P_{T_2})}^2
\leq \|f_{a^*} - f_*\|_{L^2(P_{T_2})}^2 + \frac{2}{T_2}\sum_{t=T_1 + 1}^{T_1 + T_2}(f_{a^*}(x^t) - f_{\hat{a}}(x^t))\nu^t.
$$
By the Hoeffding inequality, we have that 
$$
\frac{2}{T_2}\sum_{t=T_1 + 1}^{T_1 + T_2}(f_{a^*}(x^t) - f_{\hat{a}}(x^t))\epsilon^t
= \tilde{O}\left( \sqrt{\frac{\|f_{a^*} - f_{\hat{a}}\|_{L^2(P_{T_2})}^2}{T_2}}\right),
$$
with high probability. The right hand side can be further bounded by 
\begin{align}
\tilde{O}\left( \sqrt{\frac{\|f_{a^*} - f_{\hat{a}}\|_{L^2(P_{T_2})}^2}{T_2}}\right)
& \leq 
\tilde{O}\left( \sqrt{\frac{\|f_{a^*} - f_*\|_{L^2(P_{T_2})}^2 + \| f_{\hat{a}} - f_*\|_{L^2(P_{T_2})}^2}{T_2}}\right) \\
& 
\leq 
\frac{1}{2}\| f_{\hat{a}} - f_*\|_{L^2(P_{T_2})}^2
+
\frac{1}{2}\| f_{a^*} - f_*\|_{L^2(P_{T_2})}^2
+ \tilde{O}\left(\frac{1}{T_2} \right),
\end{align}
by the Cauchy-Schwarz inequality. 
Then, by moving the term $\frac{1}{2}\| f_{\hat{a}} - f_*\|_{L^2(P_{T_2})}^2$ in the right hand side to the left hand side, we have that 
$$
\| f_{\hat{a}} - f_*\|_{L^2(P_{T_2})}^2
\leq 
3 \| f_{a^*} - f_*\|_{L^2(P_{T_2})}^2
+ \tilde{O}\left(\frac{1}{T_2} \right), 
$$
with high probability. 
This also yields that 
\begin{align}
\| f_{\hat{a}} - f_*\|_{L^2(P_{T_2})}
& \leq 
\sqrt{3}  \| f_{a^*} - f_*\|_{L^2(P_{T_2})}
+ \tilde{O}\left(\frac{1}{\sqrt{T_2}} \right) \\
& =  \tilde{O}\left(M^\frac12 (|J_{\mathrm{min}}|^{-\frac12}+\tilde{\varepsilon}^{\frac12}) + \frac{1}{\sqrt{T_2}} \right),
\end{align}
where we used $\|f_{a^*} - f_*\|_{L^2(P_{T_2})}=\tilde{O}(M^\frac12 (|J_{\mathrm{min}}|^{-\frac12}+\tilde{\varepsilon}^{\frac12}))$ by Lemma \ref{lemm:ApproximationByAstar}.

\paragraph{(2)} The first term in \eqref{eq:LoneLossFirstBound} can be bounded by 
the standard Rademacher complexity argument (e.g., Chapter 4 of \citet{wainwright2019high}). Specifically, its expectation can be bounded as
\begin{align}
& \E_{(x^t)_{t=T_1+1}^{T_1 + T_2}}\left[ \sup_{a \in \R^{J}}
\left( \|f_{a} - f_*\|_{L^1(P_x)} 
-\|f_{a} - f_*\|_{L^1(P_{T_2})}  \right) \right] 
\label{eq:LonenormDiffFirstBound-100}
\\
\leq &
2  \E_{(x^t,\sigma_t)_{t=T_1+1}^{T_1 + T_2}}\left[ \sup_{a \in \R^{J}}
\left( \frac{1}{T_2} \sum_{t=T_1+ 1}^{T_1 + T_2} \sigma_t|f_{a}(x^t) - f_*(x^t)|  \right) \right] \\ 
\leq &
4 \underbrace{\E_{(x^t,\sigma_t)_{t=T_1+1}^{T_1 + T_2}}\left[ \sup_{a \in \R^{J}}
\left( \frac{1}{T_2} \sum_{t=T_1 + 1}^{T_1 + T_2} \sigma_t f_{a}(x^t)   \right) \right] }_{=: \mathrm{Rad}(\mathcal{F}_{a^*})}
+ 
4 \E_{(x^t,\sigma_t)_{t=T_1+1}^{T_1 + T_2}}\left[  \frac{1}{T_2} \sum_{t=T_1+ 1}^{T_1 + T_2} \sigma_t  f_*(x^t)    \right], 
\label{eq:LonenormDiffFirstBound}
\end{align}
where $(\sigma_t)_{t=T_1+1}^{T_1 + T_2}$ is the i.i.d. Rademacher sequence which is independent of $(x^t)_{t=T_1+1}^{T_1 + T_2}$
and we used the vector valued contraction inequality of the Rademacher complexity in the last inequality \cite{10.1007/978-3-319-46379-7_1}. 
Unfortunately, $f_*(X)$ is neither bounded nor sub-exponential, 
and thus we cannot naively apply the Bernstein type concentration inequality to evaluate the right hand side.
Hence, we utilize Markov's inequality instead to convert \eqref{eq:LonenormDiffFirstBound-100} to a high probability bound on $\sup_{a \in \R^{J}: \|a\| \leq \|a^*\|}
\left[ \|f_{a} - f_*\|_{L^1(P_x)} 
-\|f_{a} - f_*\|_{L^1(P_{T_2})} 
\right]$. 

From Lemma 48 of \citet{damian2022neural} and its proof, for either of ReLU and polynomial activation, 
we have
\begin{align}
    \mathrm{Rad}(\mathcal{F}_{a^*}) \lesssim \sqrt{\frac{1}{T_2}}\frac{\|a^*\|_r}{J}
    \max_j\{(J\E_x[\sigma_j(\hat{w}_j^\top x + b_j)^s])^{1/s}\},
    %\sqrt{\frac{d}{T_2}}\frac{\|a^*\|_1}{J^{1/p}}.
\end{align}
for arbitrary $s \leq 1/(1- 1/r)$.
However, since $\max_j\{\hat{w}_j,b_j\} = \Ord(1)$, we have $\max_j\{\E_x[\sigma_j(\hat{w}_j^\top x + b_j)^s]\} = \Ord(1)$ whenever $s < \infty$ by noting that 
$\hat{w}_j^\top x$ is a Gaussian distribution with variance $\mathrm{Var}(\hat{w}_j^\top x)=\Ord(1)$,  which yields that the right hand side can be bounded as 
\begin{align}    
& \mathrm{Rad}(\mathcal{F}_{a^*}) \lesssim 
\sqrt{\frac{1}{T_2}}\frac{\|a^*\|_2}{J^{1/2}}~~~(r=2), \\
& \mathrm{Rad}(\mathcal{F}_{a^*}) \lesssim 
\sqrt{\frac{1}{T_2}}\frac{\|a^*\|_1}{J^{1-1/s}}~~~(r=1),
\end{align}
where arbitrary $s< \infty$ (the hidden constant depends on $s$).

Applying Lemma \ref{lemma:HPB-Y} to the second term of \eqref{eq:LonenormDiffFirstBound} yields that 
\begin{align}
 \E\left[ \frac{1}{T_2} \sum_{t=T_1+ 1}^{T_1 + T_2} \sigma_t  f_*(x^t)   \right] 
\leq  & 
\sqrt{\E\left[\frac{1}{T_2} \sum_{t=T_1+ 1}^{T_1 + T_2}  f_*^2(x^t) \right]} \\
= &
\frac{1}{\sqrt{T_2}} \sqrt{\E\left[f_*^2(X) \right]}
\lesssim \frac{2^{q+1}}{\sqrt{T_2}}.
\end{align}

\paragraph{(3)} By combining evaluations of (1) and (2) together and ignoring polylogarithmic factors, we obtain that
\begin{align}
    \|f_{\hat{a}} - f_*\|_{L^1(P_x)} 
    \lesssim 
    M^\frac12 (|J_{\mathrm{min}}|^{-1}+\tilde{\varepsilon}) + \frac{1}{\sqrt{T_2}}+\sqrt{\frac{1}{T_2}}\frac{\|a^*\|_r}{J^{1/2}}
    .\label{eq:GeneralizationError-1}
\end{align}
We set $J=\Theta(J_{\mathrm{min}}M^{C_p}\log d)$. 
Thus, for $r=2$, by using \Cref{lemm:ApproximationByAstar}, we have
\begin{align}
    \eqref{eq:GeneralizationError-1} &\lesssim M^\frac12 (|J_{\mathrm{min}}|^{-1}+\tilde{\varepsilon}) + \frac{1}{\sqrt{T_2}}+\sqrt{\frac{1}{T_2}}J^\frac12|J_{\mathrm{min}}|^{-\frac12}
    \\ & \lesssim M^\frac12 (|J_{\mathrm{min}}|^{-1}+\tilde{\varepsilon})+\sqrt{\frac{M^{C_p}\log(d)}{T_2}}.
\end{align}
Thus, by setting $T_2=\tilde{\Theta}(M^{C_p}\varepsilon^{-2})$, $\tilde{\varepsilon}=\tilde{\Theta}(M^{-\frac12}\varepsilon)$, and $J_{\mathrm{min}}=\tilde{\Theta}(M^\frac12 \varepsilon^{-1})$, we obtain that $\eqref{eq:GeneralizationError-1}\lesssim \varepsilon$.

Similarly, for $r=1$, by \Cref{lemm:ApproximationByAstar}, we have
\begin{align}
    \eqref{eq:GeneralizationError-1} &\lesssim M^\frac12 (|J_{\mathrm{min}}|^{-1}+\tilde{\varepsilon}) + \frac{1}{\sqrt{T_2}}+\sqrt{\frac{1}{T_2}}\frac{JM^{\frac12}}{J^{1-1/s}}
    \\ & \lesssim M^\frac12 (|J_{\mathrm{min}}|^{-1}+\tilde{\varepsilon})+\sqrt{\frac{J_{\mathrm{min}}^{2/s}M^{2C_p/s + 1}\log(d)^{1/s} }{T_2}}.
\end{align}
Thus, by setting $T_2=\tilde{\Theta}(M^{1 + \frac{2C_p +1}{s}}\varepsilon^{-2-\frac{2}{s}})$, $\tilde{\varepsilon}=\tilde{\Theta}(M^{-\frac12}\varepsilon)$, and $J_{\mathrm{min}}=\tilde{\Theta}(M^\frac12 \varepsilon^{-1})$ with a sufficiently large $s$, we obtain that $\eqref{eq:GeneralizationError-1}\lesssim \varepsilon$. 
\end{proof}

\bigskip

\section{Proof of CSQ Lower Bounds}\label{section:Appendix-SQ}
We consider the CSQ lower bound for the following class $\mathcal{F}_{d,M,\varsigma}^{p,q}$:
    \begin{align}
    x\sim \mathcal{N}(0,I_d), \quad y = \frac{1}{\sqrt{M}}\sum_{m=1}^M f_m(v_m^\top x)+\nu,
    \end{align}
    where $f_m=\sum_{i=p}^q a_{m,i}\He_i$ with %&|a_{i,p}|\lesssim 1\ (i=p,\cdots,q) \text{ and }
    $v_k\in \mathbb{S}^{d-1},\ |a_{m,p}|\gtrsim 1$, $\mathbb{E}_{t\sim \mathcal{N}(0,1)}[|f_m(t)|^2]=1\ (m=1,\cdots,M)$,
    and $\nu \sim \mathcal{N}(0,\varsigma^2)$.
For the lower bound we may assume $\varsigma=0$ since this is the easiest case for the learner.

The correlational statistical query (CSQ) returns an expectation of the correlation between $y$ and a query function $q\colon \mathcal{X}\to \mathbb{R}$ up to an arbitrary (adversarial) error bounded by $\tau$.

\begin{definition}[Correlational statistical query]
    For a function $g\colon \mathcal{X}\to \mathbb{R}$ and parameters $\varsigma$,
    the correlational statistical query oracle $\mathrm{CSQ}(g,\varsigma, \tau)$ returns 
    \begin{align}
        \mathbb{E}_{x,y}[yg(x)]+\nu,
    \end{align}
    where $\nu$ is an arbitrary noise that takes any value in $\nu\in [-\tau,\tau]$.
\end{definition}
Without loss of generality, we assume $\|g\|_{L^2}=1$. We prove the lower bounds on CSQ learner below.

\subsection{Proof of \Cref{theorem:CSQ-Appendix}(a)}
% The proof follows from \citet{szorenyi2009characterizing,damian2022neural}.
We consider the following model with $\varsigma=0$: 
\begin{align}
    x\sim \mathcal{N}(0,I_d),\quad y = f_*(x)=\frac{1}{\sqrt{M}}\sum_{m=1}^M \frac{1}{\sqrt{p!}}\He_p(v_m^\top x),
\end{align}
where $\{v_1,\cdots,v_M\}$ is a randomly sampled subset (without duplication) of the set $S$ specified below.
Also, the following lemma guarantees that when $M= \tilde{o}(d^{\frac{p}{4}})$ we have $|\mathbb{E}[y^2]-1| =o(1)$.
\begin{lemma}%[Lemma 2 of \citet{damian2022neural}]
\label{lemma:NearOrthogonalBasis}
    For any $A$ and $d$, there exists a set $S$ of $A$ unit vectors in $\mathbb{R}^d$ such that, 
    for any $u,v\in S$, $u\ne v$, the inner product $|u^\top v|$ is bounded by $d^{-\frac12}\sqrt{2\log A}$.
\end{lemma}
\begin{proof}
    Let us sample $A$ independent vectors $v_1,\cdots,v_A$ from the $d$-dimensional hypercube $\big[-\frac{1}{\sqrt{d}},\frac{1}{\sqrt{d}}\big]^d$.
    For each pair of $v_i$ and $v_j$ $(i\ne j)$, Hoeffding's inequality yields
    \begin{align}
        \mathbb{P}[|v_i^\top v_j|\geq t]\leq 2e^{-t^2 d}.
    \end{align}
    By setting $t=d^{-1/2}\sqrt{2\log A}$, we have $|v_i^\top v_j|\leq d^{-1/2}\sqrt{2\log A}$ with probability no less than $1-\frac{2}{A^2}$. Taking the union bound, $|v_i^\top v_j|\leq d^{-1/2}\sqrt{2\log A}$ holds for all $(i,j)$ with probability no less than $1-\frac{A(A-1)}{A^2}>0$.
    This proves the existence of the desired $S$.
\end{proof}
As a result, we obtain a set of functions $\left\{\frac{1}{\sqrt{p!}}\He_p(v^\top x)\ |\ v\in S\right\}$ with small pairwise correlation:
\begin{align}
    \int \frac{1}{\sqrt{p!}}\He_p(u^\top x) \cdot \frac{1}{\sqrt{p!}}\He_p(v^\top x) \frac{1}{(2\pi)^\frac{d}{2}}e^{-\frac{\|x\|^2}{2}}\mathrm{d}x = (u^\top v)^p \leq d^{-\frac{p}{2}}(2\log A)^{\frac{p}{2}}.
\end{align}

Based on this calculation, we know that each correlational query cannot obtain information of the true function, except for the case when the true function is in a polynomial-sized set, as shown in \cite{szorenyi2009characterizing}.
\begin{lemma}\label{lemma:BoundCorrSet}
    Suppose that $\mathcal{F}=\{f_1,\cdots,f_K\}$ is a finite set of functions such that $|\mathbb{E}_{x\sim \mathcal{N}(0,I_d)}[f_i(x)f_j(x)]|\leq \varepsilon$ for all pairs of $f_i$ and $f_j$ with $f_i\ne f_j$.
    % $L$ is sufficient large and $\log m = \mathrm{poly}(d)$.
    % For each $a_i$ and $j$, 
    Then, for any query $h$ satisfying $\|h\|_{L^2}\leq 1$, there are at most $\frac{2}{\tau^2-\varepsilon}$ functions $f_i$  %$\mathcal{A}\subseteq \mathcal{F}$
    that satisfy
    $
       \left|\mathbb{E}_{x\sim \mathcal{N}(0,I_d)}\left[f(x)h(x)\right]\right| \geq \tau.
    $
\end{lemma}
\begin{proof}
    Let
    \begin{align}
        S_+:=\left\{i\in [K]\ |\ \mathbb{E}_{x\sim \mathcal{N}(0,I_d)}\left[f_i(x)h(x)\right]>\tau\right\} \text{ and }
        S_-:=\left\{i\in [K]\ |\ \mathbb{E}_{x\sim \mathcal{N}(0,I_d)}\left[f_i(x)h(x)\right]<-\tau\right\}.
    \end{align}
    Then, because $\|h\|\leq 1$, Cauchy-Schwarz inequality yields
    \begin{align}
        |S_+|^2\tau^2
        \!\leq \mathbb{E}\left[h(x)\sum_{i\in S_+}f_i(x)\right]^2
        \!\leq \mathbb{E}\left[\left(\sum_{i\in S_+}f_i(x)\right)^2\right]
        \!\lesssim |S_+| + \varepsilon(|S_+|^2-|S_+|).
    \end{align}
    Therefore, we have that
    \begin{align}
        |S_+| \leq \frac{1-\varepsilon}{\tau^2-\varepsilon}\leq \frac{1}{\tau^2-\varepsilon}.
    \end{align}
    The same argument applies to $|S_-|$.
\end{proof}
\begin{proofof}[\Cref{theorem:CSQ-Appendix}(a)]

Consider the number of queries $Q$, the sequence of query $\{g_1,\cdots,g_Q\}$, and tolerance $\tau$.

    According to Lemmas~\ref{lemma:NearOrthogonalBasis} and \ref{lemma:BoundCorrSet}, if
    \begin{align}\label{eq:CSQ-Usual-Condition}
        Q \cdot \frac{2d^C}{\tau^2-d^{-\frac{p}{2}}(2\log A)^{\frac{p}{2}}} \leq A, 
    \end{align}
    for some $C>0$, there exists at least $A(1-d^{-C})$ vectors $v\in S$ such that
    \begin{align}
        \left|\mathbb{E}\left[\frac{1}{\sqrt{p!}}\He_p(v^\top x) g_i(x)\right]\right|
        \leq \tau\quad (i=1,\cdots,Q).
    \end{align}
    Now we consider the minimum value of $\tau$ to satisfy \eqref{eq:CSQ-Usual-Condition}.
    If we take
    \begin{align}
         A\gtrsim Qd^{\frac{p}{4}+C}\text{ and }\tau \gtrsim d^{-\frac{p}{4}}(2C\log Qd)^{\frac{p}{4}},
    \end{align}
    we have
    \begin{align}
        (\text{LHS of \eqref{eq:CSQ-Usual-Condition}})
        \lesssim Q \cdot d^{\frac{p}{4}}(C\log Qd)^{-\frac{p}{4}}
        \lesssim A(C\log Qd)^{-\frac{p}{4}}
        \leq 
        A, 
    \end{align}
    which confirms \eqref{eq:CSQ-Usual-Condition}. 
    
    Now consider the width-$M$ additive model, for each query, we return the value of
    \begin{align}
        \mathbb{E}\left[\frac{1}{\sqrt{M}}\sum_{m=1}^{M-1}\frac{1}{\sqrt{p!}}\He_p(v_{m}^\top x) g_i(x)\right]
    \end{align}
    so that the learner cannot find the true $v_M$ among $A(1-d^{-C})$ possible directions, with probability at least $1-d^{-C}$.
    Failing to do so incurs an $L^2$-error of $\Omega\left(\frac{1}{M}\right)$.
    This is because for two sets of vectors $\{v_{m}\}_{m=1}^M$ and $\{\tilde{v}_{\tilde{m}}\}_{\tilde{m}=1}^M$ in the set $S$, we have
    \begin{align}
       & \mathbb{E}\left(\frac{1}{\sqrt{M}}\sum_{m=1}^{M}\frac{1}{\sqrt{p!}}\He_p(v_{m}^\top x) -\frac{1}{\sqrt{M}}\sum_{m=1}^{M}\frac{1}{\sqrt{p!}}\He_p(\tilde{v}_{m}^\top x) \right)^2
       \geq 2 - 2 \frac1M\sum_{m,\tilde{m}=1}^M(\tilde{v}_{\tilde{m}}^\top v_{m})^{p},
    \end{align}
    and 
     $\tilde{v}_{\tilde{m}}^\top v_{m}=1$ holds for at most $M-1$ pairs and $|\tilde{v}_{\tilde{m}}^\top v_{m}|\leq d^{-\frac{p}{2}}(2\log A)^{\frac{p}{2}}$ for the others if $\{v_{m}\}_{m=1}^M\ne \{\tilde{v}_{\tilde{m}}\}_{\tilde{m}=1}^M$. 
    Therefore, we conclude that if $\tau \gtrsim M^{-\frac12}d^{-\frac{p}{4}}(C\log Qd)^{\frac{p}{4}}$, the CSQ learner cannot achieve an $L^2$-error smaller than $O\left(\frac{1}{M}\right)$ with probability more than $d^{-C}$.
\end{proofof}

\subsection{Noisy CSQ and Proof of \Cref{theorem:CSQ-Appendix}(b)}
Now we prove the latter part of \Cref{theorem:CSQ-Appendix}.
We first explain why a lower bound with $\Omega(1)$ error cannot be achieved by naively extending the argument for \Cref{theorem:CSQ-Appendix}(a).

A naive argument would go as follows. 
Suppose we construct some $\hat{f}(x)=\frac{1}{\sqrt{M}}\sum_{m=1}^M \frac{1}{\sqrt{p!}}\He_p(\hat{v}_m^\top x)$, where $\{\hat{v}_1,\cdots,\hat{v}_M\}\subset S$ using queries with tolerance $\tau_0$.
Then $\|f_*(x)-\hat{f}(x)\|_{L^2} \leq 1$ entails that a constant fraction of $\{v_1,\cdots,v_M\}$ should be identified. 
Therefore, we may use such a CSQ learner to learn a single-index function $\frac{1}{\sqrt{M}}\frac{1}{\sqrt{p!}}\He_p(v_1^\top x)$ as follows. 
If we add $\frac{1}{\sqrt{M}}\sum_{m=2}^M\frac{1}{\sqrt{p!}}\He_p(v_m^\top x)$ and apply the CSQ learner for $\frac{1}{\sqrt{M}}\sum_{m=1}^M\frac{1}{\sqrt{p!}}\He_p(v_m^\top x)$, the learner would identify $\Omega(1)$-fraction of $\{v_1,\cdots,v_M\}$ with probability $\Omega(1)$.
On the other hand, according to \Cref{theorem:CSQ-Appendix}(a), learning $\frac{1}{\sqrt{M}}\frac{1}{\sqrt{p!}}\He_p(v_1^\top x)$ requires $\tau \lesssim M^{-\frac12}d^{-\frac{p}{4}}(C\log Qd)^{\frac{p}{4}}=\tilde{O}(M^{-\frac12}d^{-\frac{p}{4}})$, with high probability.
We may identify $v_1$ with high probability by repeating this process for $\tilde{O}(1)$ rounds; the CSQ lower bound for single-index model therefore implies that $\tau_0 \lesssim M^{-\frac12}d^{-\frac{p}{4}}$. 

The mistake in the above derivation is that, for the additive model $\frac{1}{\sqrt{M}}\sum_{m=1}^M\frac{1}{\sqrt{p!}}\He_p(v_m^\top x)$, since the target direction $v_1$ to be hidden is not known by the oracle beforehand, the (adversarial) oracle should prevent the identification of as many directions $v_1,v_2,...,v_M$ as possible; whereas in the single-index setting $\frac{1}{\sqrt{p!}}\He_p(v_1^\top x)$, the oracle only need to ``hide'' one direction. 
% , which is an easier task for the adversary. 
Consequently, we cannot directly connect the identification of $\Omega(1)$-fraction of target directions in the additive model setting to the CSQ lower bound for learning single-index model.

To overcome this issue, we introduce the following sub-class of CSQ algorithms with i.i.d.~noise.
Because the noise is not adversarial but random, 
the oracle for the single index model does not use the information of the target direction, 
% the individual single-index tasks are ``decoupled'', 
and hence the lower bound for single-index model now implies the failure of learning $\Omega(1)$ fraction of directions. 
On the other hand, since the noise is no longer adversarial, our lower bound in Theorem~\ref{theorem:CSQ-Appendix}(b) has weaker the query dependence compared to Theorem~\ref{theorem:CSQ-Appendix}(a). 
\begin{definition}[Noisy CSQ]
    For a function $g\colon \mathcal{X}\to \mathbb{R}$ and parameters $(\varsigma, \tau)$,
    noisy correlational statistical query oracle $\mathrm{NoisyCSQ}(g,\varsigma, \tau)$ returns 
    \begin{align}
        \mathbb{E}_{x,y}[yg(x)]+\nu,
    \end{align}
    where $\nu$ follows from the following clipped Gaussian distribution:
    \begin{align}
        \nu=\max\{-\tau,\min\{\tilde{\nu},\tau\}\},\quad \tilde{\nu} \sim \mathcal{N}(0,\varsigma^2).
    \end{align}
\end{definition}
The clipping operation matches the noisy CSQ with $(\varsigma,\tau)$ with the standard CSQ with a tolerance $\tau$.
The following theorem gives a lower bound for noisy CSQ algorithms to learn a single-index polynomial.
\begin{theorem}\label{theorem:Appendix-NoisyCSQ}
    For any $p\geq 0$, $\varsigma>0$, $1>\tau>0$, $Q>0$, $C\gg 0$ and $M=\tilde{o}(d^{\frac{p}{4}})$, consider learning $f(x)=\frac{1}{\sqrt{p!}}\He_p(v^\top x)$, where $v$ is sampled from some distribution over $\mathbb{S}^{d-1}$.
    Suppose that
    \begin{align}\label{eq:NoisyCSQConditionOnSigma}
        \varsigma \lesssim \frac{\tau}{\sqrt{\log Qd}}.
    \end{align}
    Then, for any learner using $Q$ noisy correlational queries $\mathrm{NoisyCSQ}(g,\varsigma, \tau)$, 
    the tolerance $\tau$ must satisfy
    \begin{align}
        \tau \lesssim Q^{\frac12} \sqrt{d^{-\frac{p}{2}}(\log dQ)^{\frac{p}{2}+2}}.
    \end{align}
    Otherwise, the learner
    cannot return $\hat{f}(x)=\frac{1}{\sqrt{p!}}\He_p({\hat{v}}^\top x)$ such that $\|f(x)-\hat{f}(x)\|_{L^2}\leq 1$ 
    with probability more than $O(d^{-C})$. 
\end{theorem}
\begin{proof} Due to the choice \eqref{eq:NoisyCSQConditionOnSigma}, we know that the clipping operation on the Gaussian noise does not make a difference with probability $1-d^{-C}$.
Thus in the following we simply consider that the pure Gaussian noise is added to the expectation.
We assume that the distribution where $v$ is sampled from is the uniform distribution over the set $S$ consisting of $A$ vectors, which is defined in Lemma~\ref{lemma:NearOrthogonalBasis}.
    % We later specify the value of $A$.
    Here $A$ is taken as
    \begin{align}
        A \simeq Qd^{\frac{p}{2}+2},
    \end{align}
    so that it satisfies
    \begin{align}\label{eqref:ConditionOnK-2}
          Q \cdot \frac{2d^{C}}{2d^{-\frac{p}{2}}(2\log A)^{\frac{p}{2}}-d^{-\frac{p}{2}}(2\log A)^{\frac{p}{2}}} \leq A.
    \end{align}

    As an intermediate claim, we show that if $\mathbb{E}[q_1(x)y],\cdots,\mathbb{E}[q_{i}(x)y]$ are bounded by $d^{-\frac{p}{2}}(2\log A)^{\frac{p}{2}}$, then $\mathbb{E}[q_{i+1}(x)y]$ is also bounded by $2d^{-\frac{p}{2}}(2\log A)^{\frac{p}{2}}$ with probability at least $1-O(Q^{-1}d^{-C})$.
    Assume that $\mathbb{E}[q_1(x)y],\cdots,\mathbb{E}[q_{i}(x)y]$ are bounded by $d^{-\frac{p}{2}}(2\log A)^{\frac{p}{2}}$.
    According to Lemma~\ref{lemma:BoundCorrSet}, for each query $q(x)$, there are at most $\frac{2}{d^{-\frac{p}{2}}(2\log A)^{\frac{p}{2}}}$ vectors that has correlation larger than $d^{-\frac{p}{2}}(2\log A)^{\frac{p}{2}}$.
    Thus, 
    there are at least $(1-d^{-C})\frac{2Qd^{C}}{d^{-\frac{p}{2}}(2\log A)^{\frac{p}{2}}}$ possible vectors that satisfy the assumption.
    Under this, 
    use Lemma~\ref{lemma:ImpossibleDistinguishi} with $a_{\rm max}=\sqrt{2d^{-\frac{p}{2}}(2\log M)^{\frac{p}{2}}}$ and $D=\Omega(\frac{Qd^{C}}{d^{-\frac{p}{2}}(\log A)^{\frac{p}{2}}})$.
    Then, when 
    \begin{align}\label{eq:varsigmaissmall}
        \varsigma \gtrsim Q^{\frac12} \sqrt{d^{-\frac{p}{2}}(\log dQ)^{\frac{p}{2}+1}},
    \end{align} 
    we cannot find the desired vector with probability more than $O(D^{-1})=O(\frac{Q^{-1}d^{-C}}{d^{\frac{p}{2}}(2\log A)^{-\frac{p}{2}}})$.
    This also implies that we cannot find any vector that satisfies $|\mathbb{E}[q_{i+1}(x)y]|>d^{-\frac{p}{2}}(2\log A)^{\frac{p}{2}}$ with probability more than $O(Q^{-1}d^{-C})$; this is because otherwise we can select one of $\frac{2}{2d^{-\frac{p}{2}}(2\log A)^{\frac{p}{2}}-d^{-\frac{p}{2}}(2\log A)^{\frac{p}{2}}}$ vectors that satisfy $|\mathbb{E}[q_{i+1}(x)y]|>d^{-\frac{p}{2}}(2\log A)^{\frac{p}{2}}$ and output as the prediction of the true vector, which would succeed with probability more than $O(\frac{Q^{-1}d^{-C}}{d^{\frac{p}{2}}(2\log A)^{-\frac{p}{2}}})$.

    Now, we obtain that, with probability at least  $1-O(d^{-C})$, $\mathbb{E}[q_1(x)y],\cdots,\mathbb{E}[q_{Q}(x)y]$ are all bounded by $d^{-\frac{p}{2}}(2\log A)^{\frac{p}{2}}$.
    Again, there are at least $(1-d^{-C})\frac{2Qd^{C}}{d^{-\frac{p}{2}}(2\log A)^{\frac{p}{2}}}$ possible vectors that satisfy the conditions $\mathbb{E}[q_1(x)y],\cdots,\mathbb{E}[q_{Q}(x)y]\leq d^{-\frac{p}{2}}(2\log A)^{\frac{p}{2}}$.
    Under this, 
    we apply Lemma~\ref{lemma:ImpossibleDistinguishi} with the same $a_{\mathrm{max}}$ and $D^{-1}$ as previously, and hence we cannot identify the right vector with probability more than $O(D^{-1})=O(\frac{Q^{-1}d^{-C}}{d^{\frac{p}{2}}(2\log A)^{-\frac{p}{2}}})\lesssim O(d^{-C})$.

    Therefore, we cannot return the correct vector in $A$ with probability more than $O(d^{-C})$. 
\end{proof}
\begin{lemma}\label{lemma:ImpossibleDistinguishi}
    Let $A_i = (a_{i,1},\cdots,a_{i,Q})\in \mathbb{R}^Q\ (i=1,\cdots,D)$ be sequences of $Q$ real values satisfying $|a_{i,j}|\leq a_{\rm max}$.
    Suppose that one of $A_i$ is uniformly randomly chosen an observation $(b_1,\cdots,b_Q)\in \mathbb{R}^Q$ is generated as
     \begin{align}
         b_j \sim \mathcal{N}(a_{i,j},\varsigma^2)\quad (j=1,\cdots,Q).
     \end{align}
     Then, if 
     \begin{align}\label{eq:ScaleOfNoise}
         \varsigma \gtrsim Q^{\frac12}a_{\rm max}\sqrt{\log D},
     \end{align} 
     any algorithm cannot identify which $A_i$ is selected with probability more than $1-O(D^{-1})$.
\end{lemma}
\begin{proof}
    % Consider the case when $K=2$.
    The optimal strategy is to calculate the likelihood function and select the one with which the index $i$ takes the largest value.
    Let the likelihood function of the $i$-th sequence be $p_i(b)$ for $B\in \mathbb{R}^Q$.
    We aim to bound the success probability by $O(D^{-1})$.
    \begin{align}
        \int \frac1D\max_{i}p_i(B) \mathrm{d}B\lesssim \frac1D.
        \label{eq::ImpossibleDistinguishi-1}
    \end{align}
    To simplify the discussion, we assume $\varsigma=1$ (because scaling does not affect whether the statement holds).
    % Under this, $a_{\rm max}$ satisfies $a_{\rm max}\leq \varsigma =1$.
    Then, \eqref{eq:ScaleOfNoise} implies $a_{\rm max}^2 \leq 2Q^{-1}$.
    We have
    \begin{align}
        \eqref{eq::ImpossibleDistinguishi-1}\times D
        & \leq  \frac{1}{(2\pi)^{\frac{Q}{2}}}\int\exp\left(-\frac{\|B\|^2}{2}+\max_{i}B^\top A_i\right)\mathrm{d}B.
        \label{eq::ImpossibleDistinguishi-2}
    \end{align}
    
    We bound the expectation of $\exp(\max_{i}B^\top A_i)$ conditioned on $\|B\|$.
    By Hoeffding's inequality, $\max_{i}B^\top A_i\leq a_{\rm max}\|B\|\sqrt{\log \delta^{-1}}$ with probability at least $1-\delta$.
    Thus, 
    \begin{align}
        \mathbb{E}_{B\sim \mathbb{S}^{Q-1}(\|B\|)}[\exp(\max_{i}B^\top A_i)]
       & \leq \sum_{i=1}^\infty \frac{1}{2^i}\exp\left(a_{\rm max}\|B\|\sqrt{\log (D2^i)}\right)
  \\ & \lesssim 
\int_{t=1}^{\infty}\exp\left(a_{\rm max}\|B\|\sqrt{\log (D2^t)}-t\log 2\right)\mathrm{d}t
  \\ & \lesssim 
e^{a_{\rm max}\|B\|\sqrt{\log D}}\int_{t=1}^{\infty}\exp\left(a_{\rm max}\|B\|\sqrt{t\log 2}-t\log 2\right)\mathrm{d}t
\\ & \lesssim 
a_{\rm max}\|B\|\exp\left(\frac{a_{\rm max}^2\|B\|^2}{4}+a_{\rm max}\|B\|\sqrt{\log D}\right) + 1.
\label{eq::ImpossibleDistinguishi-3}
    \end{align}
    Applying this to \eqref{eq::ImpossibleDistinguishi-2} yields
    \begin{align}
      \eqref{eq::ImpossibleDistinguishi-2}
      \lesssim \frac{1}{(2\pi)^{\frac{Q}{2}}} \int a_{\rm max}\|B\|\exp\left(-\frac{(1-Q^{-1})\|B\|^2}{2}+a_{\rm max}\|B\|\sqrt{\log D}\right)\mathrm{d}B + 1.
      \label{eq::ImpossibleDistinguishi-4}
    \end{align}
    Here we used $a_{\rm max}^2\leq 2Q^{-1}$.
    We bound the first term as follows:
    \begin{align}
        &\frac{1}{(2\pi)^{\frac{Q}{2}}} \int a_{\rm max}\|B\|\exp\left(-\frac{(1-Q^{-1})\|B\|^2}{2}+a_{\rm max}\|B\|\sqrt{\log D}\right)\mathrm{d}B
        \\ & = \frac{(1-Q^{-1})^{-\frac{Q+1}{2}}}{(2\pi)^{\frac{Q}{2}}} \int a_{\rm max}\|B'\|\exp\left(-\frac{\|B'\|^2}{2}+\frac{a_{\rm max}\|B'\|\sqrt{\log D}}{\sqrt{1-Q^{-1}}}\right)\mathrm{d}B'\quad (B'=(1-Q^{-1})^{\frac12}B)
        \\ &
        = \frac{2a_{\rm max}(1-Q^{-1})^{-\frac{Q+1}{2}}}{2^{\frac{Q}{2}}\Gamma(\frac{Q}{2})}\int_{s=0}^\infty e^{-\frac{s^2}{2}+\frac{a_{\rm max}s\sqrt{\log D}}{\sqrt{1-Q^{-1}}}}s^{Q}\mathrm{d}s\quad (s=\|B'\|)
        \\ & = \frac{a_{\rm max}2^{\frac{Q+1}{2}}(1-Q^{-1})^{-\frac{Q+1}{2}}}{2^{\frac{Q}{2}}\Gamma(\frac{Q}{2})}\Bigg[\frac{\sqrt{2}a_{\rm max}\sqrt{\log D}}{\sqrt{1-Q^{-1}}}\Gamma\left(\frac{Q}{2}+1\right){}_1F_1\left(\frac{Q}{2}+1;\frac32;\frac{a_{\rm max}^2\log D}{2(1-Q^{-1})}\right)\\ &\hspace{65mm}+ \Gamma\left(\frac{Q+1}{2}\right){}_1F_1\left(\frac{Q+1}{2};\frac12;\frac{a_{\rm max}^2\log D}{2(1-Q^{-1})}\right)\Bigg],
         \label{eq::ImpossibleDistinguishi-5}
    \end{align}
    where ${}_1F_1(x_1;x_2;x_3)$ is the confluent hypergeometric function of the first kind defined as
    \begin{align}
    {}_1F_1(x_1;x_2;x_3)=\sum_{n=0}^\infty \frac{x_1(x_1+1)\cdots(x_1+n-1)}{x_2(x_2+1)\cdots(x_2+n-1)}\frac{x_3^n}{n!}=\sum_{n=0}^\infty \frac{1}{n!}\prod_{i=0}^{n-1}\frac{(x_1+i)x_3}{x_2+i}.
    \end{align}

    We can evaluate ${}_1F_1\left(\frac{Q}{2}+1;\frac32;\frac{a_{\rm max}^2\log D}{2(1-Q^{-1})}\right)$ as
    \begin{align}
        {}_1F_1\left(\frac{Q}{2}+1:\frac32;\frac{a_{\rm max}^2\log D}{2(1-Q^{-1})}\right)=\sum_{n=0}^\infty \frac{1}{n!}\prod_{i=0}^{n-1}\frac{(\frac{Q}{2}+1+i)\frac{a_{\rm max}^2\log D}{2(1-Q^{-1})}}{\frac32+i}
        \lesssim 1,
    \end{align}
    if $(\frac{Q}{2}+1)\frac{a_{\rm max}^2\log D}{2(1-Q^{-1})} \leq \frac32 \Leftrightarrow 1\gtrsim  Qa_{\rm max}^2\log D$ holds.
    In the same way, we have ${}_1F_1\left(\frac{Q+1}{2};\frac12;\frac{a_{\rm max}^2\log D}{2(1-Q^{-1})}\right)\lesssim 1$ if $1\gtrsim  Qa_{\rm max}^2\log D$ holds.
    Also, $(1-Q^{-1})^{-\frac{Q+1}{2}}\lesssim 1$.
    Thus, we have
    \begin{align}
        \eqref{eq::ImpossibleDistinguishi-5} &\lesssim \frac{a_{\rm max}2^{\frac{Q+1}{2}}}{2^{\frac{Q}{2}}\Gamma(\frac{Q}{2})} \left[a_{\rm max}\sqrt{\log D}\Gamma\left(\frac{Q}{2}+1\right)+\Gamma\left(\frac{Q+1}{2}\right)\right]
        \\ & \leq \frac{a_{\rm max}2^{\frac{1}{2}}}{\Gamma(\frac{Q}{2})} \left[a_{\rm max}\sqrt{\log D}\Gamma\left(\frac{Q}{2}+1\right)+\Gamma\left(\frac{Q+1}{2}\right)\right]
        \\ & \leq a_{\rm max}^2\sqrt{2}\sqrt{\log D}\left(\frac{Q}{2}+1\right)+a_{\rm max}\sqrt{2}\left(\frac{Q}{2}\right)^\frac12
        \\ & \lesssim 1        
    \end{align}
    where we used $\Gamma(x+\frac12)=\int_{t=0}^\infty e^{-t}t^{x+\frac12}\mathrm{d}t\leq \left(\int_{t=0}^\infty e^{-t}t^{x}\mathrm{d}t\right)^\frac12\left(\int_{t=0}^\infty e^{-t}t^{x-1}\mathrm{d}t\right)^\frac12\leq (\Gamma(x+1))^\frac12(\Gamma(x))^\frac12= x^{\frac12}\Gamma(x)$ (by Hölder's inequality; this argument is borrowed from \citet{qi2010bounds}) and $a_{\rm max}Q^\frac12\sqrt{\log D}\leq 1$.
    Therefore, we have successfully obtained \eqref{eq::ImpossibleDistinguishi-1} and the assertion follows.

\end{proof}

\begin{proof}[Proof of \Cref{theorem:CSQ-Appendix}(b)]
Consider learning the following model
\begin{align}
    x\sim \mathcal{N}(0,I_d),\quad  y = f_*(x)=\frac{1}{\sqrt{M}}\sum_{m=1}^M \frac{1}{\sqrt{p!}}\He_p(v_m^\top x),
\end{align}
    where $\{v_1\cdots,v_M\}$ is a randomly sampled subset (without duplication) of the set $S$ used in the proof of Lemma~\ref{theorem:Appendix-NoisyCSQ}.
    Recall that Lemma~\ref{lemma:NearOrthogonalBasis} guarantees that when $M= \tilde{o}(d^{\frac{p}{4}})$ we have $|\mathbb{E}[y^2]-1| =o(1)$. % with probability at least $(1-|S|)^M/|S|! = 1-O(d^{-C})$.
    
    According to \Cref{theorem:Appendix-NoisyCSQ}, for any learner using $Q$ noisy correlational queries  with parameters $(\varsigma, \tau)$, to learn a univariate polynomial $\frac{1}{\sqrt{p!M}}\He_p(v^\top x)$, the tolerance must satisfy 
    \begin{align}
        \tau \lesssim \frac{Q^{\frac12} (\log dQ)^{\frac{p}{4}+1}}{M^{\frac12}d^{\frac{p}{4}}},
    \end{align}
    otherwise, the learning will fail with probability more than $1-O(d^{-C})$.

    If an algorithm learns $\mathcal{F}\subset\mathcal{F}_{d,M,\varsigma}^{p,q}$ with $Q$ noisy correlation queries and returns a function with $L^2$-error smaller than $1$, we know that the algorithm need to identify at least $\Omega(1)$-fraction of directions $\{v_1,\cdots,v_M\}$.
    If so, we can use such a learner to solve the single-index polynomials, by adding $M-1$ random functions $\frac{1}{\sqrt{p!}}\He_p(v_m^\top x)$ to the given single-index polynomial and then apply the algorithm. The lower bound therefore follows from the single-index CSQ lower bound stated in Theorem \Cref{theorem:Appendix-NoisyCSQ}.
\end{proof}

\bigskip

\section{Proof of SQ Lower Bound}
This section considers the SQ lower bound for $\mathcal{F}_{d,M,\varsigma}^{p,q}$.
The statistical query oracle is formally defined as follows, which covers the previous CSQ as a special case. 
\begin{definition}
    For a function $g\colon \mathcal{X}\times \mathcal{Y} \to \mathbb{R}$ and a tolerance $\tau>0$, 
    statistical query oracle $\mathrm{SQ}(g,\tau)$ returns any value in
    \begin{align}
        \left[\mathbb{E}_{x,y}[g(x,y)]-\tau, \mathbb{E}_{x,y}[g(x,y)]+\tau\right].
    \end{align}
\end{definition}
In the following, we assume bounded queries ${\rm SQ}\colon \mathbb{R}^d\times \mathbb{R}\to [-1,1]$. 

As mentioned in the main text, one motivation of our consideration of SQ learner is the existence of efficient SQ algorithms that can solve multi-index regression beyond the CSQ complexity. 
Specifically, the algorithm proposed in \citet{chen2020learning} learns single-index polynomials (i.e., the case of $K=1$) with sample complexity $\tilde{O}(d)$ for any constants $p,q=O_d(1)$; this result can also include the multi-index model up to $K=O(1)$, under a certain non-degeneracy condition. In the first stage of their algorithm, the labels $y$ are transformed so that the information exponent is reduced to $2$, which enables a warm start; after that, projected gradient descent exponentially converges to the relevant directions. 
Although \citet{chen2020learning} only considered the noiseless case (i.e., $\varsigma=0$), it is easy to extend their strategy to the noisy setting. Specifically, the warm-start algorithm can handle label noise with standard concentration arguments, and for the second stage, $\tilde{O}(d)$ sample complexity is also obtained despite the loss of exponential convergence. 

However, Theorem~\ref{theorem:SQ-Appendix} suggests that such linear-in-$d$ complexity is no longer feasible for SQ learners to learn our additive model class. Specifically, our lower bound implies that for $M\asymp d^\gamma$ with $\gamma>0$, the exponent in the dimension dependence can be made arbitrarily large by varying $p,q=O_d(1)$. This highlights the fundamental computational hardness of the larger $M$ setting.

\subsection{Superorthogonal Polynomials}

\Cref{theorem:SQ-Appendix} relies on the existence of \textit{superorthogonal} polynomials defined in Lemma~\ref{prop:superorthogonality}.  
Recall that superorthogonality means that a polynomial and its $2$-$,\cdots,K$-th exponentiations are orthogonal, in a sense of inner product with respect to the standard Gaussian, to any polynomials up to degree $L$ whose expectation is $0$. 
For $K=1$, $f(x)=\He_{L+1}(x)$ satisfies the condition.
However, for $K\geq 2$, it is far from trivial that such a function exists. 
We defer the proof of Proposition~\ref{prop:superorthogonality} to \Cref{subsection:Reparameterization,subsection:PolynomialApproximation}, and proceed to explain how Proposition~\ref{prop:superorthogonality} is used in the SQ lower bound.

We utilize the following fact that the expectation $\mathbb{E}[q(x,y)]$ can be Taylor expanded with respect to a small perturbation of $y$, due to the Gaussian noise added to $y$.
\begin{lemma}\label{lemma:NoiseSmoothing}
    Suppose that $|g(x,y)|\leq 1$ for any $(x,y)\in \mathbb{R}^d\times \mathbb{R}$.
    Then, for $\delta \ll 1$, we have
    \begin{align}
        \mathbb{E}_{\eps\sim\mathcal{N}(0,\varsigma^2)}[g(x,z+\delta+\varepsilon)]
        =
        \sum_{k=0}^K a_{k}(x,z) \delta^k + O(\delta^{K+1}),
    \end{align}
    where $a_{k}(x,z)=\frac{1}{k!}\int g(x,w) \left(\frac{\mathrm{d}^k}{\mathrm{d} z^k} e^{-(w-z)^2/2\varsigma^2}\right)\mathrm{d}w=O(1)$.
\end{lemma}
\begin{proof}
    The proof follows from the change-of-variables in integration.
    Specifically, by letting $w=z+\delta+\varepsilon$,  
    \begin{align}
        &\mathbb{E}_{\eps\sim\mathcal{N}(0,\varsigma^2)}[g(x,z+\delta+\varepsilon)]
       \\ &= \int g(x,z+\delta+\varepsilon)
        e^{-\varepsilon^2/2\varsigma^2}\mathrm{d}\varepsilon
        \\ & =\int g(x,w)
        e^{-(w-z-\delta)^2/2\varsigma^2}\mathrm{d}w
        \\ & =\sum_{k=0}^K \underbrace{\frac{1}{k!}\frac{\mathrm{d}^k}{\mathrm{d}z^k}\int g(x,w)
        e^{-(w-z)^2/2\varsigma^2}\mathrm{d}w}_{=:a_k(x,z)} \cdot \delta^i
        \\ &\quad +
        \left.\frac{1}{(K+1)!}\frac{\mathrm{d}^{K+1}}{\mathrm{d}z^{K+1}}\int g(x,w)
        e^{-(w-z)^2/2\varsigma^2}\mathrm{d}w\right|_{z=z'} \cdot \delta^{K+1}
        \ (\text{for some $z'$})
        \\ & = \sum_{k=0}^K\int g(x,w) a_k(x,z)\delta^k+O(\delta^{K+1}).
    \end{align}
    Note that each $a_i(x,z)$ is $O(1)$ because $|g(x,w)|\leq 1$.
\end{proof}
% We now prove Theorem~\ref{theorem:SQ-Appendix}.
\begin{proofof}[\Cref{theorem:SQ-Appendix}]
    Suppose for sake of contradiction $\tau\gtrsim d^{\rho}$.
    Let us take $A=e^{\sqrt{d}}$ in Lemma~\ref{lemma:NearOrthogonalBasis}.
    Then, we have a set of unit vectors $S\subseteq \mathbb{S}^{d-1}$ such that two distinct vectors have an inner product at most $O(d^{-\frac14})$.
    To construct the target function, we randomly draw $\{v_m\}_{m=1}^M$ from $S$ and let $f_m= f$ for $m=1,\cdots,M$, where $f$ is a superorthogonal polynomial from Proposition~\ref{prop:superorthogonality} with $K\geq \frac{2\rho}{\gamma}+2$ and $L\geq 4(\gamma+\rho)-1$. 
    In addition, we construct a different target function $\frac{1}{\sqrt{M}}\sum_{i=1}^M f({v_m'}^\top x)$ with $\{v'_m\}_{m=1}^M$ in the same fashion.
    % We can ignore the probability when two of $\{v_m\}_{m=1}^M$ (or two of $\{v'_m\}_{m=1}^M$) are the same vector (otherwise expectation of $y$ is $O(1)$).
    
    For each $m$, we prove that the following holds for at least $1-\tilde{O}(e^{-\sqrt{d}}(M^2d^{2\rho}))$ fraction of random choices of $v_m,v'_m$: 
    \begin{align}
       & \mathbb{E}\left[g\left(x,\frac{1}{\sqrt{M}}\sum_{m'=1}^{m-1}f(v_{m'}^\top x)+\frac{1}{\sqrt{M}}\sum_{m'=m}^{M}f({v_{m'}'}^\top x)+\varepsilon\right)\right]
      \\  = &
        \mathbb{E}\left[g\left(x,\frac{1}{\sqrt{M}}\sum_{m'=1}^{m}f(v_{m'}^\top x)+\frac{1}{\sqrt{M}}\sum_{m'=m+1}^{M}f({v_{m'}'}^\top x)+\varepsilon\right)\right]+ o(M^{-1}d^{-\rho}).
        \label{eq:SQ-AppendiX-OneStepTaylor-10}
    \end{align}
    This is to say, when $M$ is large, swapping one single-index task results in small change in the query value. 
    To see this, from Lemma~\ref{lemma:NoiseSmoothing}, we have
    \begin{align}
        &\mathbb{E}_{\eps}\left[q\left(x,\frac{1}{\sqrt{M}}\sum_{m'=1}^{m-1}f({v_{m'}'}^\top x)+\frac{1}{\sqrt{M}}\sum_{m'=m}^{M}f(v_{m'}^\top x)+\varepsilon\right)\right]
        \\ &=
        \mathbb{E}_{\eps}\left[q\left(x,\frac{1}{\sqrt{M}}\sum_{m'=1}^{m-1}f({v_{m'}'}^\top x)+\frac{1}{\sqrt{M}}\sum_{m'=m+1}^{M}f(v_{m'}^\top x)+\varepsilon\right)\right]
       \\ & \label{eq:SQ-AppendiX-OneStepTaylor} \quad  +
        \sum_{k=1}^K a_{k}\left(x,\frac{1}{\sqrt{M}}\sum_{m'=1}^{m-1}f({v_{m'}'}^\top x)+\frac{1}{\sqrt{M}}\sum_{m'=m+1}^{M}f(v_{m'}^\top x)\right) \frac{f^k(v_m^\top x)}{M^{\frac{i}{2}}} + o(M^{-\frac{K}{2}}).
    \end{align}
    Note that $|\int f(v_i^\top x)f(v_j^\top x)e^{-\|x\|^2/2}\mathrm{d}x| \lesssim  d^{-(L+1)/4} \lesssim M^{-1}d^{-\rho}$ if $v_i\ne v_j$.
    Now from Lemma~\ref{lemma:BoundCorrSet}, for each $k$, the number of $v_m$ that satisfy
    \begin{align}
        \left|\mathbb{E}\left[ a_{k}\left(x,\frac{1}{\sqrt{M}}\sum_{m'=1}^{m-1}f({v_{m'}'}^\top x)+\frac{1}{\sqrt{M}}\sum_{m'=m+1}^{M}f(v_{m'}^\top x)\right) \frac{f^k(v_m^\top x)}{M^{\frac{i}{2}}}\right]\right| 
        \geq 2M^{-1}d^{-\rho},
    \end{align}
    is at most $O(M^2d^{2\rho})$.
    Thus, except for $O(KM^2d^{2\rho})$ choices of $v_m$, taking expectation of \eqref{eq:SQ-AppendiX-OneStepTaylor} yields
    \begin{align}
        \eqref{eq:SQ-AppendiX-OneStepTaylor}
        &= \mathbb{E}_{\eps}\left[g\left(x,\frac{1}{\sqrt{M}}\sum_{m'=1}^{m-1}f({v_{m'}'}^\top x)+\frac{1}{\sqrt{M}}\sum_{m'=m+1}^{M}f(v_{m'}^\top x)+\varepsilon\right)\right]
        \\ &\quad + \sum_{k=1}^K\mathbb{E}\left[ a_{k}\left(x,\frac{1}{\sqrt{M}}\sum_{m'=1}^{m-1}f({v_{m'}'}^\top x)+\frac{1}{\sqrt{M}}\sum_{m'=m+1}^{M}f(v_{m'}^\top x)\right)\right]\mathbb{E} \frac{\left[f^k(v_m^\top x)\right]}{M^{\frac{i}{2}}} + o(M^{-\frac{K}{2}}).
        \label{eq:SQ-AppendiX-OneStepTaylor-2}
    \end{align}
    Thus in similar fashion, 
    \begin{align}
        \eqref{eq:SQ-AppendiX-OneStepTaylor-2}
        = \mathbb{E}\left[g\left(x,\frac{1}{\sqrt{M}}\sum_{m'=1}^{m}f(v_{m'}^\top x)+\frac{1}{\sqrt{M}}\sum_{m'=m+1}^{M}f({v_{m'}'}^\top x)+\varepsilon\right)\right]+ o(M^{-1}d^{-\rho}),
    \end{align}
    which yields
    \eqref{eq:SQ-AppendiX-OneStepTaylor-10}. 
    By recursively applying \eqref{eq:SQ-AppendiX-OneStepTaylor-10}, we have
    \begin{align}
        \mathbb{E}\left[q\left(x,\frac{1}{\sqrt{M}}\sum_{m=1}^{M}f(v_m^\top x)+\varepsilon\right)\right]
        =
       \mathbb{E}\left[q\left(x,\frac{1}{\sqrt{M}}\sum_{m=1}^{M}f({v_m'}^\top x)+\varepsilon\right)\right]+ o(d^{-\rho})
    \end{align}
    for $(e^{\sqrt{d}}-\tilde{O}(M^2d^{2\rho}))^M\geq e^{\sqrt{d}M}(1-e^{-\Omega(\sqrt{d})})$ choices of $v_1,\cdots,v_M$.

    Therefore, we may return the value of $\mathbb{E}\left[g(x,\frac{1}{\sqrt{M}}\sum_{m=1}^{M}f(\bar{v}_m^\top x)+\varepsilon)\right]$ for a specific choice of $\{\bar{v}_m\}_{m=1}^M$ fixed \textit{a priori}, so that each query only removes at most $e^{-\Omega(\sqrt{d})}$ fraction of the possible choices of $v_1,\cdots,v_M$, but gives no information about the remaining directions. 
    Thus, with probability at least $1-Qe^{-\Omega(\sqrt{d})}$, $e^{\sqrt{d}M}(1-Qe^{-\Omega(\sqrt{d})})$ possible choices of $v_1,\cdots,v_M$ are equally likely.
    On the other hand, for each choice of $\frac{1}{\sqrt{M}}\sum_{m=1}^{M}f(v_m^\top x)$, there are at most $e^{\sqrt{d}M/2}$ possible choices of $f({v_1'}^\top),\cdots,f({v_M'}^\top)$ if the $L^2$-error between $\frac{1}{\sqrt{M}}\sum_{m=1}^{M}f(v_m^\top x)$ and $\frac{1}{\sqrt{M}}\sum_{m=1}^{M}f({v_m'}^\top x)$ is less than $1$; so in order to output a function with small $O(1)$ error, we need to isolate one of the $O(e^{\sqrt{d}M})$ possible hypotheses. 
    This completes the proof of \Cref{theorem:SQ-Appendix}.
\end{proofof}

\subsection{Reparameterization of Polynomials}
% \subsection{Polynomial Construction}%Construct an Auxiliary Class with a Sum of Indicator Functions}
\label{subsection:Reparameterization}
The proof of Proposition~\ref{prop:superorthogonality} requires several new techniques. 
We begin by introducing an auxiliary class of polynomials $\{h_a(x)\}$ parameterized by $a=(a_{k,l})\in [-1,1]^{L\times K}$ that satisfies the three properties below.
This is to avoid adjusting coefficients of $f(x)$ directly, because solving $\int (f(x))^k \He_l(x) e^{-x^2/2}\mathrm{d}x=0\ (1\leq k\leq K, 1\leq l\leq L)$ as simultaneous high-order equations of the coefficients would be difficult.
\begin{itemize}
    % \vspace{-1.5mm}
    \item[(P1)] $h_a(x) \not\equiv 0$.

    % \vspace{-2.5mm}
    \item[(P2)] For every $1\leq k\leq K,1\leq l\leq L$, $\int (h_a(x))^k \He_l(x) e^{-x^2/2}\mathrm{d}x$ is continuous with respect to $a$.

    % \vspace{-2.5mm}
    \item[(P3)] 
    $\int (h_a(x))^k \He_l(x) e^{-x^2/2}\mathrm{d}x>0$ holds if $a_{k,l}=1$, and $\int (h_a(x))^k \He_l(x) e^{-x^2/2}\mathrm{d}x<0$ holds if $a_{k,l}=-1$.
    % \vspace{-2.5mm}
\end{itemize}
The following lemma shows that these three properties entails the existence of a desired superorthogonal polynomial. 

\begin{lemma}\label{lemma:ExistenceOfSuperOrthogonalPoly}
    If $\{h_a(x)\}$ satisfies (P1)-(P3), there exists some coefficient $a$ such that
    $\int (h_a(x))^k \He_l(x) e^{-x^2/2}\mathrm{d}x = 0$ holds for every $1\leq k\leq K$ and $1\leq l\leq L$ but $\int (h_a(x))^{2} e^{-x^2/2}\mathrm{d}x > 0$.
\end{lemma}

\begin{figure}[t]
\begin{center}
% \centerline{\includegraphics[width=\columnwidth]{figure/LowerBound.pdf}}
\centerline{\includegraphics[width=\columnwidth]{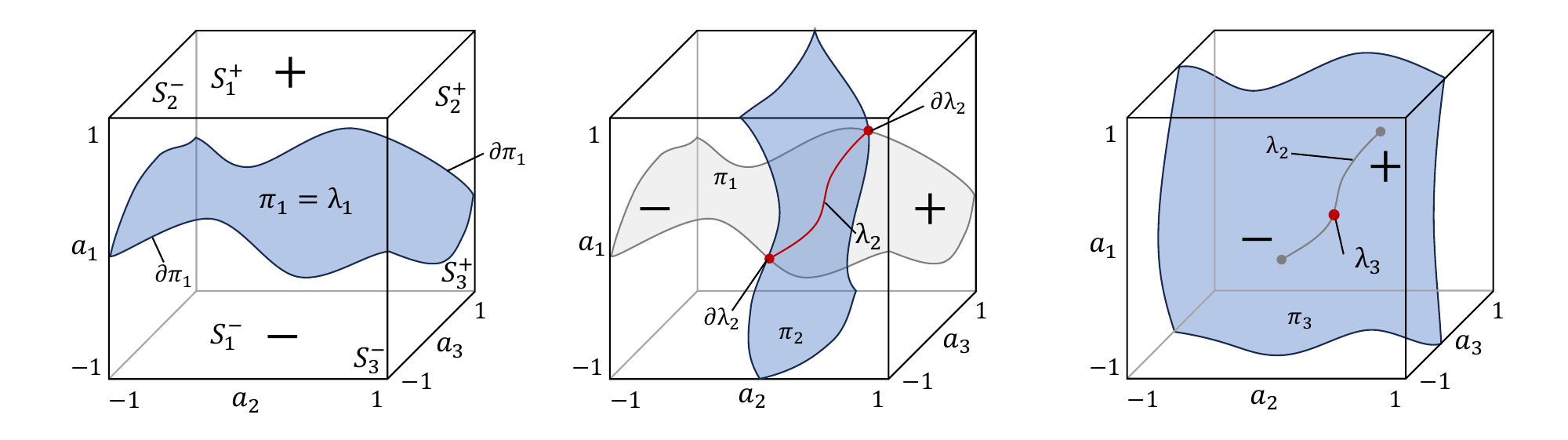}}
\caption{\small 
Illustration of the proof for $I=3$.
For each $i=1,2,3$, a $2$-dimensional curved surface $\pi_{i}$ on which $A_i(a)=0$ divides the hypercube.
First, we take $\lambda_1=\pi_1$.
Then, we take the intersection between $\lambda_{1}$ and $\pi_{2}$, which is a curved line and connects one of its boundary on $S_3^+$ and the other in $S_3^-$.
Finally, we consider the intersection of $\lambda_{2}$ and $\pi_{3}$.
Because $\lambda_2$ connects the points in $S_3^+$ and $S_3^-$ while $\pi_3$ divides the hypercube into the part containing $S_3^+$ and the one containing $S_3^-$, $\lambda_3=\lambda_2\cap \pi_3$ is not an empty set and $A_1(a)=A_2(a)=A_3(a)=0$ holds on $\lambda_3$.
}
\label{fig:3DCubeIllustration}
\end{center}
\vspace{-3mm}
\end{figure}

\begin{proof} 
    As handling two subscripts $k,l$ can be notation-heavy, we prove the following restated claim.
    For $I\in \N$, we consider a vector $a\in [-1,1]^I$ and functions $A_{i}\colon \R^I \to \R$, satisfying
\begin{itemize}
    \item[(P2)'] For every $i$, $A_i(a)$ is continuous with respect to $a$, and
    \item[(P3)'] 
    $A_i(a)>0$ holds if $a_i=1$, and  $A_i(a)<0$ holds if $a_i=-1$,
    % \vspace{-2.5mm}
\end{itemize}  
and prove that there exists some $a\in [-1,1]^I$ such that $A_i(a)=0$ for all $i$.
Regarding as $(k,l)$ and $A_i$ as $\int (h_a(x))^k \He_l(x) e^{-x^2/2}\mathrm{d}x$, this is equivalent to the assertion in the above lemma.

We name each surface of the hypercube $[-1,1]^I$ as follows: 
the surface with $a_i=1$ is denoted as $S_i^+$ and the surface with $a_i=-1$ as $S_i^-$.
Because of (P2)' and (P3)', the hypercube $[-1,1]^I$ is divided into two parts by
the curved surface $\pi_{i}\subseteq [-1,1]^{I}$, on which $A_i(a)=0$ holds, and the surface is homeomorphic to $[-1,1]^{I-1}$.
Clearly, one part contains $S_i^+$ and the other contains $S_i^-$.

Now, we inductively see that there exists $\lambda_i\subset [-1,1]^I$ homeomorphic to $[-1,1]^{I-i}$ on which $A_j(a)=0$ holds.
This is true for $i=1$, by taking $\lambda_1=\pi_1$.
When this is true for $i$, and the boundary of $\lambda_i$ is on all surfaces of the hypercube except for $S_1^+,\cdots,S_i^+$ and $S_1^-,\cdots,S_i^-$, we can take $\lambda_{i+1}\subseteq \lambda_{i} \cap \pi_{i+1}$, on which $A_j(a)=0$ holds for $j=1,\cdots,i+1$, which is homeomorphic to $[-1,1]^{I-(i+1)}$, and the boundary of which is on all surfaces of the hypercube except for $S_1^+,\cdots,S_{i+1}^+$ and $S_1^-,\cdots,S_{i+1}^-$.
Therefore by induction, we obtain $\lambda_I\subseteq \lambda_{I-1} \cap \pi_{I}$, which contains at least one point and on which $A_j(a)=0$ holds for $j=1,\cdots,I$. See Figure~\ref{fig:3DCubeIllustration} for illustration of the case where $I=3$. 
\end{proof}

% \paragraph{Construction of auxiliary functions.}
Our next goal is to construct $\{h_a(x)\}$ that satisfy the above properties.
(P1) and (P2) are easily checked in the following construction.
To meet (P3), we introduce an auxiliary class of functions $\{h_a^*(x)\}$ and will later approximate it using polynomials in Section~\ref{subsection:PolynomialApproximation}.

We first provide a high-level sketch on the construction of the auxiliary class $\{h_a^*(x)\}$ using \Cref{fig:PiecewiseConstant}.
In particular, we need to adjust the following value to satisfy (P3)
\begin{align}
    \int (f(x))^k \He_l(x) e^{-x^2/2}\mathrm{d}x.
    \label{eq:FixIConsiderIntegral}
\end{align}

First, we fix the exponent $k$ and consider to make the value \eqref{eq:FixIConsiderIntegral} all $0$ for $l=1,2,\cdots,L$.
Directly adjusting coefficients of $f(x)$ by regarding \eqref{eq:FixIConsiderIntegral} as  simultaneous high-order polynomials of the coefficients of $f(x)$ would be difficult.
Instead, we re-parameterize the problem as an almost linear simultaneous equation with respect to the new parameters. 

Specifically, we consider a piecewise linear function defined as follows.
First, we focus on \eqref{eq:FixIConsiderIntegral} for a specific $k$ by considering a
$\{0,1\}$-valued function (See Figure~3(a)).
We divide (an interval of) the real line into equal intervals, which are indexed by $i_1$ and $i_2$, with width $\varepsilon$, and for each interval, we assign a value of $1$ to the left half and assign the remaining portion a value of $0$.
Then, we consider the value of \eqref{eq:FixIConsiderIntegral}, which is the expectation of a multiplication of this function and a Hermite polynomial $\He_l(x)$. 
If the interval gets small, this is approximately equal to taking expectations of a multiplication between a constant function and the Hermite polynomial.
In other words, the value of \eqref{eq:FixIConsiderIntegral} gets closer to $0$ as the interval gets small smaller.

To modify the integral value for a specific $l$, we move the right end of each interval beginning from $x_i$, proportionally to the value of $\He_l(x_i)$.
If we move each right ends $O(a_l\varepsilon \He_l(x_i))$ ($a_{l}$ is a scalar), then \eqref{eq:FixIConsiderIntegral} for the specific $l$ changes almost \textit{linearly} with respect to $a_l$, while \eqref{eq:FixIConsiderIntegral} for the other $l$ remains (approximately) the same.
In this way, we fine-tune the value of \eqref{eq:FixIConsiderIntegral} for each $l$ separately by changing the parameter $a_l$.

Next, we consider how to simultaneously address different exponentiations $k=1,\dots,K$.
We divide each interval of width $\varepsilon$ into $K$ different sub-intervals indexed by $1,\dots,K$ (see Figure~3(b)).
For the $j$-th sub-interval, we let the height of the indicator function to be $\frac{j}{K}$.
If we change the length of the $j$-th sub-interval proportionally to $\varepsilon \alpha_{i_1,i_2,j}$ $(j=1,2,\cdots,K)$, this is approximately equivalent to setting the value of each $(f(x))^k$ around $x_{i_1,i_2}$ proportionally to
\begin{align}
\label{eq:HowMuchChangeAffects}
    \left[V \cdot \begin{pmatrix}
        \alpha_{i_1,i_2,1}
        \\
        \alpha_{i_1,i_2,2}
        \\
        \vdots
        \\
        \alpha_{i_1,i_2,K}
    \end{pmatrix}\right]_k,
\end{align}
where $V\in \mathbb{R}^{K\times K}$ is a variant of the Vandermonde matrix defined as follows:
\begin{align}
    V := \begin{bmatrix}
        \frac{1}{K} & \frac{2}{K} & \cdots & 1 \\
        (\frac{1}{K})^2 & (\frac{2}{K})^2 & \cdots & 1 \\
        \vdots &\vdots &   & \vdots\\
        (\frac{1}{K})^K & (\frac{2}{K})^K & \cdots & 1 
    \end{bmatrix}.
\end{align}

The following lemma implies that we can change each coordinate of $V\cdot (
        \alpha_{i_1,i_2,1}
        ,
        \alpha_{i_1,i_2,2}
        ,
        \vdots
        ,
        \alpha_{i_1,i_2,K}
   )^\top$ in \eqref{eq:HowMuchChangeAffects} arbitrarily by adjusting the values of $\alpha_{i_1,i_2,1},\cdots,\alpha_{i_1,i_2,K}$.
   In this way, we can control the contribution of $(f(x))^k$ to the integral values separately for different $k$ at each interval.

\begin{figure}[t]
\begin{center}
\centerline{\includegraphics[width=\columnwidth]{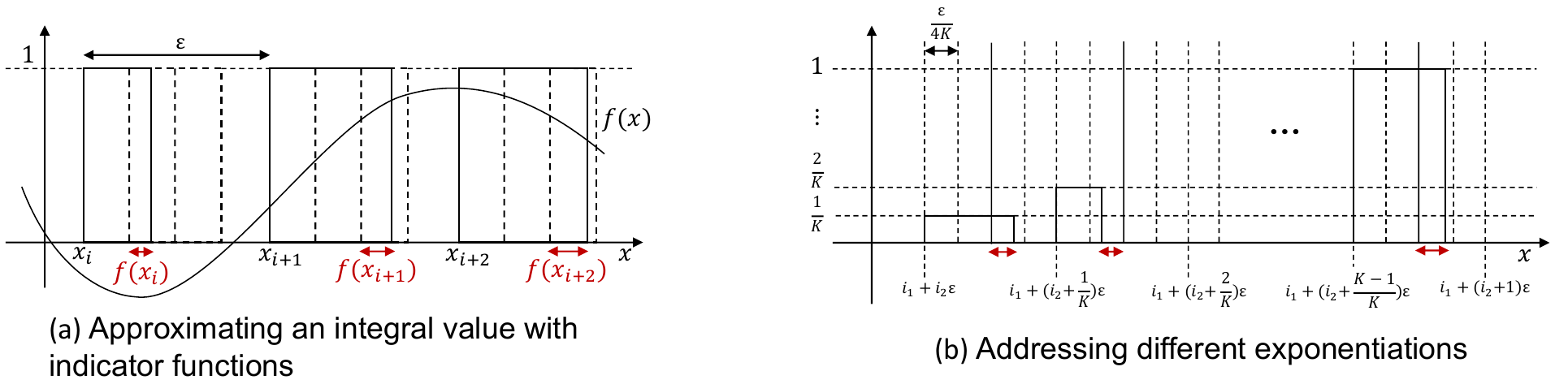}}
    \caption{\small  Approximation via piecewise constant function.
    Figure~3(a): 
    shifting the right end of each indicator function proportionally to $f(x_i)$ is approximately equivalent to subtracting $O(f(x))$ from $(h^*_a(x))^k$ in the sense of integral value.
    Figure~3(b): By considering the staircase function, we can simultaneously modify the contribution of the different exponents of $h^*_a(x)$ to the integral.}
    \label{fig:PiecewiseConstant}
\end{center}
\vspace{-3mm}
\end{figure} 
   
Putting it all together, we can adjust the integral values of \eqref{eq:FixIConsiderIntegral} separately for each $k$ and $l$ to satisfy (P3).
After obtaining a function class $\{h^*_a(x)\}$ sketched above, approximating indicator functions by  polynomials yields the desired class $\{h_a(x)\}$ that satisfies (P1)-(P3) in Section~\ref{subsection:PolynomialApproximation}. 
\begin{lemma}\label{lemma:RegularityOfVandermonde}
    The matrix $V$ is invertible.
\end{lemma}
\begin{proof}
    Let 
    \begin{align}\label{eq:RegularityOfVandermonde-1}
        \tilde{V}_x := \begin{bmatrix}
        1 & 1 & \cdots & 1 \\
        x_1 & x_2 & \cdots & x_{K+1} \\
        (x_1)^2 & (x_2)^2 & \cdots & (x_{K+1})^2 \\
        \vdots &\vdots &   & \vdots\\
        (x_1)^K & (x_2)^K & \cdots & (x_{K+1})^K 
    \end{bmatrix}.
    \end{align}
    Because of the formula for the determinant of the block matrix, we have
    \begin{align}
        \mathrm{det}(\tilde{V}_x)
        % &
        = (-1)^n \cdot \mathrm{det}\left(\begin{bmatrix}
        1 & 1 & \cdots & 1 \\
        x_1 & x_2 & \cdots & x_{K} \\
        (x_1)^2 & (x_2)^2 & \cdots & (x_{K})^2 \\
        \vdots &\vdots &   & \vdots\\
        (x_1)^K & (x_2)^K & \cdots & (x_{K})^K 
    \end{bmatrix}- 1 \cdot \begin{pmatrix}x_{K+1}\\\vdots\\x_{K+1}^{K}\end{pmatrix}(1,\cdots,1)\right).
   % \\ & = -\mathrm{det}\left(\begin{bmatrix}
   %      x_1 & x_2 & \cdots & x_{K} \\
   %      (x_1)^2 & (x_2)^2 & \cdots & (x_{K})^2 \\
   %      \vdots &\vdots &   & \vdots\\
   %      (x_1)^K & (x_2)^K & \cdots & (x_{K})^K 
   %  \end{bmatrix}\right).
    \end{align}
    It is a well-known fact on the Vandermonde matrix that $\mathrm{det}(\tilde{V}_x)\ne 0$ when $x_i\ne x_j$ for all $i\ne j$.
    Let $x_1=\frac{1}{K},x_2=\frac{2}{K},\cdots,x_K=1,x_{K+1}=0$.
    Then, the LHS of \eqref{eq:RegularityOfVandermonde-1} is $\mathrm{det}(\tilde{V}_x)\ne 0$ and the RHS of \eqref{eq:RegularityOfVandermonde-1} is equal to $(-1)^n\mathrm{det}(V)$. Therefore, we have obtained that $\mathrm{det}(V)\ne 0$.
\end{proof}

Now we formalize the above proof sketch.
We let 
\begin{align}
    M_1 = \max_{a\in [-1,1]^K}\left\|V^{-1} a\right\|_\infty
\end{align}
and define the function class $\{h^*_a(x)\}$.

\begin{definition}[An auxiliary class $\{h_a^*(x)\}$]\label{definition:HStar}
    Let $\mathbbm{1}_{s,t}(x)\ (s<t)$ be an indicator function satisfying $\mathbbm{1}_{s,t}(x)=1$ for $s\leq x\leq t$ and $=0$ otherwise.
    Fix $K,L,A_1$, and $A_2$, and let $\varepsilon := \frac{1}{A_2}$.
    We define a class of functions $\{h_\alpha^\dagger(x)\}$, parameterized by $\alpha=(\alpha_{i_1,i_2,j})\in [-1,1]^{2A_1 \times A_2 \times K}$, 
    \begin{align}\label{eq:DefinitionOfHDagger}
        h_\alpha^\dagger(x) := \sum_{i_1=-A_1}^{A_1-1}\sum_{i_2=0}^{A_2-1}\sum_{j=1}^K \frac{j}{K}\cdot \mathbbm{1}_{i_1+i_2\varepsilon + \frac{4(j-1)}{4K}\varepsilon,i_1+i_2\varepsilon + \frac{4(j-1)+2+\alpha_{i_1,i_2,j}}{4K}\varepsilon}(x).
    \end{align}
    Then, we construct a map from $a\in \mathbb{R}^{K\times L}$ to $\alpha$ as follows. For each $i_1,i_2$, we define $(\alpha(a)_{i_1,i_2,j})_{j=1}^K$ as
    \begin{align}
    \begin{pmatrix}
    \alpha(a)_{i_1,i_2,1}\\
    \alpha(a)_{i_1,i_2,2}\\ 
    \vdots \\ 
    \alpha(a)_{i_1,i_2,K}
    \end{pmatrix}
    := \frac{1}{M_1M_2}V^{-1}
    \begin{pmatrix}
    \sum_{l=1}^L a_{1,l}\He_{l}(i_1+i_2\varepsilon)\\
    \sum_{l=1}^L a_{2,l}\He_{l}(i_1+i_2\varepsilon)\\
    \vdots\\
    \sum_{l=1}^L a_{K,l}\He_{l}(i_1+i_2\varepsilon)
    \end{pmatrix},
    \end{align}
   where $M_2:=\max_{1\leq l \leq L} \max_{-A_1\leq x\leq A_1}|\He_l(x)|$.
    Based on this, we define $h^*_a(x)$ by
    \begin{align}
        h^*_a(x) := h^\dagger_{\alpha(a)}(x)
    \end{align}
    From the definitions of $M_1$, $M_2$, we have $\|\alpha(a)\|_\infty\leq 1$.
    Thus in \eqref{eq:DefinitionOfHDagger}, each interval of the indicator function does not overlap with the others and the right end is contained in $[i_1+i_2\varepsilon + \frac{4(j-1)}{4K}\varepsilon,i_1+i_2\varepsilon + \frac{4(j-1)+2}{4K}\varepsilon].$
\end{definition}

\begin{lemma}\label{lemma:IndicatorP3}
    There exist constants $A_1$ and $A_2$ such that
    $h^*(a)$ defined in \Cref{definition:HStar} satisfy the property (P3). Specifically, when $a_{i,j}=1$, we have
    \begin{align}\label{eq:P3ForHStar-L1-1}
        \int (h^*_a(x))^i \He_j(x) e^{-x^2/2}\mathrm{d}x > 0,
    \end{align}
    and when $a_{i,j}=-1$, we have
    \begin{align}\label{eq:P3ForHStar-L1-2}
        \int (h^*_a(x))^i \He_j(x) e^{-x^2/2}\mathrm{d}x < 0,
    \end{align}
    for all $1\leq i\leq L$ and $1\leq j\leq K$.
\end{lemma}
\begin{proof}
    % We only prove \eqref{eq:P3ForHStar-L1-1}, since the proof for \eqref{eq:P3ForHStar-L1-2} is essentially the same.
    First, we decompose the integral as follows:
        \begin{align}
        \int (h^*_a (x))^k \He_l e^{-x^2/2}\mathrm{d}x
       &=\sum_{i_1=-A_1}^{A_1-1}\sum_{i_2=0}^{A_2-1}\sum_{j=1}^K \left(\frac{j}{K}\right)^k\int_{i_1+i_2\varepsilon+\frac{4(j-1)\varepsilon}{4K}}^{i_1+i_2\varepsilon+\frac{(4(j-1)+2+\alpha_{i_1,i_2,j})\varepsilon}{4K}}\He_l(x) e^{-x^2/2}\mathrm{d}x
       \\ & \label{eq:P3ForHStar-L1-3} = 
       \sum_{i_1=-A_1}^{A_1-1}\sum_{i_2=0}^{A_2-1}\sum_{j=1}^K \left(\frac{j}{K}\right)^k\int_{i_1+i_2\varepsilon+\frac{4(j-1)\varepsilon}{4K}}^{i_1+i_2\varepsilon+\frac{(4(j-1)+2)\varepsilon}{4K}}\He_l(x) e^{-x^2/2}\mathrm{d}x
       \\ & \label{eq:P3ForHStar-L1-4} \quad + \sum_{i_1=-A_1}^{A_1-1}\sum_{i_2=0}^{A_2-1}\sum_{j=1}^K \left(\frac{j}{K}\right)^k\int_{i_1+i_2\varepsilon+\frac{4(j-1)+2\varepsilon}{4K}}^{i_1+i_2\varepsilon+\frac{(4(j-1)+2+\alpha_{i_1,i_2,j})\varepsilon}{4K}}\He_l(x) e^{-x^2/2}\mathrm{d}x
    \end{align}
    For the first term, we have
    \begin{align}
        \eqref{eq:P3ForHStar-L1-3}
       & = \sum_{j=1}^K \frac{2}{4K}\left(\frac{j}{K}\right)^k\sum_{i_1=-A_1}^{A_1-1}\sum_{i_2=0}^{A_2-1}\int_{i_1+i_2\varepsilon}^{i_1+(i_2+1)\varepsilon}\He_l(x) e^{-x^2/2}\mathrm{d}x
        + O(A_2^{-1})
        \\ & = \sum_{j=1}^K \frac{2}{4K}\left(\frac{j}{K}\right)^k\int_{-A_1}^{A_1}\He_l(x) e^{-x^2/2}\mathrm{d}x
        + O(A_2^{-1})
        \\ & = O(A_2^{-1};A_1) O(e^{-A_1^2/2})+ O(A_2^{-1}),
        \label{eq:P3ForHStar-L1-5}
    \end{align}
    where $O(A_2^{-1};A_1)$ is the big-$O$ notation that treats $A_1$ as a constant.
    Moreover, by the definition of $\alpha(a)$, 
    \begin{align}
        \eqref{eq:P3ForHStar-L1-4}
        &=
        \sum_{i_1=-A_1}^{A_1-1}\sum_{i_2=0}^{A_2-1}\int_{i_1+i_2\varepsilon}^{i_1+(i_2+1)\varepsilon}\sum_{j=1}^K \left(\frac{j}{K}\right)^k\frac{\alpha_{i_1,i_2,j}}{4K}\He_l(x) e^{-x^2/2}\mathrm{d}x
        + O(A_2^{-1};A_1)
        \\ & =  \sum_{i_1=-A_1}^{A_1-1}\sum_{i_2=0}^{A_2-1}\frac{1}{M_1M_2}\int_{i_1+i_2\varepsilon}^{i_1+(i_2+1)\varepsilon}\sum_{j=1}^L a_{k,j}\He_j(i_1+i_2\varepsilon)\He_l(x) e^{-x^2/2}\mathrm{d}x
        + O(A_2^{-1};A_1)
        \\ &=\sum_{i_1=-A_1}^{A_1-1}\sum_{i_2=0}^{A_2-1}\frac{1}{M_1M_2}\int_{i_1+i_2\varepsilon}^{i_1+(i_2+1)\varepsilon}\sum_{j=1}^L a_{k,j}\He_j(x)\He_l(x) e^{-x^2/2}\mathrm{d}x+ O(A_2^{-1};A_1)
        \\ & = \frac{a_{k,l}}{M_1M_2}\int_{-A_1}^{A_1}\He_l^2(x) e^{-x^2/2}\mathrm{d}x+ O(A_2^{-1};A_1)
        \\ & = \frac{a_{k,l}}{M_1M_2}\int \He_l^2(x) e^{-x^2/2}\mathrm{d}x+ O(A_2^{-1};A_1)+O( e^{-A_1^2/2})
        \label{eq:P3ForHStar-L1-6}
    \end{align}
    Now, \eqref{eq:P3ForHStar-L1-5} and \eqref{eq:P3ForHStar-L1-6} yield
    \begin{align}
        \int (h^*_a (x))^k \He_l e^{-x^2/2}\mathrm{d}x
        = \frac{a_{k,l}}{M_1M_2}\int \He_l^2(x) e^{-x^2/2}\mathrm{d}x+ O(A_2^{-1};A_1)+O( e^{-A_1^2/2}).
    \end{align}
    Since $\frac{1}{M_1M_2}\int \He_l^2(x) e^{-x^2/2}\mathrm{d}x>0$, by taking $A_1$ sufficiently large and then taking $A_2$ sufficiently large, we obtain the assertion.
\end{proof}

\subsection{Polynomial Approximation of $h^*_a$}
\label{subsection:PolynomialApproximation}
Finally, we consider the polynomial approximation of $h^*_a$, which can be reduced to polynomial approximation of each of the indicator functions. Approximation of the step/sign/indicator functions has been studied since the nineteenth century 
\citep{zolotarev1877application,akhiezer1990elements,EremenkoUniform2007}.
Among them, \citet{EremenkoUniform2007} considered the polynomial approximation of $\mathrm{sgn}(x)$ and proved the following result: the approximation error $L_m(\delta)$ with degree-$m$ polynomials, in the interval of
$[-1,-\delta]\cup [\delta,1]$, satisfies
\begin{align}
    \lim_{m\to\infty}\sqrt{m}\left(\frac{1+\delta}{1-\delta}\right)^m L_m(\delta) = \frac{1-\delta}{\sqrt{\pi \delta}}.
\end{align}
This entails that, until $m$ becomes larger than $\delta^{-1}$, the error drops proportionally to $O(m^{-1/2}\delta^{-1/2})$.
After that, the error exponentially decreases, proportionally to $e^{-\delta m}$. 
However,this result is not directly applicable to our Gaussian setting, since the error bound is for a fixed interval. 
Also, the coefficients of the polynomial are not characterized (hence higher-order polynomials could become larger outside of the interval). 
Consequently, increasing the order of polynomials may not give smaller approximation error in expectation.

Therefore, we instead make use of the following fact of Hermite expansion.
% Specifically, we use Proposition~\ref{proposition:ConvergenceOfHermiteExpansion}.
\begin{proposition}[Theorems 3 and 6 of \citet{muckenhoupt1970mean}]\label{proposition:ConvergenceOfHermiteExpansion}
    Let $U(x)=e^{-x^2/2}(1+|x|)^b$ and $V(x)=e^{-x^2/2}(1+|x|)^B$, where $b<0$, $B\geq -2/3$, and $b\leq B-1/3$.
    $s_n(x;f)$ denotes the $n$-th partial sum of a Hermite series for the target function $f$.
    If
    \begin{align}
        \int_{-\infty}^\infty |f(x)|V(x)(1+\log^+|x|+\log^+|f(x)|)<\infty,
    \end{align}
    then we have that
    \begin{align}
        \lim_{n\to\infty}\int_{-\infty}^\infty
        |s_n(x;f)-f(x)|^pU(x)^p\mathrm{d}x=0
        .
    \end{align}
\end{proposition}
    We can therefore approximate $\mathbbm{1}_{0,\infty}$ with an arbitrary accuracy with respect to the integral values.
    \begin{lemma}
        There exists a sequence of polynomials $\{p_n\}$ such that
        \begin{align}
        \lim_{n\to\infty}\int_{-\infty}^\infty
        |p_n(x)-\mathbbm{1}_{0,\infty}(x)|^k\He_l(x)e^{-x^2/2}\mathrm{d}x=0
    \end{align}
    holds for all $1\leq k\leq K$ and $1\leq l\leq L$.
    \end{lemma}
    \begin{proof}
        Let us take $b=-1$ and $B=0$ in 
        Proposition~\ref{proposition:ConvergenceOfHermiteExpansion}. 
        Because $\mathbbm{1}_{0,\infty}(x)$ is bounded, it is easy to see that $ \int_{-\infty}^\infty \mathbbm{1}_{0,\infty}(x)V(x)(1+\log^+|x|+\log^+|\mathbbm{1}_{0,\infty}(x)|)<\infty$ holds.
        We have that
        \begin{align}
        0&=\lim_{n\to\infty}\int_{-\infty}^\infty
        |s_n(x;\mathbbm{1}_{0,\infty})-\mathbbm{1}_{0,\infty}(x)|^KU(x)^K\mathrm{d}x
        \\ &\geq
        \lim_{n\to\infty}\int_{-\infty}^\infty
        |s_n(x/(K+1);\mathbbm{1}_{0,\infty})-\mathbbm{1}_{0,\infty}(x)|^K\He_l(x)e^{-x^2/2}\mathrm{d}x
        .
    \end{align}
    Therefore, we arrive at a sequence of polynomials $g_n(x):=s_n(x/(K+1);\mathbbm{1}_{0,\infty}(\cdot \times (K+1))$ such that
    \begin{align}
        \lim_{n\to\infty}\int_{-\infty}^\infty
        |s_n(x/(K+1);\mathbbm{1}_{0,\infty})-\mathbbm{1}_{0,\infty}(x)|^k\He_l(x)e^{-x^2/2}\mathrm{d}x=0
    \end{align}
    holds for all $1\leq k\leq K$ and $1\leq l\leq L$.
    \end{proof}
    Now we have obtained some $p_n$ with which we can approximate each of the indicator functions that consists of $h_a^*$ up to arbitrary accuracy.
    Each function $h_a^*$ is not identically zero, and the integral value of $h_a$ is continuous with respect to $x$. 
    Therefore, there exists a class of polynomials $\{h_a(x)\}$ that satisfies (P1)-(P3).

\end{document}